\documentclass[jair,twoside,dvipsnames,11pt]{article}
\usepackage{jair,rawfonts}

\usepackage[compatibility=false]{caption}
\usepackage{subcaption}
\usepackage{float}
\usepackage{xspace}
\usepackage[utf8]{inputenc}
\usepackage{enumitem}
\usepackage{stmaryrd}
\usepackage{comment}
\usepackage{adjustbox}         
\usepackage{booktabs,multirow,siunitx}  
\usepackage{bm}
\usepackage[graphicx]{realboxes}
\usepackage{xspace}
\usepackage{setspace}
\usepackage{empheq}

\usepackage{amsmath,amssymb,amssymb,amsthm,latexsym,thmtools}
\usepackage{graphicx,verbatim,color,epsfig,graphics,url,multirow}
\usepackage[noend,linesnumberedhidden,ruled,resetcount,noline]{algorithm2e}
\usepackage[english]{babel}
\usepackage{algorithmic}
\usepackage{dashbox}
\usepackage{stmaryrd}
\usepackage{textcomp}
\usepackage{arydshln}
\usepackage{etoolbox}
\usepackage[hidelinks,breaklinks]{hyperref}
\usepackage[noabbrev,nameinlink,capitalise]{cleveref}
\usepackage{soul}
\usepackage{calc}
\usepackage{refcount}

\usepackage{xpatch}
\makeatletter
\AtBeginDocument{\xpatchcmd{\@thm}{\thm@headpunct{.}}{\thm@headpunct{}}{}{}}
\makeatother

\usepackage{xargs}
\usepackage{xparse}
\usepackage{lastpage}

\usepackage{pgf}
\usepackage{tikz}
\usepackage{forest}

\usetikzlibrary{arrows,decorations.pathmorphing,decorations.footprints,fadings,calc,trees,mindmap,shadows,decorations.text,patterns,positioning,shapes,matrix,fit}
\usetikzlibrary{arrows,shadows,backgrounds}
\usetikzlibrary{arrows.meta}
\usetikzlibrary{positioning,fit}
\usetikzlibrary{automata}
\usetikzlibrary{matrix}
\usetikzlibrary{shapes.symbols,shapes.misc,shapes.arrows}
\usetikzlibrary{matrix,chains,scopes,decorations.pathmorphing}

\DeclareMathAlphabet{\mathcal}{OMS}{cmsy}{m}{n}

\newtheoremstyle{nthmstyle}
{3pt}
{3pt}
{}
{}
{\bfseries}
{.}
{.5em}
{}

\theoremstyle{nthmstyle}

\newcolumntype{L}[1]{>{\raggedright\let\newline\\\arraybackslash\hspace{0pt}}m{#1}}
\newcolumntype{C}[1]{>{\centering\let\newline\\\arraybackslash\hspace{0pt}}m{#1}}
\newcolumntype{R}[1]{>{\raggedleft\let\newline\\\arraybackslash\hspace{0pt}}m{#1}}

\makeatletter
\def\thm@space@setup{\thm@preskip=0pt
\thm@postskip=0pt}
\makeatother

\newtheoremstyle{thmstyle}
{4.0pt} 
{2.5pt} 
{\mdseries} 
{} 
{\bfseries} 
{.} 
{ } 
{} 

\newtheoremstyle{exmpstyle}
{4.0pt} 
{4.0pt} 
{\mdseries} 
{} 
{\itshape} 
{.} 
{ } 
{} 

\theoremstyle{thmstyle}

\newtheorem{prop}{Proposition} 
\newtheorem{defn}{Definition} 
\newtheorem{lem}{Lemma}
\newtheorem{cor}{Corollary}

\newtheorem{assump}{Assumption}

\crefname{thrm}{Theorem}{Theorems}
\crefname{prop}{Proposition}{Propositions}
\crefname{defn}{Definition}{Definitions}
\crefname{lem}{Lemma}{Lemmas}
\crefname{cor}{Corollary}{Corollaries}
\crefname{crit}{Criterion}{Criteria}
\crefname{clm}{Claim}{Claims}
\crefname{prob}{Problem}{Problems}
\crefname{assump}{Assumption}{Assumptions}

\theoremstyle{exmpstyle}

\newtheorem{exmpl}{Example}
\newtheorem*{exmp*}{Example}
\newenvironment{exmp}
{\pushQED{\qed}\exmpl}
{\popQED\endexmpl}
\crefname{rmrk}{Remark}{Remarks}
\crefname{exmp}{Example}{Examples}

\definecolor{gray}{rgb}{.4,.4,.4}
\definecolor{midgrey}{rgb}{0.5,0.5,0.5}
\definecolor{middarkgrey}{rgb}{0.35,0.35,0.35}
\definecolor{darkgrey}{rgb}{0.3,0.3,0.3}
\definecolor{midred}{rgb}{0.7,0.2,0.2}
\definecolor{darkred}{rgb}{0.7,0.1,0.1}
\definecolor{midblue}{rgb}{0.2,0.2,0.7}
\definecolor{darkblue}{rgb}{0.1,0.1,0.5}
\definecolor{midgreen}{rgb}{0.3,0.5,0.3}
\definecolor{darkgreen}{rgb}{0.1,0.5,0.1}
\definecolor{defseagreen}{cmyk}{0.69,0,0.50,0}

\newcommand{\jnoteF}[1]{}

\newcommand{\fml}[1]{{\mathcal{#1}}}

\newcommand{\tn}[1]{\textnormal{#1}}

\newcommand{\mbf}[1]{\ensuremath\mathbf{#1}}
\newcommand{\mbb}[1]{\ensuremath\mathbb{#1}}
\newcommand{\mrm}[1]{\ensuremath\mathrm{#1}}
\newcommand{\msf}[1]{\ensuremath\mathsf{#1}}
\newcommand{\tbf}[1]{\textbf{#1}}

\newcommand{\stwop}{\Sigma_2^{\tn{p}}}

\DeclareMathOperator*{\nentails}{\nvDash}
\DeclareMathOperator*{\entails}{\vDash}

\DeclareMathOperator*{\limply}{\rightarrow}

\DeclareMathOperator*{\axp}{\msf{AXp}}
\DeclareMathOperator*{\cxp}{\msf{CXp}}

\DeclareMathOperator*{\waxp}{\msf{WeakAXp}}
\DeclareMathOperator*{\wcxp}{\msf{WeakCXp}}

\NewDocumentCommand{\cloper}{ O{} } {\ensuremath\fml{M}_{#1}=(\fml{F}_{#1},\mbb{D}_{#1},\mbb{F}_{#1},\fml{K}_{#1},\kappa_{#1})}

\NewDocumentCommand{\xpprob}{ O{\mbf{v}} O{c} O{}}{\ensuremath\fml{E}_{#3}=(\fml{M}_{#3},(#1,#2))}

\tikzset{
  0 my edge/.style={densely dashed, my edge},
  my edge/.style={-{Stealth[]}},
}

\SetAlgoNoEnd
\SetAlgoNoLine
\SetFillComment
\SetKwBlock{Let}{let}{end}
\SetKwBlock{FBlock}{}{}
\SetKw{KwNot}{not\xspace}
\SetKw{KwAnd}{and\xspace}
\SetKw{KwOr}{or\xspace}
\SetKw{Break}{break\xspace}
\SetKw{Cont}{continue\xspace}
\SetKwData{false}{{false}}
\SetKwData{true}{{true}}
\SetKwData{st}{\small{\sl st}}
\SetKwData{cores}{$\mathcal{C}$}
\SetKwFunction{SAT}{SAT}
\SetKwBlock{Let}{let}{end}
\SetKwBlock{FBlock}{}{end}
\SetKwInOut{Global}{Global}
\SetKwHangingKw{Algorithm}{Algorithm}
\SetKwHangingKw{Function}{function}
\SetKw{Func}{Function}
\SetKwBlock{Begin}{}{}

\SetAlCapSty{}
\SetAlCapSkip{0.5em}

\newcommand{\BotBlankLine}{\vspace*{1.5pt}}

\definecolor{darkred}{rgb}{0.7,0.1,0.1}
\newcommand{\hlight}[1]{{\color{darkred}#1}}
\newcommand{\rhlight}[1]{\hlight{#1}}
\newcommand{\dghlight}[1]{{\color[RGB]{0,120,0}#1}}
\definecolor{tred3}{HTML}{A40000}
\definecolor{tgreen3}{HTML}{4E9A06}

\newcommand{\tsf}[1]{{\textsf{#1}}}

\definecolor{midblue}{rgb}{0.2,0.2,0.7}
\definecolor{darkblue}{rgb}{0.1,0.1,0.5}

\hypersetup{
  colorlinks = true,
  linkcolor = RoyalBlue,
  urlcolor  = RoyalBlue,
  citecolor = RoyalBlue,
  anchorcolor = RoyalBlue
}

\DeclareRobustCommand{\rchi}{{\mathpalette\irchi\relax}}
\newcommand{\irchi}[2]{\raisebox{\depth}{$#1\chi$}} 

\newtoggle{showpbox}
\settoggle{showpbox}{false}

\crefname{enumi}{}{}
\setlist[enumerate,1]{label=\arabic*., labelsep=0.5em, leftmargin=*}

\setitemize{noitemsep,topsep=1.5pt,parsep=0pt,partopsep=0pt}
\setenumerate{noitemsep,topsep=1.5pt,parsep=0pt,partopsep=0pt}

\makeatletter

\patchcmd{\NAT@citex}
  {\@citea\NAT@hyper@{%
     \NAT@nmfmt{\NAT@nm}%
     \hyper@natlinkbreak{\NAT@aysep\NAT@spacechar}{\@citeb\@extra@b@citeb}%
     \NAT@date}}
  {\@citea\NAT@nmfmt{\NAT@nm}%
   \NAT@aysep\NAT@spacechar\NAT@hyper@{\NAT@date}}{}{}

\patchcmd{\NAT@citex}
  {\@citea\NAT@hyper@{%
     \NAT@nmfmt{\NAT@nm}%
     \hyper@natlinkbreak{\NAT@spacechar\NAT@@open\if*#1*\else#1\NAT@spacechar\fi}%
       {\@citeb\@extra@b@citeb}%
     \NAT@date}}
  {\@citea\NAT@nmfmt{\NAT@nm}%
   \NAT@spacechar\NAT@@open\if*#1*\else#1\NAT@spacechar\fi\NAT@hyper@{\NAT@date}}
  {}{}

\makeatother

\newcommand{\PaperFirstPage}{1}
\newcommand{\PaperLastPage}{\pageref{LastPage}}
\newcommand{\SubDate}{12/2021}
\newcommand{\PubDate}{09/2022}
\newcommand{\PaperVolume}{75}
\newcommand{\PaperYear}{2022}

\newcommand{\pagest}{261}
\renewcommand{\PaperFirstPage}{\pagest}
\renewcommand{\PaperLastPage}{\getpagerefnumber{lastpage}}
\renewcommand{\SubDate}{12/2021}
\renewcommand{\PubDate}{09/2022}
\renewcommand{\PaperVolume}{75}
\renewcommand{\PaperYear}{2022}

\newcommand{\PaperTitle}{On Tackling Explanation Redundancy in Decision Trees}
\newcommand{\ShortTitle}{Explanation Redundancy in Decision Trees}

\jairheading{\PaperVolume}{\PaperYear}{\PaperFirstPage-\PaperLastPage}{\SubDate}{\PubDate}
\ShortHeadings{\ShortTitle}
{Izza, Ignatiev, \& Marques-Silva}
\firstpageno{\pagest}

\begin{document}

\title{\PaperTitle}

\author{%
  \name Yacine Izza \email yacine.izza@univ-toulouse.fr \\
  \addr University of Toulouse, Toulouse, France
  \AND
  \name Alexey Ignatiev \email alexey.ignatiev@monash.edu \\
  \addr Monash University, Melbourne, Australia
  \AND
  \name Joao Marques-Silva \email joao.marques-silva@irit.fr \\
  \addr IRIT, CNRS, Toulouse, France
}


\maketitle

\label{firstpage}
\begin{abstract}
  Decision trees (DTs) epitomize the ideal of interpretability of
  machine learning (ML) models. The interpretability of decision trees
  motivates explainability approaches by so-called intrinsic
  interpretability, and it is at the core of recent proposals for
  applying interpretable ML models in high-risk applications.
  The belief in DT interpretability is justified by the fact that
  explanations for DT predictions are generally expected to be
  succinct. Indeed, in the case of DTs, explanations correspond to
  DT paths.
  Since decision trees are ideally shallow, and so paths contain far
  fewer features than the total number of features, explanations
  in DTs are expected to be succinct, and hence interpretable.
  This paper offers both theoretical and experimental arguments
  demonstrating that, as long as interpretability of decision trees
  equates with succinctness of explanations, then decision trees
  ought not be deemed interpretable.
  The paper introduces logically rigorous path explanations and path
  explanation redundancy, and proves that there exist functions for
  which decision trees must exhibit paths with explanation redundancy
  that is arbitrarily larger 
  than the actual path explanation. The paper also proves that only a
  very restricted class of functions can be represented with DTs that
  exhibit no explanation redundancy.
  In addition, the paper includes experimental results substantiating
  that path explanation redundancy is observed ubiquitously in
  decision trees, including those obtained using different tree
  learning algorithms, but also in a wide range of publicly available
  decision trees.
  The paper also proposes polynomial-time algorithms for eliminating
  path explanation redundancy, which in practice require negligible
  time to compute. Thus, these algorithms serve to indirectly attain
  irreducible, and so succinct, explanations for decision trees.
  Furthermore, the paper includes novel results related with duality
  and enumeration of explanations, based on using SAT solvers as
  witness-producing NP-oracles.
\end{abstract}

\clearpage
\tableofcontents
\clearpage

\section{Introduction} \label{sec:intro}

The cognitive limits of human decision makers~\cite{miller-pr56}
substantiate why succinctness is one of the key requirements of
explanations of machine learning (ML) models. Succinct explanations
are generally accepted to be easier to understand by human decision
makers, but are also easier to diagnose or debug.
Decision trees (DTs) epitomize so-called interpretable machine
learning models~\cite{breiman-ss01,rudin-naturemi19,molnar-bk20}, in
part because paths in the tree (which are possibly short, and so
potentially succinct) represent explanations of predictions.

Decision trees (DTs) find a wide range of practical
uses\footnote{%
  From an ever-increasing range of practical uses, example
  references
  include~\cite{kumar-kis08,kumar-bk09,kotsiantis-air13,valdes-naturesr16,bertsimas-jco18,bertsimas-as18,bertsimas-jamaped19,islam-acmcs19,bertsimas-jo19a,bertsimas-jo19b,bertsimas-ajt19,bertsimas-jps19,bertsimas-corr19,lundberg-naturemi20,valdes-pnas20,bertsimas-hcms20,bertsimas-jco20,bertsimas-plosone20a,bertsimas-plosone20b,bertsimas-plosone20c,suh-naturesr20,bertsimas-corr20,herrera-acmcs21,islam-acm-tkdd21,bertsimas-ml21,bertsimas-jacs21,bertsimas-jtacs21,bertsimas-msom21,bertsimas-corr21}.
}.
Moreover, DTs are the most visible example of a collection of machine
learning (ML) models that have recently been advocated as essential
for high-risk applications~\cite{rudin-naturemi19}. Decision trees
also represent explainability approaches based on intrinsic
interpretability~\cite{molnar-bk20}\footnote{%
  Interpretability is generally accepted to be a subjective concept,
  without a rigorous definition~\cite{lipton-cacm18}.
  Similar to other works~\cite{molnar-bk20}, this paper relates
  interpretability with succinctness of the explanations provided.}.
Given a decision tree, some input and the resulting prediction, the
explanation associated with that prediction is the path in the
decision tree consistent with the input. This simple observation
justifies in part why decision trees have been deemed interpretable
for at least two decades~\cite{breiman-ss01}, an observation that is
widely taken for
granted~\cite{freitas-sigkdd13,bertsimas-jo19a,molnar-bk20,herrera-if20a},
that motivates many of the applications referenced above, and which
explains the interest in learning optimal decision trees, especially
in recent
years\footnote{%
  Standard references include~\cite{nijssen-kdd07,hebrard-cp09,nijssen-dmkd10,bertsimas-ml17,verwer-cpaior17,nipms-ijcai18,verwer-aaai19,rudin-nips19,avellaneda-corr19,avellaneda-aaai20,schaus-cj20,schaus-aaai20,rudin-icml20,janota-sat20,hebrard-ijcai20,schaus-ijcai20a,schaus-ijcai20b,demirovic-aaai21,szeider-aaai21a,szeider-aaai21b,mcilraith-cp21,ansotegui-corr21,demirovic-jmlr22,rudin-aaai22}}, 
and notably when it is well-known that learning optimal (smallest) DTs
is NP-hard~\cite{rivest-ipl76}.
It should be noted that earlier work encompasses different optimality
criteria, some of which is tightly related with succinctness of
explanations (e.g.\ as measured by average path length).
In contrast with earlier work, this paper offers a different
perspective. Concretely, the paper proves that paths in decision trees
can be arbitrarily larger (on the number of features) than a logically
rigorous explanation for a prediction. Furthermore, the experimental
results, obtained on a wide range of datasets and also on publicly
available DTs, demonstrate that DTs in practice naturally exhibit the
same limitation, i.e.\ DTs almost invariably have paths that contain
more literals than what a logically rigorous explanation requires. The
experiments also demonstrate that redundancy of literals in DT paths
exists \emph{even} for optimal (and/or sparse) decision
trees~\cite{bertsimas-ml17,rudin-nips19,rudin-icml20,rudin-corr21}.
The main corollary of the paper's theoretical and experimental results
is that succinctness of explanations cannot be ensured by the paths in
decision trees, and must instead be computed with logically rigorous
approaches.
This corollary has significant practical consequences, in some
high-risk applications, but also in situations that are
safety-critical. For example, in a medical
application~\cite{valdes-naturesr16}, an explanation that contains
literals that are unnecessary, may prevent a physician from focusing
on the symptoms that are actually crucial for correct diagnosis. In
more general settings, non-succinct explanations may be beyond the
grasp of human-decision makers~\cite{miller-pr56}, whereas
(subset-minimal) succinct explanations may not.

Explanations, such as the ones informally sketched above, essentially
represent an answer to a ``\tbf{Why?}'' question, i.e.\ \emph{why} is
the prediction the one obtained? Such explanations aim at succinctness
by being subset-minimal (or irreducible). These explanations are
referred to as PI-explanations or abductive explanations
(AXp's)~\cite{darwiche-ijcai18,inms-aaai19}.
A different class of explanations answer a ``\tbf{Why not?}''
question, i.e.\ \emph{why didn't} one get a prediction different from
the one obtained? Or what would be necessary to change to get a
different prediction? This sort of explanations also aim at
succinctness by being subset-minimal, and are referred to as
contrastive explanations (CXp's)~\cite{miller-aij19,inams-aiia20}.

This paper shows that succinctness of explanations of paths in DTs can
be achieved efficiently in practice.
Concretely, the paper shows that logically rigorous explanations,
i.e.\ both AXp's and CXp's, can be computed in polynomial time, and so
in practice require negligible time to compute.
Furthermore, the paper shows that, whereas AXp's can be arbitrarily
smaller than a path in a DT, CXp's \emph{cannot}. Concretely, the
paper shows that, for any prediction, a contrastive explanation
corresponds exactly to the conditions provided by one of the paths in
the decision tree.
Furthermore, the paper proposes path-specific variants of both AXp's
and CXp's, as opposed to the instance-specific definitions studied in
earlier work. Path-specific explanations relate with the conditions
(i.e.\ the literals) on a given path, and so are instance-independent.
In addition, the paper shows that these variants of AXp's and CXp's
can also be computed in polynomial time.

Compared with earlier
work~\cite{iims-corr20,hiims-kr21},
this paper offers comprehensive evidence regarding the redundancy of
path-based explanations in DTs. Concretely, the paper proves that
i) size-minimal DTs can exhibit arbitrary explanation redundancy,
ii) in practice explanation redundancy is often observed,
iii) DTs without explanation redundancy correspond to a very specific 
class of classifiers represented as non-overlapping minimal
disjunctive normal form formulas,
iv) (provably) optimal sparse DTs also invariably exhibit path
explanation redundancy,
v) example DTs used in most textbooks and other representative
references also exhibit explanation redundancy
and, finally,
vi) other types of explanations (concretely path explanations, which
are investigated in this paper) reveal important properties in terms
of redundancy of explanations.
In addition, the paper builds on earlier
work~\cite{iims-corr20,hiims-kr21} showing that tools claiming
interpretable AI solutions~\cite{bertsimas-ml17,iai} also exhibit path
explanation redundancy, and that this occurs with other well-known
decision tree learners. Therefore, our results serve to complement
any state-of-the-art approach for learning DTs, allowing the
computation of path explanations which are often shorter than DT
paths.
More importantly, our results can be used to provide much-needed
succinct explanations in high-risk and safety-critical applications.

\paragraph{Main results.}
The paper's main results are organized as follows:
\begin{enumerate}[nosep]
\item The paper formalizes in detail the computation of explanations
  in decision trees, such that decision trees are allowed to have both
  categorical and ordinal features, taking values from arbitrary
  domains, and such that an arbitrary number of classes is allowed;
\item The paper introduces explanation functions (as an extension of
  prime implicant explanations), and proposes conditions for
  monotonicity of the definition of abductive and contrastive
  explanations, which in turn yields a generalized form of minimal
  hitting set duality between abductive and contrastive explanations;
\item The paper identifies nesting properties of abductive and
  contrastive explanations, which allows enumerating abductive
  explanations from a subset of the features;
\item The paper uses the two previous results to introduce path
  explanations and path explanation redundancy, where path
  explanations are distinguished from instance-specific explanations;
\item The paper proves that optimal decision trees can exhibit path
  explanation redundancy, and that DTs that do not exhibit path
  explanation redundancy must correspond to minimal generalized
  decision functions (GDF)~\cite{hiicams-aaai22} represented in
  disjunctive normal form (DNF). The class of functions that can be
  represented with such DNF GDFs is argued to be very unlikely to be
  obtained in practice;
\item The paper proposes algorithms for the computation of
  path explanations, as follows:
  \begin{enumerate}[nosep]
  \item Three algorithms for computing abductive path explanations,
    two of which build on earlier work~\cite{iims-corr20,hiims-kr21},
    and a novel one that relates with reasoning about overconstrained
    Horn formulas;
  \item One algorithm for computing all contrastive path explanations;
    and
  \item One algorithm for enumerating abductive path explanations by
    starting from the hypergraph of contrastive path explanations.
  \end{enumerate}
\item The paper offers extensive experimental evidence, attesting to
  the significance of identifying and removing explanation redundancy
  from decision tree paths.
\end{enumerate}

\paragraph{Organization.}
The paper is organized as follows.
\cref{sec:prelim} introduces the notation and definitions used in
the rest of the paper.
~\cref{sec:ndual,sec:rdt} detail the paper's theoretical foundations,
namely duality results and path explanations for DTs. Path
explanations are significant, because these allow relating abductive
and contrastive explanations with the literals in the DT paths.
Concretely, 
\cref{sec:ndual} proposes a generalization of abductive and
contrastive explanations to explanation functions such that duality
between explanations is respected. This section also reveals a nesting
property of explanations, and introduces path explanations.
Furthermore, the section shows how the two previous results apply in
the case of path explanations for DTs.
Moreover, \cref{sec:rdt} builds on path explanations to formalize
\emph{path explanation redundancy} (PXR) for DTs. First, this section
proves that there exist minimum-size DTs that necessarily exhibit PXR.
Second, the section relates DTs that do not exhibit path explanation
redundancy with minimal generalized decision
functions~\cite{hiicams-corr21,hiicams-aaai22}.
In addition, this section shows that optimal sparse
DTs~\cite{rudin-nips19} exhibit PXR, and shows that PXR can represent
in practice a much larger fraction of a path than the explanation
itself.
\cref{sec:xdt} proposes three polynomial-time algorithms for computing
one abductive path explanation, including a novel and simple
propositional Horn encoding.
This section also covers the computation of contrastive path
explanations, and the enumeration of path explanations.
\cref{sec:res} presents experimental results that confirm the paper's
main claims:
i) PXR occurs naturally, and can be found in DTs used in a vast number
of research and survey papers and textbooks published over the years;
ii) PXR is ubiquitous in DTs learned with different tree learning
algorithms, and that the time taken to compute explanations (be them
abductive or contrastive) is always negligible;
iii) PXR can represent a very significant percentage of the length of
tree paths;
and
iv) PXR exists even in trees that are optimal (and
sparse)~\cite{verwer-aaai19,rudin-nips19,rudin-icml20,rudin-corr21}.
\cref{sec:relw} overviews related work on computing explanations
for DTs, and \cref{sec:conc} concludes the paper.

\section{Preliminaries} \label{sec:prelim}

This section overviews the definitions and notation used in the
remainder of the paper.
\cref{ssec:funcs} briefly summarizes the notation for functions used
in the paper, emphasizing function parameterizations, which we will
use to represent families of functions.
\cref{ssec:logic} includes a brief overview of the logic foundations
the paper builds upon.
\cref{ssec:classif} introduces classification problems and the
associated notation.
\cref{ssec:dts} introduces decision trees and outlines a formalization
that is vital for reasoning about DTs.
Although DTs are among the best understood ML models, a rigorous
formalization is required to reason about explanations.
Afterwards, \cref{ssec:fxai} overviews formal explainability.
Finally,~\cref{ssec:sumup} summarizes the notation introduced in this
section and used in the rest of the paper.

\subsection{Function Representation} \label{ssec:funcs}
A function is well-known to be a mapping from one set to another. We
will allow functions to be parameterized, thus in fact defining
families of related functions, which depend on the choices of
parameters. Furthermore, we will allow functions to be parameterized 
on an arbitrary (and not necessarily defined a priori) number of
parameters. (Parameterization serves to represent families of
functions, with arguments which are distinguished from the other
arguments, e.g.\ selected features vs.\ points in feature space.)
As an example, $f:D\to{C}$ maps a domain $D$ into a codomain $C$.
If $d\in{D}$, then $f(d)$ denotes the value of $C$ that $d\in{D}$ is
mapped to.
$f(d; \pi_1, \pi_2)$ denotes that $f$ is parameterized on some given
parameters $\pi_1$ and $\pi_2$. Moreover, $f(d; \pi, \ldots)$ denotes
that $f$ is parameterized on $\pi$ as well as on a number of
additional but yet-undefined parameters. To keep the notation as
simple as possible, we will reveal parameterizations only when
relevant.

\subsection{Logic Foundations} \label{ssec:logic}

Definitions and notation standard in mathematical logic, concretely
related with propositional logic and decidable fragments of
first-order logic,
will be used throughout the paper~\cite{sat-handbook21}.
Propositional formulas are defined over boolean variables taken
from some set $X=\{x_1,x_2,\ldots,x_m\}$, where each boolean variable
takes values from $\mbb{B}=\{0,1\}$.
A literal is a variable $x_i$ or its negation $\neg{x_i}$. A
propositional formula is defined inductively using literals and the
standard logic operators $\lor$ and $\land$~\footnote{%
  For simplicity, we restrict the set of allowed logic operators.
  The inductive definition of propositional formulas above could be
  extended to accommodate for universal and existential operators; it
  could also be extended to accommodate for other well-known logic
  operators, including $\neg$, $\limply$ and $\leftrightarrow$, among
  others.}:
i) Literals are propositional formulas;
ii) If $\varphi_1$ and $\varphi_2$ are propositional formulas, then
$\varphi_1\lor\varphi_2$ is a propositional formula;
and
iii) If $\varphi_1$ and $\varphi_2$ are propositional formulas, then
$\varphi_1\land\varphi_2$ is a propositional formula.
A conjunctive normal form (CNF) formula $\varphi$ is a conjunction of
disjunctions of literals. A disjunction of literals is referred to as
a \emph{clause}.
A disjunctive normal form (DNF) formula is a disjunction of
conjunctions of literals. A conjunction of literals is referred to as
a \emph{term}. 
A Horn formula is a CNF formula where each clause does not contain
more than one non-negated literal.
We will use quantification where necessary, with $\forall$ and
$\exists$ having respectively  the meaning of universal and
existential quantification of variables over their domains.

A \emph{truth assignment} represents a point
$\mbf{v}=(v_1,\ldots,v_m)$ of $\mbb{B}^{m}=\{0,1\}^{m}$, where the
value assigned to each $x_i$ is associated with coordinate $i$. 
$\mbf{v}\entails\varphi$ is defined inductively on the structure of
$\varphi$:
i) $\mbf{v}\entails(\varphi_1\lor\varphi_2)$ iff
$\mbf{v}\entails\varphi_1$ or $\mbf{v}\entails\varphi_2$;
ii) $\mbf{v}\entails(\varphi_1\land\varphi_2)$ iff
$\mbf{v}\entails\varphi_1$ and $\mbf{v}\entails\varphi_2$;
iii) $\mbf{v}\entails\neg{x_i}$ iff $v_i=0$;
and
iv) $\mbf{v}\entails{x_i}$ iff $v_i=1$.
If a truth assignment $\mbf{v}$ is such that $\mbf{v}\entails\varphi$,
then $\varphi$ is \emph{satisfied} by $\mbf{v}$,
and we say that $\mbf{v}$ is a \emph{model};
otherwise $\varphi$ is \emph{falsified} by $\mbf{v}$, and we write
$\mbf{v}\nentails\varphi$.
A formula $\varphi$ is \emph{satisfiable} if there exists a truth
assignment that satisfies $\varphi$; otherwise it is
\emph{unsatisfiable} (or \emph{overconstrained}, or
\emph{inconsistent}).
If $\pi$ and $\kappa$ are propositional formulas, then we write 
$(\pi\entails\kappa)$ to denote that
$\forall(\mbf{x}\in\mbb{B}^m).(\mbf{x}\entails\pi)\limply(\mbf{x}\entails\kappa)$. Similarly,
we write $\pi\nentails\kappa$ to denote that
$\exists(\mbf{x}\in\mbb{B}^m).(\mbf{x}\entails\pi)\land(\mbf{x}\nentails\kappa)$.
A term $\pi$ is a \emph{prime implicant} of $\kappa$, if
$\pi\entails\kappa$ and for any term $\theta$ such that
$\theta\entails\pi\land\pi\nentails\theta$, it does not hold that 
$\theta\entails\kappa$.
Similarly, a clause $\psi$ is a \emph{prime implicate} of $\kappa$, if
$\kappa\entails\psi$ and for any clause $\gamma$ such that
$\gamma\entails\psi\land\psi\nentails\gamma$, it does not hold that
$\kappa\entails\gamma$. 

The definitions above can be extended to domains other than boolean 
domains, by allowing the variables to take values from domains that
are not necessarily boolean, and by defining literals using
appropriate relational operators~\cite{sat-handbook21}. Well-known
examples of relational operators include those in
$\{\le,\ge,<,>,=,\in\}$, among others.
We can also consider functions whose codomain is not necessarily
boolean, and can also include those in logic formulas again using
suitable relational operators.
Concrete examples will be introduced later in this section when
describing decision trees, but also when introducing formal
explanations.
Furthermore, in a number of situations, it is convenient to talk about
formulas that consist of conjunctions of other formulas as \emph{sets
  of constraints}, where each constraint can represent a clause, or a
more complex (propositional) formula, thus allowing set notation to be
used with conjunctions of constraints.

For an overconstrained formula, not all of its constraints can be
satisfied simultaneously. In general, overconstrained formulas are 
split into a set of \emph{hard} constraints (i.e.\ $\fml{H}$) and a
set of \emph{soft} (or \emph{breakable}, or weighted, or costed)
constraints (i.e.\ $\fml{B}$), In such  settings, a number of
computational problems can be defined for reasoning about the pairs 
$(\fml{H},\fml{B})$, including:
i) finding an assignment that maximizes the cost of satisfied soft
constraints, i.e.\ the maximum satisfiability (MaxSAT) problem;
ii) finding subset-maximal subsets of $\fml{B}$ which, together with
$\fml{H}$ are satisfiable, i.e.\ finding a maximal satisfiable subset
(MSS);
iii) finding a subset-minimal set of clauses $\fml{C}\subseteq\fml{B}$
which, if removed from $\fml{B}$, cause
$\fml{H}\cup(\fml{B}\setminus\fml{C})$ to be satisfiable,
i.e.\ finding a minimal correction subset (MCS);
and
iv) finding a subset-minimal set of clauses $\fml{U}\subseteq\fml{B}$
which together with $\fml{H}$ are inconsistent, i.e.\ finding a
minimal unsatisfiable subset (MUS).
There is a comprehensive body of research on algorithms for reasoning
about overconstrained
formulas~\cite{blms-aicom12,mshjpb-ijcai13,msjb-cav13,mpms-ijcai15,amms-sat15,lpmms-cj16,mipms-sat16,msimp-jelia16,msjm-aij17,msm-ijcai20,sat-handbook21}.

For some problems, we will use a SAT solver as an oracle. Although a
SAT solver is used for solving a well-known NP-complete
problem~\cite{cook-stoc71}, it is also the case that a SAT solver
ought not be equated with an NP oracle~\cite{msjm-aij17}. This
observation is justified by the fact that SAT solvers report
satisfying assignments (or witnesses) for satisfiable formulas.
Moreover, most SAT solvers also report \emph{summaries} in the case of
unsatisfiable formulas, where a summary is a subset of the clauses
that is itself inconsistent. As a result, when using a SAT solver as a
oracle, we are in fact considering a witness-producing (and most often
summary-providing) NP-oracle.

\subsection{Classification Problems} \label{ssec:classif}
The paper considers classification problems, defined on a set of
features $\fml{F}=\{1,\ldots,m\}$, where each feature $i$ takes values
from a domain $\fml{D}_i$, and $m=|\fml{F}|$ denotes the number of
features. Each domain $\fml{D}_i$ may be categorical or ordinal.
Ordinal domains can be integer or real-valued. The set of domains is
represented by $\mbb{D}=(\fml{D}_1,\ldots,\fml{D}_m)$. The union of
domains is $\mbb{U}=\cup_{i\in\fml{F}}\fml{D}_i$.
(For the sake of simplicity, several of the examples studied in this
paper consider $\fml{D}_i=\{0,1\}$ (i.e.\ binary features).)
Feature space is defined by
$\mbb{F}=\fml{D}_1\times\fml{D}_2\times\ldots\times\fml{D}_m$.
To refer to an arbitrary point in feature space we use the notation
$\mbf{x}=(x_1,\ldots,x_m)$, whereas to refer to a concrete (constant)
point in feature space we use the notation
$\mbf{v}=(v_1,\ldots,v_m)$, with $v_i\in{D_i}$, $i=1,\ldots,m$.
Similarly to the case of domains, and for the sake of simplicity, most
examples in the paper consider a binary classification problem, with
two classes $|\fml{K}|=2$, e.g.\
$\fml{K}=\{\tbf{0},\tbf{1}\}$,
$\fml{K}=\{\mbf{N},\mbf{Y}\}$,
or
$\fml{K}=\{\ominus,\oplus\}$. However, the results in the paper apply
to any decision (or classification) tree used as a classifier.
A classifier computes a non-constant classification function
$\kappa$ that maps the feature space $\mbb{F}$ into a set of classes,
$\kappa:\mbb{F}\to\fml{K}$.
Furthermore, a \emph{boolean classifier} is such that
$\mbb{F}=\{0,1\}^m$ and $\fml{K}=\{0,1\}$.
An \emph{instance} $\fml{I}$ (or example) denotes a pair
$\fml{I}=(\mbf{v},c)$, where $\mbf{v}\in\mbb{F}$ and $c\in\fml{K}$,
such that $\kappa(\mbf{v})=c$. 
To train a classifier (in our case we are interested in DTs), we start 
from a set of instances
$\fml{I}_{T}=\{{I}_1,\ldots,{I}_n\}$.
Algorithms for learning different families of classifiers can be found
in standard
references~\cite{breiman-bk84,quinlan-ml86,quinlan-bk93,ripley-bk96,mitchell-bk97,russell-bk10,flach-bk12,zhou-bk12,shalev-shwartz-bk14,alpaydin-bk14,poole-bk17,bramer-bk20,zhou-bk21}.
There are also algorithms that learn optimal classifiers
(e.g.~decision trees), and some examples are referenced
in~\cref{sec:intro}.

In this paper, a \emph{literal} represents a condition on the values
of a feature.
Depending on the value assigned to the feature, the literal can be
satisfied or falsified. Throughout the paper, and for consistency of
notation, literals will always be of the form $(x_i\in{S_l})$, where
$S_l\subseteq\fml{D}_i$. This literal is satisfied when feature $i$ is
assigned a value from set $S_l$; otherwise it is falsified. For
simplicity of notation, when 
$|S_l|=1$ and $S_l=\{v_i\}$ , we may instead represent a literal by
$(x_i=v_i)$. Moreover, the universe of literals is
$\mbb{L}=\{\fml{L}\,|\,\fml{L}=(x_i\in{S_l}),i\in\fml{F},S_l\subseteq\fml{D}_i\}$.

A point $\mbf{v}=(v_1,\ldots,v_m)$ in feature space
($\mbf{v}\in\mbb{F}$) can also be described by a set of $m$
\emph{literals}, each of the form $(x_i=v_i)$, i.e.\
$\{(x_i=v_i)|i=1,\ldots,m\}$. Alternatively, literals may be
represented using set notation,
i.e.\ $\{(x_i\in\{v_i\})\,|\,i=1,\ldots,m\}$.

Finally, given the definitions above, the universe of
\emph{classification problems} is represented by the set
$\mbb{M}=\{\fml{M}\,|\,\fml{M}=(\fml{F},\mbb{D},\mbb{F},\fml{K},\kappa)\}$,
where each tuple $\fml{M}=(\fml{F},\mbb{D},\mbb{F},\fml{K},\kappa)$
represents a concrete classification problem.

\subsection{Decision Trees} \label{ssec:dts}
A decision tree $\fml{T}=(V,E)$ is a directed acyclic
graph having at most one path between every pair of
nodes, with $V=\{1,\ldots,\mathfrak{V}\}$ and
$E\subseteq{V}\times{V}$.
Moreover, $V$ is partitioned into a set of non-terminal nodes $N$ and
a set of terminal nodes $T$, i.e.\ $V=N\cup{T}$. When referring to the
\emph{size} of the decision tree, we will use $|\fml{T}|$.
$\fml{T}$ has a root node, $\tsf{root}(\fml{T})\in{V}$ characterized
by having no incoming edges, with the convention being that
$\tsf{root}(\fml{T})=1$.
All other nodes have exactly one incoming edge.
Each terminal node is associated with an element $c$ of $\fml{K}$.
Concretely, we assume a function $\varsigma$ mapping terminal nodes to
one of the classes, $\varsigma:T\to\fml{K}$.
For non-terminal nodes $\sigma:N\to2^{V}$ maps each node $r$ to the
set of child nodes of $r$.
The paper considers only univariate decision trees (i.e.\ each
non-terminal node tests only a single feature). (Possible alternatives
include multivariate decision trees~\cite{utgoff-ml95}, but also
non-grounded decision trees~\cite{blockeel-aij98}; these are beyond
the scope of this paper.)
As a result, each non-terminal node is assigned a single feature.
Specifically, we assume a function $\phi$ mapping non-terminal nodes
to one of the features, $\phi:N\to\fml{F}$. As noted earlier, a
variable $x_i$ is used to denote values (from $\fml{D}_i$) that can be
assigned to feature $i$.
A feature $i$ may be associated with multiple nodes connecting the
tree's root node to some terminal node.
Each edge $(r,s)\in{E}$, with $\phi(r)=i$, is associated with a
literal, representing the values from $\fml{D}_i$ for which the edge
is declared consistent. Concretely, $\varepsilon:E\to\mbb{L}$ maps
each edge $(r,s)$ with a literal of the form $x_i\in{S_l}$, with
$i=\phi(r)$ and $S_l\subseteq\fml{D}_i$.
As noted earlier, literals will \emph{always} be of the form
$(x_i\in{S_l})$, with $S_l\subsetneq\fml{D}_i$. (We could consider a
larger set of relational operators for representing literals,
e.g.\ ${\{{\not\in},{=},{\not=},{<},{\le},{\ge},>\}}$ among  
others. To simplify reasoning about decision trees, only the $\in$
relational operator will be used; the other relational operators can
be translated to the $\in$ operator.)
The definition of literals assumed in the paper allows an edge to be
consistent with multiple values, and so the DTs considered in this
paper effectively correspond to multi-edge decision
trees~\cite{zeger-tit11}. This more generalized definition of literals
allows modeling the DTs learned by well-known tree learning
tools~\cite{utgoff-ml97}.
Nevertheless, when denoting that an edge for a node labeled with
feature $i$ is consistent only with a single value $v_i$, we may
simply label the edge with $v_i$ or with $x_i=v_i$, for the sake of
simplicity.
For a given feature, two literals are inconsistent if these represent
sets of values that do not intersect.

\begin{exmp*}
  The literals 
$(x_1\in\{0\})$ and $(x_1\in\{1\})$ are inconsistent, because
$\{0\}\cap\{1\}=\emptyset$. In contrast, the literals
$(x_1\in\{1,3\})$ and $(x_1\in\{2,3,4\})$ are consistent, because
  $\{1,3\}\cap\{2,3,4\}=\{3\}\not=\emptyset$.
\end{exmp*}

A (complete) path $R_k$ in a DT $\fml{T}$ represents a sequence of
nodes $\langle{r_1},{r_2},\ldots,{r_l}\rangle$, with
$r_1,r_2,\ldots,r_l\in{V}$, that connect the root node
to one of the terminal nodes, 
and such that $(r_j,r_{j+1})\in{E}$.
Hence,
$r_1=\tsf{root}(\fml{T})=1$ and $r_l\in{T}$.
Furthermore, the number of paths in a DT $\fml{T}$ is $|T|$, i.e.\ the
number of terminal nodes.
Each path is assigned an identifier $R_k\in\fml{R}$,
where $\fml{R}$ denotes the set of paths of $\fml{T}$. The sequence of
nodes associated with path $R_k\in\fml{R}$, is
$\tsf{seq}(R_k)=\langle{r_1},{r_2},\ldots,{r_l}\rangle$. 
For simplicity, and with a mild abuse of notation, we also use $R_k$
to represent the sequence of nodes associated with the identified
$R_k$.

A DT (as any other classifier) computes a (non-constant)
classification function $\kappa:\mbb{F}\to\fml{K}$.
When studying explanations, we will consider a concrete instance
$(\mbf{v},c)$, with $c=\kappa(\mbf{v})$, and distinguish two sets of
paths, one corresponding to paths with prediction $c$ and another
corresponding to paths with a different prediction. Thus, 
the set $\fml{P}=\{P_1,\ldots,P_{k_1}\}\subseteq\fml{R}$ denotes the
paths corresponding to a prediction of $c$.
Moreover, the set $\fml{Q}=\{Q_1,\ldots,Q_{k_2}\}\subseteq\fml{R}$
denotes the paths corresponding to a prediction in
$\fml{K}\setminus\{c\}$.
Furthermore, given the definition of $\fml{P}$ and $\fml{Q}$, it is
the case that $\fml{R}=\fml{P}\cup\fml{Q}$. When referring to
$R_k\in\fml{R}$, this may represent a path in $\fml{P}$ or a path in
$\fml{Q}$.

Each path in a DT $\fml{T}$ is associated with a (consistent)
conjunction of literals, denoting the values assigned to the features
so as to reach the terminal node in the path.
We will represent the set of literals of some tree path
$R_k\in\fml{R}$ by $\mrm{\Lambda}(R_k)$.
Likewise, the set of features in some tree path $R_k\in\fml{R}$ is
represented by $\mrm{\Phi}(R_k)$.
Moreover, each terminal node associated with a path is represented by
$\tau(R_k)$, $\tau:\fml{R}\to{T}$.
Each path in the tree \emph{entails} (meaning that it is sufficient
for) the prediction associated with the path's terminal node. Let
$c\in\fml{K}$ denote the prediction associated with path $R_k$,
i.e.\ $c=\varsigma(\tau(R_k))$.
Then, it holds that,
\begin{equation} \label{eq:ent01}
\forall(\mbf{x}\in\mbb{F}).
\left[
\bigwedge_{(x_i\in{S_l})\in\mrm{\Lambda}(R_k)}(x_i\in{S_l}) 
\right]
\limply(\kappa(\mbf{x})=c)
\end{equation}
where $\mbf{x}=(x_1,\ldots,x_i,\ldots,x_m)$, $c\in\fml{K}$, and
each $S_l\subseteq\fml{D}_i$.
\begin{exmp*}
  For the example shown in~\cref{fig:runex01}, it is the case that,
  \[
  \forall((x_1,x_2,x_3)\in\{0,1\}^3).
  \left[(x_1\in\{1\})\land(x_2\in\{1\})\land(x_3\in\{1\})\right]\limply(\kappa(c)=\tbf{1})
  \qedhere
  \]
  Furthermore, the classification function associated with this DT can 
  be represented as follows, for $\mbf{x}=(x_1,x_2,x_3)$:
  \[
  \kappa(\mbf{x})=\left\{
  \begin{array}{ccl}
    1 & & \tn{iff
      $[(x_1\in\{1\})\land(x_2\in\{1\})\land(x_3\in\{1\})]\lor[(x_1\in\{1\})\land(x_2\in\{0\})]$}\\[3.5pt]
    0 & & \tn{iff
      $[(x_1\in\{0\})]\lor[(x_1\in\{1\})\land(x_2\in\{1\})\land(x_3\in\{0\})]$}\\
  \end{array}
  \right.
  \]
\end{exmp*}

As discussed below, one underlying assumption is that any pair of
paths in $\fml{R}$ must have at least one pair of inconsistent
literals.
Let $(r_j,r_{j+1})$ denote some edge in path $R_k\in\fml{R}$.
Let $i$ be the feature associated with $r_j$, and let
$S_{ij}\subsetneq\fml{D}_i$ represent the set of the literal 
$(x_i\in{S_{ij}})$, that is associated with the edge $(r_j,r_{j+1})$.
Given $\mbf{v}\in\mbb{F}$, the edge $(r_j,r_{j+1})$ is consistent with
$\mbf{v}$ if $v_i\in{S_{ij}}$; otherwise the edge is inconsistent.
A predicate
$\tsf{consistent}(R_k,\mbf{v})$ is associated with each path $R_k$ and
each point $\mbf{v}$ in feature space,
$\tsf{consistent}:\fml{R}\times\mbb{F}\to\{0,1\}$ (or alternatively,
$\tsf{consistent}\subseteq\fml{R}\times\mbb{F}$).
The predicate $\tsf{consistent}$ is defined as follows: given
the path $R_k\in\fml{R}$ and the point in feature space
$\mbf{v}\in\mbb{F}$,
$\tsf{consistent}$ takes value 1 if all edges of $R_k$ are consistent
given $\mbf{v}$; otherwise $\tsf{consistent}$ takes value 0.

The paper makes the following general assumption with respect to
decision trees.

\begin{assump} \label{assump:dts}
  For a DT $\fml{T}$, it holds that:
  \begin{enumerate}[nosep]
  \item For each point $\mbf{v}$ in feature space, there exists exactly
    one path consistent with $\mbf{v}$.
    \[
    \forall(\mbf{x}\in\mbb{F}).%
    \left[\exists(R_k\in\fml{R}).\tsf{consistent}(R_k,\mbf{x})\land%
    \forall(R_l\in\fml{R}\setminus\{R_k\}).\neg\tsf{consistent}(R_l,\mbf{x})\right]
    \]
    i.e.\
    each point in feature space must be consistent with at least one
    path, and no point in feature space can be consistent with more than one
    path.
  \item For each tree path $R_k\in\fml{R}$, there exists at least one
    point in feature space that is consistent with the path:
    \[
    \forall(R_k\in\fml{R}).\exists(\mbf{x}\in\mbb{F}).\tsf{consistent}(R_k,\mbf{x})
    \]
    i.e.\ there can be no logically inconsistent paths in a DT.
  \end{enumerate}
\end{assump}
\noindent
Unless stated otherwise, for the results presented in this paper
it is presupposed that \cref{assump:dts} holds%
\footnote{%
  \cref{assump:dts} outlines what one might consider fairly reasonable
  conditions regarding the organization of decision trees, and indeed
  appears to capture the intuitive notion of what a decision tree
  should represent.
  However, and perhaps surprisingly, there are recent examples of tree
  learning tools that can learn DTs with logically inconsistent paths,
  e.g.~\cite[Fig.~4]{valdes-naturesr16} and~\cite[Fig.~6b]{rudin-nips19}.
  Fortunately, it is simple to devise linear-time algorithms, on the
  size of the DT (and for domains of constant size), for removing
  logically inconsistent paths.
  There are also well-known examples of DTs with points in feature
  space inconsistent with all the DT paths, i.e.\ DTs with
  \emph{dead-ends}~\cite[Figure~8.1]{stork-bk01}.}.

The following additional definitions will be considered for DTs.
First, let $\rho:\fml{F}\times\fml{R}\to2^{\mbb{U}}$ be such that
$\rho(i,R_k)$ represents the set of values of feature $i$, taken from
$\fml{D}_i$, that are consistent with path $R_k\in\fml{R}$. Clearly,
$\rho(i,R_k)$ is computed by intersecting all the literals testing the
value of feature $i$:
\begin{equation} \label{eq:rhodef}
  \rho(i,R_k)=\bigcap_{(x_i\in{S_l})\in{\mrm{\Lambda(R_k)}}}S_l
\end{equation}
Observe that $\rho(i,R_k)$ serves to aggregate literals that test the
same feature into a single set of values, each of which is consistent
with path $R_k$.

\begin{exmp*}
  For the example shown in~\cref{fig:runex01}, with
  $Q_2=\langle1,3,5,6\rangle$, $\rho(2,P_2)=\{1\}$ and
  $\rho(3,P_2)=\{0\}$.
\end{exmp*}

Moreover,
let $\rchi_I:\mbb{F}\times\fml{R}\to2^{\fml{F}}$  be such that
$\rchi_I(\mbf{v},Q_l)$ represents the subset of features $i$ which
takes a value (in $\mbf{v}$) that is inconsistent with the
values of $i$ that are consistent with $Q_l$.
Similarly, let $\rchi_P:\fml{R}\times\fml{R}\to2^{\fml{F}}$ be such
that $\rchi_P(P_k,Q_l)$ represents the subset of features $i$ for
which each value consistent with $P_k$ is inconsistent with the
consistent values of $i$ that are consistent with $Q_l$.

\begin{exmp*}
  For the example shown in~\cref{fig:runex01}, with
  $\mbf{v}=(0,0,0,0)$,
  $P_1=\langle1,3,4\rangle$, and
  $Q_2=\langle1,3,5,6\rangle$, then
  $\rchi_I(\mbf{v},P_1)=\{1\}$ and
  $\rchi_P(P_1,Q_2)=\{2\}$.
\end{exmp*}

\paragraph{Running examples.}
Throughout the paper, a number of decision trees will be used as
running examples. These DTs are taken from existing
references~\cite{poole-bk17,rudin-nips19,zhou-bk21,rudin-corr21}%
\footnote{%
  The choice of examples taken from published references is
  deliberate, and aims at illustrating the importance of computing
  path explanations for decision trees.}.
For each of the running examples, and with the purpose of simplifying
the analysis, original feature domains are mapped to symbolic
(numbered) domains. Moreover, all examples of classification problems
map to two classes, which we will represent either by
$\{\tbf{0},\tbf{1}\}$ or by $\{\tbf{N},\tbf{Y}\}$. Clearly, these
modifications do not change in any way the semantics of the original
problems.

\begin{exmp}
  \footnote{%
    In this paper, examples that are referenced from the text or by
    other examples are numbered; the others are not.}
  \label{ex:runex01}
  \cref{fig:runex01} is adapted from~\cite{poole-bk17}.
  The original DT is learned from a given dataset~\cite{poole-bk17}
  using a variant of ID3~\cite{quinlan-bk93}.
  %
  %
  %
  As can be observed, $N=\{1,3,5\}$ and $T=\{2,4,6,7\}$. Given the
  instance $(\mbf{v},c)=((1,1,1),1)$, we set
  $P_1=\langle1,3,4\rangle$, $P_2=\langle1,3,5,7\rangle$,
  $Q_1=\langle1,2\rangle$, $Q_2=\langle1,3,5,6\rangle$. Moreover,
  $P_2$ is the path consistent with the instance.
  For path $P_2$, we have
  $\mrm{\Lambda}(P_2)=\{(x_1\in\{1\}),(x_2\in\{1\}),(x_3\in\{1\})\}$.
  Clearly, the literals associated with $\mbf{v}=(1,1,1)$,
  i.e.\ $x_1=1$, $x_2=1$ and $x_3=1$, are consistent with
  $x_1\in\{1\}$, $x_2\in\{1\}$ and $x_3\in\{1\}$, respectively.
  Additional results for this DT are summarized in~\cref{tab:dtrees}
  (see~\cpageref{tab:dtrees}).
\end{exmp}

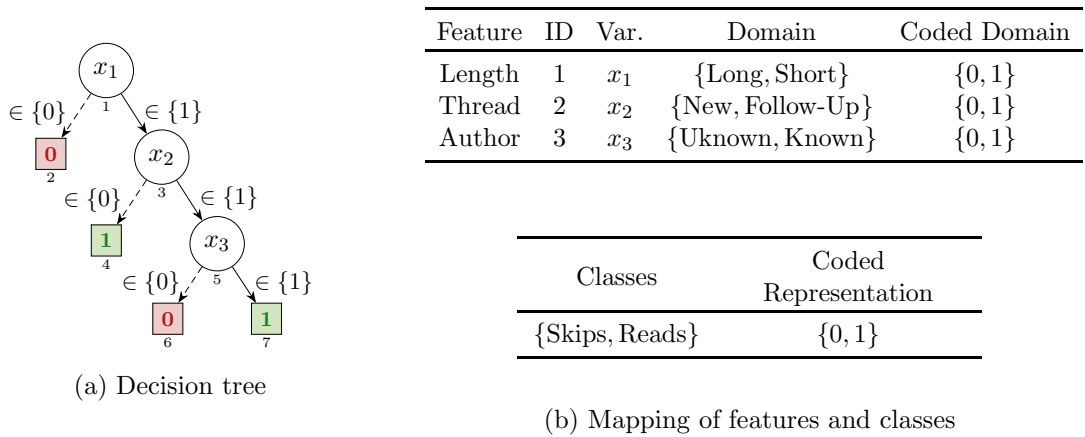
\begin{figure}[t]
  \begin{subfigure}{0.3125\textwidth}
    \scalebox{0.9}{\forestset{
  BDT/.style={
    for tree={
      l=1.125cm,s sep=1.0cm,
      if n children=0{}{circle},
      draw,
      edge={
        my edge
      },
      if n=1{
        edge+={0 my edge},
      }{},
    }
  },
}
\begin{forest}
  BDT
  [$x_1$, label={[yshift=-6.875ex]{{\tiny1}}}
    [{\footnotesize\color{darkred}{\tbf{0}}},
      label={[yshift=-4.775ex]{{\tiny2}}},
      fill={tred3!20}, edge label={node[pos=0.45,left,xshift=-1.0pt]
        {{\small$\in\{0\}$}}}]
    [$x_2$, label={[yshift=-6.875ex]{{\tiny3}}},
      edge label={node[pos=0.45,right,xshift=1.0pt]
        {{\small$\in\{1\}$}}}
      [{\footnotesize\color{darkgreen}{\tbf{1}}},
        label={[yshift=-4.775ex]{{\tiny4}}}, fill={tgreen3!25},
        edge label={node[pos=0.45,left,xshift=-1.0pt] {{\small$\in\{0\}$}}}]
      [$x_3$, , label={[yshift=-6.875ex]{{\tiny5}}},
        edge label={node[pos=0.45,right,xshift=1.0pt]
          {{\small$\in\{1\}$}}} 
        [{\footnotesize\color{darkred}{\tbf{0}}},
          label={[yshift=-4.775ex]{{\tiny6}}}, fill={tred3!20},
          edge label={node[pos=0.45,left,xshift=-1.0pt] {{\small$\in\{0\}$}}}]
        [{\footnotesize\color{darkgreen}{\tbf{1}}},
          label={[yshift=-4.775ex]{{\tiny7}}}, fill={tgreen3!25},
          edge label={node[pos=0.45,right,xshift=1.0pt] {{\small$\in\{1\}$}}}]
      ]
    ]
  ]
\end{forest}}
    \caption{Decision tree}
  \end{subfigure}
  \begin{subfigure}{0.6875\textwidth}
    \begin{center}
      \scalebox{0.9}{
        \renewcommand{\tabcolsep}{0.425em}
        \begin{tabular}{ccccc} \toprule
          Feature & ID & Var. & Domain & Coded Domain \\ \toprule
          Length  & 1  & $x_1$  & $\{\tn{Long},\tn{Short}\}$ & $\{0,1\}$ \\
          Thread  & 2  & $x_2$  & $\{\tn{New},\tn{Follow-Up}\}$ & $\{0,1\}$ \\
          Author  & 3  & $x_3$  & $\{\tn{Uknown},\tn{Known}\}$ & $\{0,1\}$ \\
          \bottomrule
        \end{tabular}
      }
    \end{center}
    
    \smallskip

    \begin{center}
      \scalebox{0.9}{
        \begin{tabular}{C{2.5cm}C{3.5cm}} \toprule
          Classes & Coded Representation \\ \toprule
          $\{\tn{Skips},\tn{Reads}\}$ & $\{0,1\}$ \\
          \bottomrule
        \end{tabular}
      }
    \end{center}
    \caption{Mapping of features and classes}
  \end{subfigure}
  \caption{Decision tree, adapted from~\cite[Ch.~07,~Fig.~7.4]{poole-bk17}}
  \label{fig:runex01}
\end{figure}

\begin{exmp} \label{ex:runex02}
  \cref{fig:runex02} is adapted from~\cite{rudin-nips19}. The original
  DT was produced with the tool OSDT (optimal sparse decision
  trees)~\cite{rudin-nips19}.
  %
  %
  %
  As can be observed, $N=\{1,2,4,5,7,8,10\}$ and
  $T=\{3,6,9,11,12,13,14,15\}$.
  Given the instance $(\mbf{v},c)=((0,0,1,0,1),1)$, we set
  $P_1=\langle1,2,4,7,10,15\rangle$, $P_2=\langle1,2,4,7,11\rangle$,
  $P_3=\langle1,2,5,8,13\rangle$, 
  $P_4=\langle1,2,5,9\rangle$, $P_5=\langle1,3\rangle$, and then
  $Q_1=\langle1,2,4,6\rangle$, $Q_2=\langle1,2,4,7,10,14\rangle$,
  $Q_3=\langle1,2,5,8,12\rangle$.
  Moreover, $P_1$ is the path consistent with the instance.
  As can be observed,
  $\mrm{\Lambda}(P_1)=\{(x_1\in\{0\}),(x_2\in\{0\}),(x_3\in\{1\}),(x_4\in\{0\}),(x_5\in\{1\})\}$,
  and the literals associated with $\mbf{v}=(0,0,1,0,1)$ are
  $\{x_1=0,x_2=0,x_3=1,x_4=0,x_5=1\}$, hence being pairwise consistent.
  Additional results for this DT are summarized in~\cref{tab:osdt}
  (see~\cpageref{tab:osdt}).
\end{exmp}
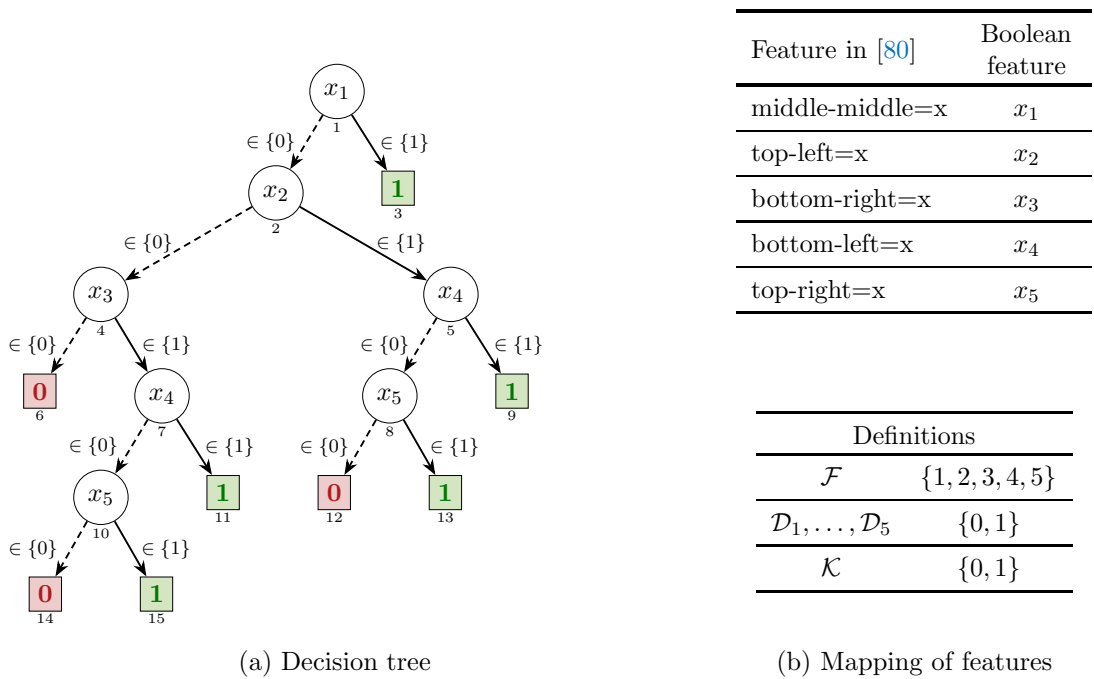
\begin{figure}[t]
  \begin{subfigure}[b]{0.6\textwidth}
    \scalebox{0.9}{
%
\forestset{
  BDT/.style={
    for tree={
      l=1.5cm,s sep=1.15cm,
      if n children=0{}{circle}, 
      draw=black,
      text=black,
      edge={
        my edge
      },
      if n=1{
        edge+={0 my edge},
      }{},
      edge=thick,
    }
  },
}
%
%
\begin{forest}
  BDT
  [{$x_1$}, label={[yshift=-6.875ex]{{\tiny1}}} 
    [{$x_2$}, label={[yshift=-6.875ex]{{\tiny2}}}, 
      edge label={node[midway,left,xshift=-0.5pt] {{\scriptsize$\in\{0\}$}}}
      [{$x_3$}, label={[yshift=-6.875ex]{{\tiny4}}}, 
        edge label={node[midway,left,xshift=-2.5pt] {{\scriptsize$\in\{0\}$}}}
        [\rhlight{\textbf{0}}, label={[yshift=-5.0ex]{{\tiny6}}},
          edge label={node[midway,left,xshift=-0.5pt] {{\scriptsize$\in\{0\}$}}},
          rectangle, fill={tred3!20} ]
        [{$x_4$}, label={[yshift=-6.875ex]{{\tiny7}}}, 
          edge label={node[midway,right,xshift=0.5pt] {{\scriptsize$\in\{1\}$}}}
          [{$x_5$}, label={[yshift=-6.875ex]{{\tiny10}}}, 
            edge label={node[midway,left,xshift=-1.5pt] {{\scriptsize$\in\{0\}$}}}
            [\rhlight{\textbf{0}}, label={[yshift=-5.0ex]{{\tiny14}}},
              edge label={node[midway,left,xshift=-1.5pt] {{\scriptsize$\in\{0\}$}}},
              rectangle, fill={tred3!20} ]
            [\dghlight{\textbf{1}}, label={[yshift=-5.0ex]{{\tiny15}}},
              edge label={node[midway,right,xshift=0.5pt] {{\scriptsize$\in\{1\}$}}},
              rectangle, fill={tgreen3!25} ]
          ]
          [\dghlight{\textbf{1}}, label={[yshift=-5.0ex]{{\tiny11}}},
            edge label={node[midway,right,xshift=0.5pt] {{\scriptsize$\in\{1\}$}}},
            rectangle, fill={tgreen3!25} ]
        ]
      ]
      [{$x_4$}, label={[yshift=-6.875ex]{{\tiny5}}}, 
        edge label={node[midway,right,xshift=1.5pt] {{\scriptsize$\in\{1\}$}}}
        [{$x_5$}, label={[yshift=-6.875ex]{{\tiny8}}}, 
          edge label={node[midway,left,xshift=-0.5pt] {{\scriptsize$\in\{0\}$}}}
          [\rhlight{\textbf{0}}, label={[yshift=-5.0ex]{{\tiny12}}},
            edge label={node[midway,left,xshift=-0.5pt] {{\scriptsize$\in\{0\}$}}},
            rectangle, fill={tred3!20} ]
          [\dghlight{\textbf{1}}, label={[yshift=-5.0ex]{{\tiny13}}},
            edge label={node[midway,right,xshift=0.5pt] {{\scriptsize$\in\{1\}$}}},
            rectangle, fill={tgreen3!25} ]
        ]
        [\dghlight{\textbf{1}}, label={[yshift=-5.0ex]{{\tiny9}}},
          edge label={node[midway,right,xshift=0.5pt] {{\scriptsize$\in\{1\}$}}},
          rectangle, fill={tgreen3!25} ]
      ]
    ]
    [\dghlight{\textbf{1}}, label={[yshift=-5.0ex]{{\tiny3}}},
      edge label={node[midway,right,xshift=0.5pt] {{\scriptsize$\in\{1\}$}}},
      rectangle, fill={tgreen3!25} ]
  ]
\end{forest}}
    \caption{Decision tree}
  \end{subfigure}
  \begin{subfigure}[b]{0.4\textwidth}
    \begin{center}
      \scalebox{0.9}{
        \begin{tabular}{lC{1.5cm}} \toprule
          Feature in~\cite{rudin-nips19} & Boolean feature \\ \toprule
          middle-middle=x & $x_1$ \\ \midrule
          top-left=x & $x_2$ \\ \midrule
          bottom-right=x & $x_3$ \\ \midrule
          bottom-left=x & $x_4$ \\ \midrule
          top-right=x & $x_5$ \\ 
          \bottomrule
        \end{tabular}
      }
    \end{center}

    \bigskip

    \begin{center}
      \scalebox{0.9}{
        \begin{tabular}{cc} \toprule
          \multicolumn{2}{c}{Definitions} \\ \toprule
          $\fml{F}$ & $\{1,2,3,4,5\}$ \\ \midrule
          $\fml{D}_1,\ldots,\fml{D}_5$ & $\{0,1\}$ \\ \midrule
          $\fml{K}$ & $\{0,1\}$ \\
          \bottomrule
        \end{tabular}
      }
    \end{center}
    \smallskip
    \caption{Mapping of features}
  \end{subfigure}
  \caption{Decision tree, adapted from~\cite[Figure~5b]{rudin-nips19}, for
    the \tsf{tic-tac-toe} dataset}
  \label{fig:runex02}
\end{figure}

\begin{exmp} \label{ex:runex03}
  \cref{fig:runex03} is adapted from~\cite{zhou-bk21}, and illustrates
  the application of a standard tree learning algorithm, but where the
  features are categorical (and non-binary).
  %
  In this case, and for completeness, we show features 5 and 6
  (resp.~\tsf{sound} and \tsf{umbilicus}), but these are not
  associated with any node in the DT.
  As can be observed, $N=\{1,2,3,6,11\}$, $T=\{4,5,7,8,9,10,12,13\}$.
  Given the instance $(\mbf{v},c)=((1,2,1,2),\tbf{Y})$, we set
  $P_1=\langle1,2,5\rangle$, $P_2=\langle1,2,6,10\rangle$,
  $P_3=\langle1,2,6,11,12\rangle$, $P_4=\{1,3,9\}$, and then
  $Q_1=\langle1,2,6,11,13\rangle$, $Q_2=\langle1,2,7\rangle$,
  $Q_3=\langle1,3,8\rangle$, $Q_4=\langle1,4\rangle$.
  Moreover, $P_3$ is the path consistent with the instance.
  Additional results for this DT are summarized in~\cref{tab:dtrees}
  (see~\cpageref{tab:dtrees}). 
\end{exmp}

\begin{figure}[t]
  \begin{subfigure}[b]{0.725\textwidth}
    \begin{center}
      \scalebox{0.9}{
%
\forestset{
  BDT/.style={
    for tree={
      l=1.5cm,s sep=1.5cm,
      if n children=0{}{circle}, 
      draw=black,
      text=black,
      edge={
        my edge
      },
      edge=thin,
    }
  },
}
\begin{forest}
  BDT
  [{$x_1$}, label={[yshift=-6.875ex]{{\tiny1}}}, 
    [{$x_2$}, label={[yshift=-0.625ex]{{\tiny2}}}, 
      edge label={node[pos=0.55,above,xshift=-5.5pt] {{\scriptsize$\in\{1\}$}}}
      [\dghlight{\textbf{Y}}, label={[yshift=-5.125ex]{{\tiny5}}},
        rectangle, fill={tgreen3!25},
        edge label={node[pos=0.45,left,xshift=-4pt] {{\scriptsize$\in\{1\}$}}}
      ]
      [{$x_4$}, label={[yshift=-6.875ex]{{\tiny6}}}, 
        edge label={node[pos=0.7,right,xshift=-0.575pt] {{\scriptsize$\in\{2\}$}}}
        [\dghlight{\textbf{Y}}, label={[yshift=-5.125ex]{{\tiny10}}},
          rectangle, fill={tgreen3!25},
          edge label={node[pos=0.45,left,xshift=-1.5pt] {{\scriptsize$\in\{1,3\}$}}}
        ]
        [{$x_3$}, label={[yshift=-6.875ex]{{\tiny11}}}, 
          edge label={node[pos=0.45,right,xshift=0.25pt] {{\scriptsize$\in\{2\}$}}}
          [\dghlight{\textbf{Y}}, label={[yshift=-5.125ex]{{\tiny12}}},
            rectangle, fill={tgreen3!25},
            edge label={node[pos=0.45,left,xshift=-0.5pt] {{\scriptsize$\in\{1\}$}}}
          ]
          [\rhlight{\textbf{N}}, label={[yshift=-5.125ex]{{\tiny13}}},
            rectangle, fill={tred3!20},
            edge label={node[pos=0.45,right,xshift=0.25pt] {{\scriptsize$\in\{2\}$}}}
          ]
        ]
      ]
      [\rhlight{\textbf{N}}, label={[yshift=-5.125ex]{{\tiny7}}},
        rectangle, fill={tred3!20},
        edge label={node[pos=0.45,right,xshift=3.0pt] {{\scriptsize$\in\{3\}$}}}
      ]
    ]
    [{$x_3$}, label={[yshift=-6.875ex]{{\tiny3}}}, 
      edge label={node[pos=0.75,left,xshift=-0.5pt] {{\scriptsize$\in\{2\}$}}}
      [\rhlight{\textbf{N}}, label={[yshift=-5.125ex]{{\tiny8}}},
        rectangle, fill={tred3!20},
        edge label={node[pos=0.45,left,xshift=-0.25pt] {{\scriptsize$\in\{1\}$}}}
      ]
      [\dghlight{\textbf{Y}}, label={[yshift=-5.125ex]{{\tiny9}}},
        rectangle, fill={tgreen3!25},
        edge label={node[pos=0.45,right,xshift=0.25pt] {{\scriptsize$\in\{2\}$}}}
      ]
    ]
    [\rhlight{\textbf{N}}, label={[yshift=-5.125ex]{{\tiny4}}},
      rectangle, fill={tred3!20},
      edge label={node[pos=0.55,above,xshift=4.5pt] {{\scriptsize$\in\{3\}$}}}
    ]
  ]
\end{forest}}
    \end{center}
    \caption{Decision tree}
  \end{subfigure}
  \hspace*{-1.5cm}
  \begin{subfigure}[b]{0.25\textwidth}
    \hspace*{-1.5cm}
    \scalebox{0.875}{
      \begin{tabular}{C{2.5cm}C{3.5cm}} \toprule
        Classes & Coded Representation \\ \toprule
        $\{\tn{ripe},\tn{unripe}\}$ &
        $\{\tn{\tbf{Y}},\tn{\tbf{N}}\}$ \\
        \bottomrule
      \end{tabular}
    }

    \bigskip

    \caption{Mapping of classes}
  \end{subfigure}

  \medskip
  
  \begin{subfigure}{\textwidth}
    \begin{center}
      \scalebox{0.885}{
        \begin{tabular}{ccccc} \toprule
          Feature & ID & Var.  & Domain & Coded Domain \\ \toprule
          Texture & 1  & $x_1$ &
          $\{\tn{clear},\tn{clightly~blurry},\tn{blurry}\}$ &
          $\{1,2,3\}$ \\
          Root    & 2  & $x_2$ &
          $\{\tn{curly},\tn{slightly~curly},\tn{curly}\}$ &
          $\{1,2,3\}$ \\
          Surface & 3  & $x_3$ & $\{\tn{hard},\tn{soft}\}$ &
          $\{1,2\}$\\
          Color   & 4  & $x_4$ & $\{\tn{green},\tn{dark},\tn{light}\}$
          & $\{1,2,3\}$ \\
          Sound & 5 & $x_5$ & $\{\tn{crisp},\tn{muffled},\tn{dull}\}$
          & $\{1,2,3\}$ \\
          Umbilicus & 6 & $x_6$ &
          $\{\tn{flat},\tn{slightly~hollow},\tn{hollow}\}$ &
          $\{1,2,3\}$ \\
          \bottomrule
        \end{tabular}
      }
    \end{center}
    \caption{Mapping of features}
  \end{subfigure}
  
  \caption{Decision tree adapted from~\cite[Ch.~04,~Fig.~4.3]{zhou-bk21}}
  \label{fig:runex03}
\end{figure}

\begin{exmp} \label{ex:runex04}
  \cref{fig:runex04} is adapted from~\cite{rudin-corr21}, and shows a
  DT for the \tsf{recividism} dataset~\cite{propublica16}. Features are
  categorical or ordinal.
  (Feature \tsf{Priors} ranges from 0 to 38, and feature \tsf{Age}
  ranges from 18 to 96. The symbolic names $\tn{MxP}=38$,
  $\tn{MnA}=18$, $\tn{MxA}=96$ are shown in the DT.)
  A distinguishing feature of this running example is that one of the
  features (\tsf{Priors}) is tested more than once along some of the
  paths.
  %
  %
  %
  As can be observed, $N=\{1,3,4,6\}$, $T=\{2,5,7,8,9\}$.
  Given the instance $(\mbf{v},c)=((2,20,0),\tbf{Y})$, we set
  $P_1=\langle1,2\rangle$, $P_2=\langle1,3,4,6,8\rangle$,
  $P_3=\langle1,3,4,7\rangle$, and then
  $Q_1=\langle1,3,4,6,9\rangle$, $Q_2=\langle1,3,5\rangle$.
  Moreover, $P_2$ is the path consistent with the instance.
  Additional results for this DT are summarized in~\cref{tab:osdt}
  (see~\cpageref{tab:osdt}).
\end{exmp}

\begin{figure}[t]
  \begin{subfigure}[b]{0.725\textwidth}
    \begin{center}
      \scalebox{0.9}{
%
\forestset{
  BDT/.style={
    for tree={
      l=1.5cm,s sep=1.5cm,
      if n children=0{}{circle}, 
      draw=black,
      text=black,
      edge={
        my edge
      },
      edge=thin,
    }
  },
}
\begin{forest}
  BDT
  [{$x_1$}, label={[yshift=-6.875ex]{{\tiny1}}}, 
    [\dghlight{\textbf{Y}}, label={[yshift=-5.125ex]{{\tiny2}}},
      rectangle, fill={tgreen3!25},
      edge label={node[pos=0.45,left,xshift=-1pt] {{\scriptsize$\in\{4..\tn{MxP}\}$}}}
    ]
    [{$x_2$}, label={[yshift=-6.875ex]{{\tiny3}}}, 
      edge label={node[pos=0.45,right,xshift=1.0pt] {{\scriptsize$\in\{0..3\}$}}}
      [{$x_3$}, label={[yshift=-6.875ex]{{\tiny4}}}, 
        edge label={node[pos=0.55,left,xshift=-0.5pt] {{\scriptsize$\in\{\tn{MnA}..25\}$}}}
        [{$x_1$}, label={[yshift=-6.875ex]{{\tiny6}}}, 
          edge label={node[pos=0.45,left,xshift=-0.25pt] {{\scriptsize$\in\{0\}$}}}
          [\dghlight{\textbf{Y}}, label={[yshift=-5.125ex]{{\tiny8}}},
            rectangle, fill={tgreen3!25},
            edge label={node[pos=0.45,left,xshift=-0.5pt] {{\scriptsize$\in\{2,3\}$}}}
          ]
          [\rhlight{\textbf{N}}, label={[yshift=-5.125ex]{{\tiny9}}},
            rectangle, fill={tred3!20},
            edge label={node[pos=0.45,right,xshift=0.25pt] {{\scriptsize$\in\{0,1\}$}}}
          ]
        ]
        [\dghlight{\textbf{Y}}, label={[yshift=-5.125ex]{{\tiny7}}},
          rectangle, fill={tgreen3!25},
          edge label={node[pos=0.45,right,xshift=0.5pt] {{\scriptsize$\in\{1\}$}}}
        ]
      ]
      [\rhlight{\textbf{N}}, label={[yshift=-5.125ex]{{\tiny5}}},
        rectangle, fill={tred3!20},
        edge label={node[pos=0.55,right,xshift=0.25pt] {{\scriptsize$\in\{26..\tn{MxA}\}$}}}
      ]
    ]
  ]
\end{forest}

}
    \end{center}
    \caption{Decision tree}
  \end{subfigure}
  \hspace*{-1.5cm}
  \begin{subfigure}[b]{0.25\textwidth}
    \hspace*{-1.5cm}
    \scalebox{0.875}{
      \begin{tabular}{C{2.5cm}C{3.5cm}} \toprule
        Classes & Coded Representation \\ \toprule
        $\{\tn{Yes},\tn{No}\}$ &
        $\{\tn{\tbf{Y}},\tn{\tbf{N}}\}$ \\
        \bottomrule
      \end{tabular}
    }

    \bigskip

    \caption{Mapping of classes}
  \end{subfigure}

  \medskip
  
  \begin{subfigure}{\textwidth}
    \begin{center}
      \scalebox{0.885}{
        \begin{tabular}{ccccc} \toprule
          Feature & ID & Var.  & Domain & Coded Domain \\ \toprule
          Priors & 1  & $x_1$ &
          $\{0,\ldots,\tn{MxP}\}$ &
          $\{0,\ldots,\tn{MxP}\}$ \\
          Age    & 2  & $x_2$ &
          $\{\tn{MnA},\ldots,\tn{MxA}\}$ &
          $\{\tn{MnA},\ldots,\tn{MxA}\}$ \\
          Juvenile crimes & 3  & $x_3$ & $\{0,1\}$ &
          $\{0,1\}$\\
          \bottomrule
        \end{tabular}
      }
    \end{center}
    \caption{Mapping of features}
  \end{subfigure}
  
  \caption{Decision tree adapted from~\cite[Figure~2]{rudin-corr21}.
    According to the dataset, $\tn{MnA}=18$, $\tn{MxA}=96$ and
    $\tn{MxP}=38$, but the numbers are be left symbolic.
  }
  \label{fig:runex04}
\end{figure}

\subsection{Formal Explainability} \label{ssec:fxai}
Formal explanation\footnote{%
  There is an extensive body of work on non-formal XAI approaches to
  XAI~\cite{berrada-ieee-access18,muller-dsp18,muller-bk19,pedreschi-acmcs19,muller-ieee-proc21,guan-ieee-tnnls21,holzinger-bk22,holzinger-xxai22b,doran-jair22}.}
approaches have been studied in a growing body of
research in recent 
years\footnote{%
  A sample of references on formal explainability includes~\cite{darwiche-ijcai18,inms-aaai19,darwiche-aaai19,inms-nips19,nsmims-sat19,hazan-aies19,marquis-kr20,darwiche-pods20,darwiche-ecai20,darwiche-kr20,toni-kr20,mazure-sum20,inams-aiia20,msgcin-nips20,iims-corr20,msgcin-icml21,ims-ijcai21,kwiatkowska-ijcai21,hiims-kr21,marquis-kr21,ims-sat21,asher-cdmake21,cms-cp21,mazure-cikm21,hiicams-corr21,toni-aij21,lorini-clar21,kutyniok-jair21,darwiche-jair21,tan-nips21,barcelo-nips21,hiicams-aaai22,iisms-aaai22,msi-aaai22,rubin-aaai22}.}. 
Concretely, this paper uses the definition of \emph{abductive
explanation}~\cite{inms-aaai19} (AXp), which corresponds to
a PI-explanation~\cite{darwiche-ijcai18} in the case of boolean 
classifiers. AXp's represent prime implicants of the discrete-valued
classifier function (which computes the predicted class)\footnote{%
  There exist also standard references with detailed overviews of the
  uses of prime implicants in the context of boolean
  functions~\cite{somenzi-bk06,crama-bk11}.
  Generalizations of prime implicants beyond boolean domains have been 
  considered before~\cite{marquis-fair91}.
  Prime implicants have also been referred to as minimum satisfying
  assignments in first-order logic~(FOL)~\cite{mcmillan-cav12}, and
  have been studied in modal and description
  logics~\cite{bienvenu-jair09}.
}.
Throughout this paper we will opt to use the acronym AXp to refer to
abductive explanations.

Let us consider a given classifier, computing a classification function
$\kappa$ on feature space $\mbb{F}$, a point $\mbf{v}\in\mbb{F}$, with
prediction $c=\kappa(\mbf{v})$, and let $\fml{X}$ denote a subset of
the set of features $\fml{F}$, 
$\fml{X}\subseteq\fml{F}$. $\fml{X}$ is a weak AXp for the instance
$(\mbf{v},c)$ if,
\begin{equation} \label{eq:axp1}
  \begin{array}{rcl}
    \waxp(\fml{X}) & ~:=~~ &
    \forall(\mbf{x}\in\mbb{F}).%
    \left[\bigwedge_{i\in\fml{X}}(x_i=v_i)\right]\limply(\kappa(\mbf{x})=c)\\ 
  \end{array}
\end{equation}
(We could highlight that $\waxp$ is parameterized on $\kappa$,
$\mbf{v}$ and $c$, but opt not to clutter the notation, and so these
dependencies will be left implicit.)
Thus, given an instance $(\mbf{v},c)$, a (weak) AXp is a set of
features which, if fixed to the values dictated by $\mbf{v}$, then the
prediction is guaranteed to be $c$, independently of the values
assigned to the other features.
$\fml{X}$ is an AXp if, besides being a weak AXp, it is also
subset-minimal, i.e.
\begin{equation} \label{eq:axp2a}
  \begin{array}{rcl}
    \axp(\fml{X}) & ~:=~~ &
    \waxp(\fml{X})\land\forall(\fml{X}'\subsetneq\fml{X}).\neg\waxp(\fml{X}')\\
  \end{array}
\end{equation}  
An AXp can be viewed as a possible answer to a ``\tbf{Why?}''
question, i.e.\ why is the classifier's prediction $c$?

It should be plain in this work, but also in earlier work, that the
representation of AXp's using subsets of features aims at simplicity.
The sufficient condition for the prediction is evidently the
conjunction of literals associated with the features contained in the
AXp.

Similarly to the case of AXp's, one can define (weak) contrastive
explanations (CXp's)~\cite{miller-aij19,inams-aiia20}.
$\fml{Y}\subseteq\fml{F}$ is a weak CXp for the instance $(\mbf{v},c)$
if,
\begin{equation} \label{eq:cxp1}
  \begin{array}{rcl}
    \wcxp(\fml{Y}) & ~:=~~ & \exists(\mbf{x}\in\mbb{F}).%
    \left[\bigwedge_{i\not\in\fml{Y}}(x_i=v_i)\right]\land(\kappa(\mbf{x})\not=c)\\ 
  \end{array}
\end{equation}
(As before, for simplicity we keep the parameterization of $\wcxp$ on
$\kappa$, $\mbf{v}$ and $c$ implicit.)
Thus, given an instance $(\mbf{v},c)$, a (weak) CXp is a set of
features which, if allowed to take any value from their domain, then
there is an assignment to the features that changes the prediction to
a class other than $c$, this while the features not in the explanation
are kept to their values (\emph{ceteris paribus}).

Furthermore, a set $\fml{Y}\subseteq\fml{F}$ is a CXp if, besides
being a weak CXp, it is also subset-minimal, i.e.
\begin{equation} \label{eq:cxp2a}
  \begin{array}{rcl}
    \cxp(\fml{Y}) & ~:=~~ &
    \wcxp(\fml{Y})\land\forall(\fml{Y}'\subsetneq\fml{Y}).\neg\wcxp(\fml{Y}')\\
  \end{array}
\end{equation}  
A CXp can be viewed as a possible answer to a ``\tbf{Why Not?}''
question, i.e.\ why isn't the classifier's prediction a class other
than $c$?
A different perspective for a contrastive explanation is as the answer
to a \emph{How?} question, i.e.\ how to change the features so as to
change the prediction. In recent literature this alternative view has
been investigated under the name \emph{actionable
  recourse}~\cite{liu-fat19,alfano-fat20,valera-facct21,valera-corr20}.
It should be underlined that whereas AXp's correspond to prime
implicants of the boolean function $(\kappa(\mbf{x})=c)$ that are
consistent with some point $\mbf{v}\in\mbb{F}$, CXp are \emph{not}
prime implicates of function $(\kappa(\mbf{x})=c)$. Nevertheless, the
concept of \emph{counterexample} studied in formal
explainability~\cite{inms-aaai19} corresponds to prime implicates of
the function $(\kappa(\mbf{x})=c)$ (which are not restricted to be
consistent with some specific point $\mbf{v}\in\mbb{F}$).
\label{def:cex}

One important observation is that, independently of what $\kappa$
represents, the $\waxp$ and $\wcxp$ predicates (respectively defined
using~\eqref{eq:axp1} and~\eqref{eq:cxp1}) are \emph{monotone}\footnote{%
  Clearly, from the definition of $\waxp$ (resp.~$\wcxp$), if
  $\waxp(\fml{Z})$ (resp.\ $\wcxp(\fml{Z})$) holds, then 
  $\waxp(\fml{Z}')$ (resp.~$\wcxp(\fml{Z}')$) also holds for any
  superset $\fml{Z}'$ of $\fml{Z}$. If $\waxp(\fml{Z})$
  (resp.~$\wcxp(\fml{Z})$) does not hold,
  then $\waxp(\fml{Z}')$ (resp.~$\wcxp(\fml{Z}')$) also does not hold
  for any superset $\fml{Z}'$ of $\fml{Z}$.}.
This means that the tests for minimality (i.e.,
respectively~\eqref{eq:axp2a} and~\eqref{eq:cxp2a}) can be simplified
to:
\begin{equation} \label{eq:axp2b}
  \begin{array}{rcl}
    \axp(\fml{X}) & ~:=~~ &
    \waxp(\fml{X})\land\forall(t\in\fml{X}).\neg\waxp(\fml{X}\setminus\{t\})\\
  \end{array}
\end{equation}  
and,
\begin{equation} \label{eq:cxp2b}
  \begin{array}{rcl}
    \cxp(\fml{Y}) & ~:=~~ &
    \wcxp(\fml{Y})\land\forall(t\in\fml{Y}).\neg\wcxp(\fml{Y}\setminus\{t\})\\
  \end{array}
\end{equation}  
Observe that, instead of considering all possible subsets of
$\fml{X}$ (resp.~$\fml{Y}$), it suffices to consider the subsets
obtained by removing a single element from $\fml{X}$
(resp.~$\fml{Y}$).
This observation is 
at the core of the algorithms proposed in recent years for computing
AXp's and CXp's of a growing range of families of
classifiers~\cite{inms-aaai19,inms-nips19,nsmims-sat19,msgcin-nips20,iims-corr20,msgcin-icml21,ims-ijcai21,kwiatkowska-ijcai21,hiims-kr21,ims-sat21,hiicams-corr21}.

\begin{exmp} \label{ex:runex01a}
  For the DT in~\cref{fig:runex01}, consider the instance
  $((1,1,1),1)$ (i.e.\ if \tsf{Length} is \tsf{Short}, and
  \tsf{Thread} is \tsf{Follow-Up}, and \tsf{Author} is \tsf{Known},
  then predict \tsf{Reads}).
  The paths in $\fml{P}$ are: $\fml{P}=\{P_1,P_2\}$, with
  $P_1=\langle{1},3,4\rangle$ and $P_2=\langle{1},3,5,7\rangle$.
  The paths in $\fml{Q}$ are: $\fml{Q}=\{Q_1,Q_2\}$, with
  $Q_1=\langle1,2\rangle$ and $Q_2=\langle1,3,5,6\rangle$.
  Path $P_2$ is consistent with the instance; all other paths are
  inconsistent with the instance. The features associated with $P_2$
  are $\mrm{\Phi}(P_2)=\{1,2,3\}$, and the path literals associated
  with path $P_2$ are
  $\Lambda(P_2)=\{(x_1\in\{1\}),(x_2\in\{1\}),(x_3\in\{1\})\}$.
  Nevertheless, from~\cref{fig:runex01}, it is clear that
  $\fml{X}=\{1,3\}$ is a weak AXp. Indeed, if feature 2 (feature
  variable $x_2$) is allowed to take any value in its domain, then the
  prediction remains unchanged. Hence, it is the case that, with 
  $\mbf{x}=(x_1,x_2,x_3)$,  $\forall(\mbf{x}\in\{0,1\}^3).[(x_1)\land(x_3)]\limply\kappa(\mbf{x})$.
  Furthermore, $\fml{X}$ is minimal, since dropping either 1 or 3 from
  $\fml{X}$ will cause the weak AXp condition to fail.

  CXp's can be computed in a similar way. One can also observe that if
  either $x_1$ or $x_3$ are allowed to take any value from their
  domains, then there is an assignment that causes the prediction to
  change. Thus, $\fml{Y}_1=\{1\}$ or $\fml{Y}_2=\{3\}$ are CXp's of
  the given instance.
\end{exmp}

Given the definitions of AXp and CXp, and building on Reiter's seminal
work~\cite{reiter-aij87}, recent work~\cite{inams-aiia20} %
proved the following duality between minimal hitting sets\footnote{%
  Recall that a set $\fml{H}$ is a \emph{hitting set} of a set of
  sets $\fml{S}=\{S_1,\ldots,S_k\}$ if
  $\fml{H}\cap{S_i}\not=\emptyset$ for $i=1,\ldots,k$. $\fml{H}$ is a
  minimal hitting set of $\fml{S}$, if $\fml{H}$ is a hitting set of
  $\fml{S}$, and there is no proper subset of $\fml{H}$ that is also a
  hitting set of $\fml{S}$.}:
\begin{prop}[Minimal hitting-set duality between AXp's and CXp's]
  \label{prop:xpdual}
  AXp's are minimal hitting sets (MHSes) of CXp's and vice-versa.
\end{prop}
We refer to~\cref{prop:xpdual} as MHS duality between AXp's and CXp's.
The previous result has been used in more recent papers for enabling the
enumeration of
explanations~\cite{msgcin-icml21,ims-sat21,hiims-kr21}.
Furthermore, a consequence of \cref{prop:xpdual} is the following
result:

\begin{lem} \label{lm:xpdual2}
  Given a classifier function $\kappa:\mbb{F}\to\fml{K}$, defined on a
  set of features $\fml{F}$, a feature $i\in\fml{F}$ is
  included in some AXp iff $i$ is included in some CXp.
\end{lem}

Another minimal hitting-set duality result, different
from~\cref{prop:xpdual}, was investigated in earlier
work~\cite{inms-nips19}, and relates \emph{global} AXp's (i.e.\ not
restricted to be consistent with a specific point $\mbf{v}\in\mbb{F}$)
and counterexamples (see~\cpageref{def:cex}).

Given the above, the universe of \emph{explanation problems} is
defined by
$\mbb{E}_I=\{\fml{E}\,|\,\fml{E}=(\fml{M}, (\mbf{v},c)),
\fml{M}\in\mbb{M}, \mbf{v}\in\mbb{F}, c\in\fml{K}, c=\kappa(\mbf{v})\}$.
As a result, a tuple $(\fml{M},(\mbf{v},c))$ will allow us to
unambiguously represent the classification problem $\fml{M}$ for which
we will be computing AXp's and CXp's given the instance
$(\mbf{v},c)$.

\subsection{Summary of Notation} \label{ssec:sumup}
The notation used throughout the paper is summarized
in~\cref{tab:notation} (see~\cpageref{tab:notation}). (We should note
that some of the notation introduced in this paper has also been used
in a number of recent
works~\footnote{%
  See for example~\cite{inms-aaai19,nsmims-sat19,inms-nips19,ignatiev-ijcai20,msgcin-nips20,icshms-cp20,iims-corr20,msgcin-icml21,ims-ijcai21,ims-sat21,cms-cp21,hiims-kr21,msi-aaai22,hiicams-aaai22,iisms-aaai22}.}.)

\begin{table}[t]
  \begin{center}
    \scalebox{0.9125}{
      \renewcommand{\tabcolsep}{0.5em}
      \renewcommand{\arraystretch}{1.1225}
      \begin{tabular}{ccc} \toprule
        \textbf{Symbol} & \textbf{Definition} & \textbf{Meaning} \\ \midrule
        $\fml{F}$ & $\{1,\ldots,m\}$ & Set of features \\
        $\fml{D}_i$ & -- & Domain of feature $i$ \\
        $\mbb{D}$ & $\mbb{D}=(\fml{D}_1,\ldots,\fml{D}_m)$ & Range of
        domains, $\fml{D}_i=\mbb{D}(i)$ \\
        $\mbb{U}$ & $\mbb{U}=\cup_{i\in\fml{F}}\fml{D}_i$ & Union of domains \\
        $\mbb{F}$ & $\fml{D}_1\times\fml{D}_2\times\ldots\times\fml{D}_m$
        & Feature space \\
        $x_i$ & $x_i\in\fml{D}_i$ & Variable associated with feature
        $i$ \\
        $\fml{L}$ & $\fml{L}=(x_i\in{S_l})$ & Literal, with $S_l\subsetneq\fml{D}_i$ \\
        $\mbb{L}$ &
        $\mbb{L}=\{x_i\in{S_l}\}$ & Sets of literals,
        $i\in\fml{F}\land{S_l}\subsetneq\fml{D}_i$ \\
        $\fml{K}$ & $\{c_1,\ldots,c_K\}$ & Set of classes \\
        $\kappa$ & $\kappa:\mbb{F}\to\fml{K}$ & Classification function \\
        $\fml{I}$ & $\fml{I}=(\mbf{v},c)$ &
        Instance, with $\mbf{v}\in\mbb{F},c\in\fml{K}$ \\
        $\mbb{M}$ & $\mbb{M}=\{(\fml{F},\mbb{D},\mbb{F},\fml{K},\kappa)\}$ &
        Universe of classification problems \\
        $\mbb{E}_I$ & $\mbb{E}_I=\{(\fml{M},(\mbf{v},c))\}$ &
        Explanation problems,
        $\fml{M}\in\mbb{M},\mbf{v}\in\mbb{F},c\in\fml{K}$\\
        $\xi$ & $\xi:\mbb{F}\to\{0,1\}$ &
        Explanation function, $\xi(\mbf{x};\fml{Z},\ldots)$,
        $\fml{Z}\subseteq\fml{F}$\\
        $\mbb{E}_S$ & $\mbb{E}_S=\{(\fml{M},(\xi,\fml{Z},c))\}$ &
        XP problems,
        $\fml{M}\in\mbb{M},\fml{Z}\subseteq\mbb{F},c\in\fml{K}$,
        $\xi$: XP function\\
        $\mbb{E}_P$ & $\mbb{E}_P=\{(\fml{M},R_k)\}$ &
        Path-related XP problems,
        $\fml{M}\in\mbb{M},R_k\in\fml{R}$\\
        \midrule
        $\fml{T}$ & $\fml{T}=(V,E)$ & Decision tree, with nodes $V$ and
        edges $E$ \\
        $V$ & ${N}\cup{T}$ & Set of nodes in DT $\fml{T}$ \\
        $T$ & -- & Terminal nodes \\
        $\varsigma$ & $\varsigma:{T}\to\fml{K}$ & Class associated with
        each terminal node \\
        $N$ & -- & Non-terminal nodes \\
        $\phi$ & $\phi:{N}\to\fml{F}$ & Feature associated with
        each non-terminal node \\
        $\sigma$ & $\sigma:N\to2^{V}$ & Child nodes of non-terminal node\\
        $\varepsilon$ & $\varepsilon:{E}\to\mbb{L}$ &
        Lit.~$x_i\in{S_l}$ associated with edge $(r,s)$, $i=\phi(r)$
        \\
        \midrule
        $\fml{R}$ & -- & Paths in DT $\fml{T}$ \\
        $R_k$ & $R_k=\langle{r_1},\ldots,{r_l}\rangle$ & Path in DT
        $\fml{T}$, with tree nodes $r_1,\ldots,r_l$ \\
        $\tsf{seq}$ & -- & Sequence of tree nodes in $R_k\in\fml{R}$ \\
        $\tau$ & $\tau:\fml{R}\to{T}$ & Terminal node associated with
        path $R_k\in\fml{R}$\\
        $\mrm{\Phi}$ & $\mrm{\Phi}:\fml{R}\to2^{\fml{F}}$ & Features
        associated with path $R_k$ in $\fml{R}$ \\
        $\mrm{\Lambda}$ & $\mrm{\Lambda}:\fml{R}\to2^{\mbb{L}}$ &
        Literals associated with path $R_k$ in $\fml{R}$ \\
        $\rho$ & $\rho:\fml{F}\times\fml{R}\to2^{\mbb{U}}$ &
        Values of feature $i$ consistent with $R_k\in\fml{R}$\\
        $\rchi_I$ & $\rchi_I:\mbb{F}\times\fml{R}\to2^{\fml{F}}$ &
        Features that are inconsistent between instance and path \\
        $\rchi_P$ & $\rchi_P:\fml{R}\times\fml{R}\to2^{\fml{F}}$ &
        Features that are inconsistent between two paths \\
        \midrule
        $\fml{H}$ & -- & hard constraints/clauses \\
        $\fml{B}$ & -- & soft constraints/clauses \\
        \bottomrule
      \end{tabular}
    }
    \caption{Summary of the notation used throughout the paper}
    \label{tab:notation}
  \end{center}
\end{table}

\section{Duality of Explanations \& Path-Based Explanations}
\label{sec:ndual}

This section builds on recent work on duality of
explanations~\cite{inams-aiia20} (see~\cref{ssec:fxai}), and makes the
following contributions:
\begin{enumerate}
\item Explanations are generalized to explanation functions and
  conditions are outlined for minimal hitting-set (MHS) duality of
  explanations to hold in this more general setting.
\item Explanations are shown to respect a nesting property, with MHS
  duality holding for nested explanations.
\end{enumerate}
Furthermore, the section highlights how the results above can be used
for relating the computation of explanations of a DT with specific
tree paths instead of being instance-specific.

\subsection{Generalized Explanations \& Duality} \label{ssec:gxps}

\paragraph{Explanation functions.}
Besides prime implicants of discrete-valued functions, we can envision
a generalized explanation function $\xi:\mbb{F}\to\{0,1\}$, and
redefine both weak AXp's and weak CXp's, assuming such a generalized
explanation function\footnote{%
  Explanation functions have been studied in earlier
  work on formal explainability~\cite{hazan-aies19}.
}.
However, we impose that $\xi$ be parameterized on a selected
subset $\fml{Z}$ of the features, and also on other parameters which
we may leave undefined, or instead opt to include. This
parameterization will be represented by: $\xi(\mbf{x}; \fml{Z},
\ldots)$. For example, if $\xi$ represents a prime implicant that is
sufficient for the prediction, the parameterization (as discussed
in~\cref{ssec:fxai}) is the restriction of the conjunction of literals
to those features in $\fml{Z}$, where the literals are of the form
$x_i=v_i$ (i.e.\ the parameterization on $\fml{Z}$ serves to select
the coordinate values of $\mbf{v}$ associated with the features in
$\fml{Z}$).
However, it is possible to consider explanation functions that involve
other types of literals. Concretely, we will allow explanation
functions to involve literals of the form $(x_i\in{S_l})$.

Earlier work on formal explainability has most often considered as the
underlying explanation function the prime implicants of
discrete-valued functions, defined on arbitrary feature spaces. Hence,
given an instance $(\mbf{v},c)$, a possible definition of explanation
function is:
\begin{equation} \label{xpf:01}
  \xi(\mbf{x}; \fml{Z}, \mbf{v}) = \bigwedge_{i\in\fml{Z}}(x_i=v_i) 
\end{equation}
A clear limitation of using such prime implicants as the
explanation function is that we are equating each feature with a
\emph{single} value from its domain. For categorical features this
is not a major issue, but for ordinal features it can be too
restrictive.

In the case of DT paths, a viable explanation function is:
\begin{equation} \label{xpf:02}
  \xi(\mbf{x}; \fml{Z}, R_k, \mrm{\Lambda}(R_k)) =
  \bigwedge_{i\in\fml{Z},(x_i\in{S_l})\in\mrm{\Lambda}(R_k)}(x_i\in{S_l}) 
\end{equation}
(For simplicity, the parameterization on $R_k$ could be ignored,
since $R_k$ is in fact a constant when computing explanations that
relate with itself.) 

\begin{exmp} \label{ex:xpfs}
  For the running example in~\cref{fig:runex04}, consider the instance
  $((2,25,0),\tbf{Y})$, consistent with path
  $P_1=\langle1,3,4,6,8\rangle$. It is possible to conclude that a
  weak AXp is $\{1,2\}$. Observe that there are three features with
  literals in the path, i.e.\ $\{1,2,3\}=\fml{F}$, and that changing
  the value of feature 3 does not change the prediction; hence a weak
  AXp is $\{1,2\}$.
  Using the first explanation function above (see~\eqref{xpf:01}), one
  could claim that $(x_1=2)\land(x_2=25)$ suffices for the prediction.
  However, using the second explanation function above
  (see~\eqref{xpf:02}), one would be able to claim instead that
  $(x_1\in\{2,3\})\land(x_2\in\{\tn{MnA}..25\})$ suffices for the
  prediction. Clearly, the second explanation function is markedly
  more informative regarding which values suffice for the prediction.
  (Another extension that this paper does not investigate, is that
  $(x_1\in\{2..\tn{MxP}\})\land(x_2\in\{\tn{MnA}..25\})$ would also
  suffice for the prediction; this is the subject of future work.)\\
  The two explanation functions above exhibit important properties,
  including duality relationships; this will be discussed later in
  this section. Nevertheless, other explanation functions could be
  envisioned.
\end{exmp}

\paragraph{Generalizing AXp's and CXp's.}
Explanation functions serve to generalize weak AXp's and CXp's, as
follows:

\begin{defn}[$\waxp$ and $\wcxp$] \label{def:wxp}
  Given a classification problem $\cloper$, an explanation problem
  $\xpprob$, and an explanation function $\xi$,
  $\fml{Z}\subseteq\fml{F}$ is a weak abductive explanation if,
  \begin{equation} \label{eq:gwaxp}
    \begin{array}{lcr}
      \waxp(\fml{Z}) & ~:=~~ &
      \forall(\mbf{x}\in\mbb{F}).
      \xi(\mbf{x};\fml{Z},\ldots)\limply(\kappa(\mbf{x})=c)
      \\
    \end{array}
  \end{equation}
  $\fml{Z}$ is a weak contrastive explanation if,
  \begin{equation}  \label{eq:gwcxp}
    \begin{array}{lcr}
      \wcxp(\fml{Z}) & ~:=~~ &
      \exists(\mbf{x}\in\mbb{F}).
      \xi(\mbf{x};\fml{F}\setminus\fml{Z},\ldots)\land(\kappa(\mbf{x})\not=c)
      \\
    \end{array}
  \end{equation}
\end{defn}

For simplicity, the parameterization of $\waxp$ and $\wcxp$, on
$\xi$, $\fml{M}$ and $\fml{E}$, $\mbf{v}$, etc.\ is left implicit;
this will be clear from the context.

A consequence of the definition of $\waxp$ and $\wcxp$ is that we have
the following immediate result:
\begin{prop} \label{prop:exxp}
  For any $\fml{Z}\subseteq\fml{F}$, it is the case that,
  \[
  \waxp(\fml{Z})\leftrightarrow\neg\wcxp(\fml{F}\setminus\fml{Z})
  \]
\end{prop}

\begin{proof}
  $\waxp(\fml{Z})$ states that,
  \[
  \forall(\mbf{x}\in\mbb{F}).
  \xi(\mbf{x};\fml{Z},\ldots)\limply(\kappa(\mbf{x})=c)
  \]
  whereas, $\wcxp(\fml{F}\setminus\fml{Z})$ states that,
  \[
  \exists(\mbf{x}\in\mbb{F}).
  \xi(\mbf{x};\fml{Z},\ldots)\land(\kappa(\mbf{x})\not=c)
  \]
  which is the logical negation of $\waxp(\fml{Z})$. Thus, if
  $\waxp(\fml{Z})$ is true, then it must be the case that
  $\wcxp(\fml{F}\setminus\fml{Z})$ is false, and vice-versa.
\end{proof}

We will also need to consider sets of explanations and subset-minimal
explanations. Hence, the following definitions are used:
\begin{defn}[$\mbb{S}_{\tn{waxp}}$, $\mbb{S}_{\tn{wcxp}}$, $\mbb{A}$,
    $\mbb{C}$] \label{def:xps}
  Given $\fml{M}$ and $\fml{E}$, the following sets of sets are
  defined:
  \begin{equation} \label{eq:sdefs}
    \begin{array}{l}
      \mbb{S}_{\tn{waxp}}=\{\fml{Z}\in\fml{F}\,|\,\waxp(\fml{Z})\}\\[1.5pt]
      \mbb{S}_{\tn{wcxp}}=\{\fml{Z}\in\fml{F}\,|\,\wcxp(\fml{Z})\}\\
    \end{array}
  \end{equation}
  The set $\mbb{A}$ of the subset-minimal sets of $\mbb{S}_{\tn{waxp}}$
  represents the AXp's, i.e.
  \begin{equation}
    \mbb{A} =
    \{\fml{Z}\in\mbb{S}_{\tn{waxp}}\,|\,\forall(\fml{Z}'\subsetneq\fml{Z}).\neg\waxp(\fml{Z}')\}
  \end{equation}
  The set $\mbb{C}$ of the subset-minimal sets of $\mbb{S}_{\tn{wcxp}}$
  represents the CXp's, i.e.
  \begin{equation}
    \mbb{C} =
    \{\fml{Z}\in\mbb{S}_{\tn{wcxp}}\,|\,\forall(\fml{Z}'\subsetneq\fml{Z}).\neg\wcxp(\fml{Z}')\}
  \end{equation}
\end{defn}

Furthermore, we are especially 
interested in explanation functions that guarantee the monotonicity of
$\waxp$ and $\wcxp$. (As noted in~\cref{ssec:fxai}, the monotonicity
of these predicates enables devising more efficient algorithms for
computing AXp's and CXp's.)
Taking into consideration that, from~\eqref{eq:gwaxp}
and~\eqref{eq:gwcxp}, $\waxp$ and $\wcxp$ (and so also AXp and CXp)
are defined in terms of $\xi$, then we have the following definition:
\begin{defn} \label{def:monof}
  An explanation function $\xi$ is \emph{monotone-inducing} if, given
  $\xi$:
  \begin{enumerate}
  \item $\waxp(\emptyset)=0$ and $\wcxp(\emptyset)=0$;
  \item $\waxp(\fml{F})=1$ and $\wcxp(\fml{F})=1$;
  \item Moreover, it holds that, for $\fml{A}_0\subseteq\fml{F}$,
    \[
    \begin{array}{l}
      \waxp(\fml{A}_0)\limply\forall(\fml{A}_1\supseteq\fml{A}_0).\waxp(\fml{A}_1)
      \\[1.5pt]
      \wcxp(\fml{A}_0)\limply\forall(\fml{A}_1\supseteq\fml{A}_0).\wcxp(\fml{A}_1)
      \\
    \end{array}
    \]
    (i.e.\ if $\fml{A}_0$ is a weak AXp (resp.\ weak CXp) then any of
    its supersets (resp.~subsets) is also a weak AXp (resp.\ weak
    CXp).)
  \end{enumerate}
\end{defn}

\begin{exmp}
  The two explanation functions described in~\cref{ex:xpfs} are
  monotone-inducing. The fact that the explanation function associated
  with path literals is monotone-inducing will be pivotal for
  computing path explanations.
\end{exmp}

Given the above, we can now state the main result of this section.

\begin{prop} \label{prop:ndual}
  Given $\fml{M}$ and $\fml{E}$, $\xi$ is a monotone-inducing
  explanation function iff
  each element of $\mbb{A}$ is an MHS of the elements of $\mbb{C}$,
  and vice-versa. (This is to say that the AXp's of $\fml{E}$ are
  MHSes of the CXp's of $\fml{E}$ and vice-versa.)
\end{prop}

\begin{proof}
  The proof is split into cases:
  \begin{enumerate}[label=\roman*)] 
  \item If $\xi$ is  a monotone-inducing explanation function, then
    AXp's are MHSes of CXp's and vice-versa.\\
    Let $\fml{A}\in\mbb{A}$ be an AXp. Thus, $\fml{A}$ is a
    subset-minimal set such that \eqref{eq:gwaxp} holds. We claim that
    $\fml{A}$ must hit every CXp $\fml{C}$ of $\mbb{C}$. For the sake of
    contradiction, let us assume that this was not the case. Then, there
    would exist some $\fml{C}\in\mbb{C}$, not hit by $\fml{A}$. As a
    result, $\fml{F}\setminus\fml{A}$ would necessarily contain
    $\fml{C}$. Since $\fml{C}$ is a CXp, then \eqref{eq:gwcxp} would
    be satisfied. But this is impossible due to \cref{prop:exxp}; a
    contradiction.\\
    What remains to show is that the hitting set $\fml{A}$ is
    subset-minimal. Suppose it was not minimal. Then, we could create a
    minimal hitting set $\fml{A}'\subsetneq\fml{A}$, since $\fml{A}'$
    would hit all the CXp's in $\mbb{C}$, then $\eqref{eq:gwcxp}$ could
    be falsified by $\fml{F}\setminus\fml{A}'$. However, by
    \cref{prop:exxp}, then $\fml{A}'$ would satisfy \eqref{eq:gwaxp},
    and so $\fml{A}$ would not be minimal; a contradiction.\\
    A similar argument can be used to prove that each
    $\fml{C}\in\mbb{C}$ must hit every $\fml{A}\in\mbb{A}$.
  \item If AXp's are MHSes of CXp's and vice-versa, then $\xi$ is a
    monotone-inducing explanation function.\\
    This follows from the definition of monotone-inducing explanation
    function.\qedhere
  \end{enumerate}
\end{proof}

The result above can be related not only with recent results on the
duality of explanations~\cite{inms-nips19,inams-aiia20}, but also with
other well-known results on duality in different
areas~\cite{reiter-aij87,lozinskii-jetai03,slaney-ecai14}.
Finally, $\mbb{E}_{S}$ will be used to denote the set of explanation
problems given a classification problem $\fml{M}$, a subset $\fml{Z}$
of the features, and an explanation function $\xi$, parameterized on
$\fml{Z}$ and other parameters:
$\mbb{E}_{S}=\{\fml{E}\,|\,\fml{E}=(\fml{M},(\xi,\fml{Z},c)),\fml{M}\in\mbb{M},\fml{Z}\subseteq\fml{F},\xi\tn{~is an explanation function}\}$.

\subsection{Restricted Duality} \label{ssec:rdual}

This section investigates a restricted form of duality that results
from AXp's exhibiting what can be viewed as a property of
\emph{nesting}.
We consider an explanation problem $\fml{E}=(\fml{M},(\mbf{v},c))$ and
a monotone-inducing explanation function $\xi$.
Moreover, we let $\fml{Z}\subseteq\fml{F}$, with
$\fml{Z}=\{i_1,i_2,\ldots,i_M\}$, represent a weak AXp, i.e.
\begin{equation} \label{eq:rd-waxp}
  \forall(\mbf{x}\in\mbb{F}).\left(\xi(\mbf{x};\fml{Z},\ldots)\right)\limply(\kappa(\mbf{x})=c)
\end{equation}
Furthermore, let us define
$\mbb{F}_{\fml{Z}}=\fml{D}_{i_1}\times\fml{D}_{i_2}\times\cdots\times\fml{D}_{i_M}$,
$\mbb{D}_{\fml{Z}}=(\fml{D}_{i_1},\fml{D}_{i_2},\ldots,\fml{D}_{i_M})$,
and let $\iota:Z=\{1,\ldots,M\}\to\fml{Z}=\{i_1,\ldots,i_M\}$ be a
bijective function that maps coordinates 1 to $M$ into the actual
features' indices in $\fml{Z}$, i.e.\ $\iota(r)=i_r$, $r=1,\ldots,M$%
~\footnote{
  With a slight abuse of notation, we will use $\iota^{{-}1}(\fml{Z})$
  to denote the set $Z=\{1,\ldots,M\}$. We will also use
  $\iota(X)=\fml{Z}\subseteq\fml{Z}$ and
  $\iota^{{-}1}(\fml{X})=X\subseteq{Z}$ to represent the mappings of
  sets of features.}.
In addition, we introduce the predicate $\tn{prj}_{\fml{Z}}$, such
that $\tn{prj}_{\fml{Z}}(\mbf{x},\mbf{y})$ holds when\ $\mbf{y}$ is
the projection of $\mbf{x}$ on the coordinates specified by $\fml{Z}$,
i.e.\ $y_j=x_{\iota(j)}$ for all $j\in\fml{Z}$.
(Observe that $\tn{prj}_{\fml{Z}}$ is effectively parameterized on
$\iota$, but this is left implicit.)
In the concrete case of $\mbf{v}$, we define
$\mbf{u}\in\mbb{F}_{\fml{Z}}$, such that
$\tn{prj}_{\fml{Z}}(\mbf{v},\mbf{u})$ is true.
Moreover, define a binary classifier
$\kappa_{\fml{Z}}:\mbb{F}_{\fml{Z}}\to\{0,1\}$, as follows:
\begin{equation} \label{eq:rkappa}
  \kappa_{\fml{Z}}(\mbf{y})=\left\{
  \begin{array}{ll}
    1,~ & \tn{if~%
      $\forall(\mbf{x}\in\mbb{F}).\left[%
        \tn{prj}_{\fml{Z}}(\mbf{x},\mbf{y})%
        \land\xi(\mbf{x};\fml{Z},\ldots)\right]\limply(\kappa(\mbf{x})=c)$}\\[7pt]
    0,~ & \tn{if~
      $\exists(\mbf{x}\in\mbb{F}).\left[%
        \tn{prj}_{\fml{Z}}(\mbf{x},\mbf{y})%
        \land\xi(\mbf{x};\fml{Z},\ldots)\right]\land(\kappa(\mbf{x})\not=c)$}\\
  \end{array}
  \right.
\end{equation}
Observe that, by definition of $\kappa_{\fml{Z}}$, one can conclude
that $\kappa_{\fml{Z}}$ is independent of the features in
$\fml{F}\setminus\fml{Z}$.
Also note that $\kappa_{\fml{Z}}(\mbf{y})=1$ only if
$\kappa(\mbf{x})=c$ for all points $\mbf{x}\in\mbb{F}$ which project
into $\mbf{y}\in\mbb{F}_{\fml{Z}}$.

Given the definition of $\kappa_{\fml{Z}}$, we can now define both a
\emph{restricted} classification problem
$\fml{M}_{\fml{Z}}=(Z,\mbb{D}_{\fml{Z}},\mbb{F}_{\fml{Z}},\{0,1\},\kappa_{\fml{Z}})$,
and associated explanation problem
$\fml{E}_{\fml{Z}}=(\fml{M}_{\fml{Z}},(\mbf{u},1))$.
Clearly, for the explanation problem $\fml{E}_{\fml{Z}}$, it must be
the case that AXp's are the MHSes of the CXp's and
vice-versa~\cite{inams-aiia20}.
Furthermore, it is plain that the AXp's and CXp's of
$\fml{E}_{\fml{Z}}$ are subsets of $Z$.

\begin{exmp} \label{ex:runex02ab}
  Consider the DT from~\cref{fig:runex02}, and path
  $P_4=\langle1,2,5,9\rangle$, with $\varsigma(\tau(9))=\tbf{1}$.
  Let $\mbf{v}=(0,1,0,1)$, consistent with $P_4$.
  It is simple to conclude that $\fml{Z}=\{1,2,4\}$ is a weak AXp of
  $(\mbf{v},c)$. Moreover, we let
  $\iota(1)=1,\iota(2)=2,\iota(3)=4$, with $Z=\{1,2,3\}$.
  Given the above, we can define $\kappa_{\fml{Z}}$.

  \begin{center}
    \begin{tabular}{cccc} \toprule
      $y_1$ & $y_2$ & $y_3$   & $\kappa_{\fml{Z}}$ \\ \cmidrule(lr){1-3} \cmidrule(lr){4-4}
      0   &    0    &  0,1    &   \tbf{0} \\
      0   &    1    &  0      &   \tbf{0} \\
      0   &    1    &  1      &   \tbf{1} \\
      1   &    0,1  &  0,1    &   \tbf{1} \\
      \bottomrule
    \end{tabular}
  \end{center}
  (Observe that the use of ',' in the rows serves solely to collapse
  multiple rows into one.)
  We can now compute the AXp's/CXp's for the explanation problem
  $(\fml{M}_{\fml{Z}},(\mbf{u},\tbf{1}))$, with $\mbf{u}=(0,1,1)$,
  since $\tn{prj}_{\fml{Z}}((0,1,0,1),(0,1,1))$ holds.\\
  Given the explanation problem $\fml{E}_{\fml{Z}}$, and from the
  definition of $\kappa_{\fml{Z}}$ in the table above, an AXp is
  $\{2,3\}$. Clearly, the CXp's will be $\{2\}$ and $\{3\}$.
  We can now map the AXp's and CXp's of $\fml{E}_{\fml{Z}}$ to the
  features of $\fml{F}$. For the AXp, we get a set of features
  $\{2,4\}$, which we will later argue that it is also an AXp of
  $\fml{E}$. For the CXp's, we get $\{2\}$ and $\{4\}$, which we will
  shortly argue that are subsets of CXp's of $\fml{E}$.
  Further, we will later argue that these sets of features relate with
  abductive and contrastive explanations associated with path $P_4$.
\end{exmp}

Furthermore, given the definitions above, the following additional
results also hold. Given a set $\fml{Z}$, and the resulting restricted
binary classifier $\kappa_{\fml{Z}}$, there is a one to one mapping of
AXp's between those of $\fml{E}_{\fml{Z}}$ and those of $\fml{E}$;
however, each CXp of $\fml{E}_{\fml{Z}}$ is a subset of some CXp of
$\fml{E}$.

\begin{prop} \label{prop:rdaxp}
  $X\subseteq{Z}=\iota^{{-}1}(\fml{Z})$ is an AXp of
  $\fml{E}_{\fml{Z}}$ iff $\fml{X}=\iota(X)\subseteq\fml{Z}$ is an
  AXp of $\fml{E}$.
\end{prop}

\begin{proof}
  Let $\mbf{y}$ be a point in $\mbb{F}_{\fml{Z}}$ consistent with the
  features in $X$, and so exhibiting prediction 1. Then, by definition
  of $\kappa_{\fml{Z}}$, it is the case that the prediction of
  $\kappa$ for any $\mbf{x}$, such that $\tn{prj}(\mbf{x},\mbf{y})$
  holds, must be $c$. \\
  Similarly, let $\mbf{x}$ be a point in $\mbb{F}$ consistent with the
  features in $\fml{X}$, and so exhibiting prediction $c$. Then, by
  definition of $\kappa_{\fml{Z}}$, it is the case that the prediction
  of $\kappa_{\fml{Z}}$ for $\mbf{y}$, such that
  $\tn{prj}(\mbf{x},\mbf{y})$ holds, must be $1$. \\
  Since by hypothesis, $X$ is subset-minimal, then $\fml{X}=\iota(X)$
  is subset-minimal.
\end{proof}

\begin{prop} \label{prop:rdcxp}
  Each CXp $Y\subseteq{Z}=\iota^{{-}1}(\fml{Z})$ of
  $\fml{E}_{\fml{Z}}$ is such that $\fml{Y}=\iota(Y)\subseteq\fml{Z}$
  is a subset of some CXp of $\fml{E}$.
\end{prop}

\begin{proof}
  By definition, a CXp $Y\subseteq{Z}$ of $\fml{E}_{\fml{Z}}$ is a
  subset-minimal set of features in $Z$ which, if allowed to take any
  value from their domains, suffice to change the prediction.
  However, for $\fml{E}$ and given $\mbf{v}$, the features in
  $\fml{F}\setminus\fml{Z}$ take specific fixed values, dictated by
  $\mbf{v}$. Hence, some of these features may be required to change
  their values for the prediction of $\kappa$ to change from $c$ to
  some of the class in $\fml{K}\setminus\{c\}$. This follows from the
  definition of $\kappa_{\fml{Z}}$ in~\eqref{eq:rkappa}.
  As a result, it may be necessary to add to $\fml{Y}=\iota(Y)$
  additional features from $\fml{F}\setminus\fml{Z}$ so that the
  prediction changes.
  A minimal such set is a CXp of $\fml{E}$ and it represents a
  superset of $\fml{Y}$.
  Furthermore, no feature in $Y$ (and so in the resulting $\fml{Y}$)
  is redundant, since $Y$ is by definition a minimal set, even if the
  features not in $\fml{Z}$ are allowed to change their value.
\end{proof}

Furthermore, one additional result that is a consequence of the
previous results is that the relationships between AXp's and CXp's can
be stated in terms of AXp's and CXp's that are restricted to some
\emph{seed} set.

\begin{defn}[Set-restricted AXp's/CXp's]
  Let $\fml{E}$ be an explanation problem and let $\fml{Z}\in\fml{F}$
  be a weak AXp of $\fml{E}$.
  The $\fml{Z}$-set-restricted AXp's are the AXp's of
  $\fml{E}_{\fml{Z}}$ mapped by $\iota$ to the indices of features in
  $\fml{F}$, and it is represented by $\mbb{A}_{\fml{Z}}$.
  The $\fml{Z}$-set-restricted CXp's are the CXp's of
  $\fml{E}_{\fml{Z}}$ mapped by $\iota$ to the indices of features in
  $\fml{F}$, and it is represented by $\mbb{C}_{\fml{Z}}$.
\end{defn}

Given the definition of set-restricted AXp's/CXp's, we have the
following result:
\begin{prop} \label{prop:rset}
  Let $\fml{E}$ be an explanation problem, and let $\fml{Z}$ be a weak
  AXp for $\fml{E}$.
  Then, $\mbb{A}_{\fml{Z}}\subseteq\mbb{A}$, i.e.\ each
  $\fml{Z}$-set-restricted AXp is also an AXp.
  Furthermore, for each $\fml{W}\in\mbb{C}_{\fml{Z}}$, there exists
  $\fml{Y}\in\mbb{C}$ such that $\fml{W}\subseteq\fml{Y}$.
\end{prop}

\begin{proof}
  This result is a consequence of~\cref{prop:rdaxp,prop:rdcxp}.
\end{proof}

Furthermore, due to MHS duality between AXp's and CXp's, we can compute
all the AXp's of $\fml{E}=(\fml{M},(\mbf{v},c))$ that are contained in
$\fml{Z}$, by hitting set dualization using the CXp's in
$\mbb{C}_{\fml{Z}}$.

\begin{prop}  \label{prop:rdual}
  Each element of $\mbb{A}_{\fml{Z}}$ is a MHS of the elements in
  $\mbb{C}_{\fml{Z}}$ and vice-versa.
\end{prop}

\begin{proof}
  This result follows
  from~\cref{prop:xpdual,prop:rdaxp,prop:rdcxp,prop:rset}.
\end{proof}

Building on earlier results on duality of
explanations~\cite{inms-nips19,inams-aiia20},~\cref{prop:rdaxp,prop:rdcxp,prop:rdual,prop:rset}
uncover yet another dimension of the duality of explanations. This new
dimension reveals nesting properties of AXp's and CXp's.
\begin{cor}
  Let $\fml{W}\subseteq\fml{Z}\subseteq\fml{F}$.
  Then,
  \begin{enumerate}
  \item The $\fml{W}$-set-restricted AXp's
    are a subset of the $\fml{Z}$-set-restricted AXp's.
  \item Each $\fml{W}$-set-restricted CXp
    is a subset of some $\fml{Z}$-set-restricted CXp.
  \item The $\fml{W}$(or $\fml{Z}$)-set-restricted AXp's
    can be obtained from the $\fml{W}$(or $\fml{Z}$)-set-restricted
    CXp's
    by hitting set dualization, and vice-versa.
  \end{enumerate}
\end{cor}

\begin{exmp}
  Consider the running example from~\cref{fig:runex02}, and the
  instance $((0,1,1,1,1),\tbf{1})$ consistent with path
  $\langle1,2,5\rangle$, and defining an explanation problem
  $\fml{E}$.
  Consider the set of features $\fml{W}=\{2,4\}$. Clearly,
  $(x_2=1)\land(x_4=1)$ suffices for the prediction.
  We can also conclude that $\{2,4\}$ is an AXp. Moreover, $\{2\}$ and
  $\{4\}$ are $\fml{W}$-set-restricted CXp's, and MHS duality is
  observed.
  Now consider the set of features $\fml{Z}=\{2,3,4,5\}$. Clearly,
  $(x_2=1)\land(x_3=1)\land(x_4=1)\land(x_5=1)$ suffices for the
  prediction.
  In this case, careful analysis reveals that $\{2,4\}$ and $\{3,5\}$
  are $\fml{Z}$-set-restricted AXp's of $\fml{E}$ (and so also plain
  AXp's of $\fml{E}$).
  As a result,$\{2,3\}$, $\{2,5\}$, $\{3,4\}$, $\{4,5\}$ are
  $\fml{Z}$-set-restricted CXp's, and again MHS duality is observed.
  As can be observed, for the subset $\fml{W}$ of $\fml{Z}$, the AXp's
  are a subset of the AXp's of $\fml{Z}$, and each CXp restricted to
  $\fml{W}$ is a subset of the CXp's restricted to $\fml{Z}$.
  \\
  Another observation related with this example, is that although both
  $\{2,4\}$ and $\{3,5\}$ are AXp's of the original explanation
  problem, only the first one is clearly related with path $P_4$ of
  the DT.
\end{exmp}

\subsection{Path Explanations} \label{ssec:pxps}

Paths in DTs can contain literals for a subset of the features, and
can be consistent with many (possibly uncountable) points in feature
space.
The goal of this section is to investigate \emph{path explanations};
these represent sets of features such that~\eqref{eq:axp1} holds true
for \emph{any} instance consistent with some given path. We will
consider both abductive and contrastive path explanations, but we will
also investigate how enumeration of path explanations can be
instrumented.

We will now show how the results in~\cref{ssec:gxps,ssec:rdual} can
be used to formalize path explanations and subsequently the concept
of explanation redundancy in DT paths. First, \cref{ssec:gxps} showed
how to reason in terms of literals associated with paths and not
literals associated with points in feature space. Second,
\cref{ssec:rdual} showed how to analyze duality of explanations in the
case when sets of features (concretely those not tested in a given
path) are excluded from explanations.

Consider a path $R_k\in\fml{R}$ in a DT $\fml{T}$. We define the
following (path-based) explanation function, for
$\fml{Z}\subseteq\mrm{\Phi(R_k)}$:
\begin{equation} \label{eq:pathf}
\xi(\mbf{x}; \fml{Z}, \ldots) = \left[\bigwedge_{%
    \substack{j\in\fml{Z} \\[1.5pt]
      (x_j\in{S_l})\in\mrm{\Lambda(R_k)}}}
  (x_j\in{S_l}) \right]
\end{equation}
(Observe that this explanation function was first discussed
in~\cref{ex:xpfs}.)
As a result, given the proposed explanation function $\xi$, and as
outlined in~\cref{ssec:gxps} we can define both weak AXp's and
CXp's.

\begin{exmp}
  For the DT in~\cref{fig:runex02}, we consider path
  $P_4=\langle1,2,5,9\rangle$, and so with $c=\tbf{1}$. In this case,
  we have that $\mrm{\Phi}(P_4)=\{1,2,4\}$. For
  $\fml{Z}=\mrm{\Phi}(P_4)$ we get,
  \[
  \xi(\mbf{x}; \fml{Z}, \ldots) = \left[%
    (x_1\in\{0\})\land(x_2\in\{1\})\land(x_4\in\{1\})\right] \qedhere
  \]
\end{exmp}

In addition, path explanations are defined using the explanation
function proposed in~\eqref{eq:pathf}.
\begin{defn}[Path Explanations] \label{defn:pxp}
  A (weak) path AXp (resp.~CXp) is a (weak) AXp (resp.~CXp) given the
  explanation function \eqref{eq:pathf}.
\end{defn}
A path AXp will be denoted an \emph{abductive path explanation}
(APXp); a path CXp will be denoted a \emph{contrastive path
  explanation} (CPXp).
An explanation problem associated with a path in a DT is represented
by the tuple $(\fml{M},R_k)$.
Moreover, to distinguish the two kinds of explanations, those
introduced in~\cref{ssec:fxai} will be referred to as
\emph{instance-based} explanations. Observe that the key difference between
instance-based and path-based explanations are the literals used in
the definition of explanation. For instance-based explanations, the
literals are obtained from the point in feature space, whereas for
path-based explanations, the literals are obtained from the conditions
on features specified along the given path.

One alternative to path explanations would be to consider
instance-based AXp's and CXp's, as introduced in~\cref{ssec:fxai}, by
considering some point in feature space consistent with the given
path. However, such explanations offer information that might be too
specific.

\begin{exmp*}
  Consider a classification problem $\fml{M}$ with
  $\fml{F}=\{1,2,3\}$, $\mbb{D}=(\fml{D}_1,\fml{D}_2,\fml{D}_3)$, with
  $\fml{D}_1=\mbb{R}$ and $\fml{D}_2=\fml{D}_3=\mbb{N}_0$.
  Let the classifier be represented by a DT, with path
  $P_1=\langle1,2,3,4\rangle$ with
  $\mrm{\Lambda}(P_1)=\{(x_1\in[\tn{V}_{\tn{min}},10]),(x_2\in\{0,1,2,3,4\}),(x_3\in\{0,1\})$,
  and with $\phi(1)=3$, $\phi(2)=2$, $\phi(3)=1$, and with
  $\varsigma(\tau(P_2))=1$.
  Given the instance $(\mbf{v},c)=((0,0),1)$, let the AXp be
  $\{1,2\}$. The information that the conjunction
  $(x_1=0)\land(x_2=0)$ represents a sufficient condition for the
  prediction to be 1, is clearly less instructive than the information
  that $(x_1\in[\tn{V}_{\tn{min}},10])\land(x_2\in\{0,1,2,3,4\})$ also
  represents a sufficient condition for the prediction to be 1.
\end{exmp*}

Moreover, from~\cref{prop:rdual} one can readily conclude that path
AXp's and CXp's exhibit MHS duality.

\begin{prop} \label{prop:pdual}
  For a DT $\fml{T}$ with set of paths $\fml{R}$, and a path
  $R_k\in\fml{R}$, the APXp's of $R_k$ are the MHSes of the CPXp's of
  $R_k$ and vice-versa.
\end{prop}

\begin{proof}
  This result instantiates, in the case of paths in DTs, the result of
  \cref{prop:rdual} for restricted duality.
\end{proof}

Given the generalized definition of weak AXp in~\cref{def:wxp}, it is
plain that, for $\waxp$ defined using $\xi$,
$\waxp(\fml{Z})=1$ and
$\wcxp(\fml{Z})=1$
for
$\mrm{\Phi}(P_k)\subseteq\fml{Z}\subseteq\fml{F}$, and so
$\waxp(\fml{F})$ and $\wcxp(\fml{F})$ are true. (Observe that it is
assumed that the classifier is non-constant.)
It is also clear that $\waxp(\emptyset)=0$ and $\wcxp(\emptyset)=0$.
Finally, one can also conclude that if $\waxp(\fml{A}_0)=1$, then
$\waxp(\fml{A}_1)=1$ for $\fml{A}_1\supseteq\fml{A}_0$. The same
observation holds for weak CXp's.
As a result, by~\cref{def:monof} we can conclude that $\xi$ is
monotone-inducing. Thus, by~\cref{prop:ndual}, there is duality
between AXp's and CXp's given the explanation function $\xi$.

Despite $\xi$ representing an explanation function, we must also
understand how the explanations obtained with $\xi$ relate with the
explanations for the decision tree $\fml{T}$.

As shown next, we can relate path explanations and path explanation
duality with restricted duality.

\begin{prop} \label{prop:dtpathxp}
  Let $\fml{M}$ be the classification problem associated with DT
  $\fml{T}$, let $\fml{E}=(\fml{M},(\mbf{v},c))$ denote an
  explanation problem given some instance $(\mbf{v},c)$ consistent
  with $P_k\in\fml{R}$, and let $\xi$ be the explanation function
  associated with $P_k$, i.e.\ the conjunction of the literals in
  $\mrm{\Lambda}(P_k)$. Then,
  \begin{enumerate}
  \item Each APXp of $P_k$ is an AXp for $\fml{E}$ that is contained
    in $\mrm{\Phi}(P_k)$;
  \item Each CPXp of $P_k$ is a subset of some CXp for $\fml{E}$
    that is contained in $\mrm{\Phi}(P_k)$.
  \end{enumerate}
\end{prop}

\begin{proof}
  This result follows from the results in~\cref{ssec:gxps,ssec:rdual}
  and the results earlier in this section.
\end{proof}

\begin{exmp}
  We revisit \cref{ex:runex02ab}.
  Let $\fml{Z}=\mrm{\Phi}(P_4)=\{1,2,4\}$:
  $\mrm{\Lambda}(P_4)=\{(x_1\in\{0\}),(x_2\in\{1\}),(x_4\in\{1\})\}$
  Thus, the explanation function can be defined as follows,
  \[
  \xi(\mbf{x}; \fml{Z}, \ldots) =
  \bigwedge_{j\in\fml{Z},(x_j\in{S_l})\in\mrm{\Lambda}(P_4)}(x_j\in{S_l})
  \]
  Given the definition of path explanations (and so of (generalized)
  AXp's and CXp's), we can conclude that $\fml{X}=\{2,4\}$ is a path
  AXp for $P_4$. Moreover, $\fml{Y}_1=\{2\}$ and $\fml{Y}_2=\{4\}$ are
  path CXp's for $P_4$.
  It can be observed that $\fml{X}$ is an AXp for \emph{any} instance
  $(\mbf{v},\tbf{1})$, with $\mbf{v}$ consistent with $P_4$. However,
  both $\fml{Y}_1$ and $\fml{Y}_2$ are subsets of CXp's of possible
  instances $(\mbf{v},\tbf{1})$, consistent with $P_4$. For example,
  one can identify a CXp $\{2,3\}$ and also a CXp $\{4,5\}$.
\end{exmp}

The fact that path explanations can be related with AXp's restricted
to a specific set of features also signifies that not all
instance-based explanations represent path explanations.
This observation can be related with the distinction between
path-restricted and path-unrestricted explanations first studied
in~\cite{iims-corr20}.

\begin{exmp}
  For the running example shown in~\cref{fig:runex02}, we analyze the
  abductive explanations of path $P_4=\langle1,2,5,9\rangle$.
  Suppose we are given the instance is $(\mbf{v},c)=((0,1,1,1,1),1)$.
  An AXp is $\{3,5\}$. However, this explanation offers little insight
  to why the prediction is $\tbf{1}$ for the instances that are
  consistent with $P_4$. Using the nomenclature of earlier
  work~\cite{iims-corr20}, whereas $\{3,5\}$ is a path-unrestricted
  explanation, $\{2,4\}$ is a path-restricted explanation. In this
  paper, we consider only path explanations, and so $\{2,4\}$ is the
  only path AXp we are interested in computing.
\end{exmp}

\begin{table}[t]
  \begin{center}
    \renewcommand{\arraystretch}{1.075}
    \begin{tabular}{cccc} \toprule
      \textbf{Explanation} &
      \textbf{Definition} &
      \textbf{Literals used in $\bm\xi$} &
      \textbf{Features containing XP}
      \\ \midrule
      AXp, path-unrestricted &
      \eqref{eq:axp1}\eqref{eq:axp2b} & 
      Instance-based &
      $\fml{F}$
      \\ 
      CXp, path-unrestricted &
      \eqref{eq:cxp1}\eqref{eq:cxp2b} & 
      Instance-based &
      $\fml{F}$
      \\ 
      AXp, path-restricted &
      \eqref{eq:axp1}\eqref{eq:axp2b} & 
      Instance-based &
      $\mrm{\Phi}(R_k)$
      \\ 
      CXp, path-restricted &
      \eqref{eq:cxp1}\eqref{eq:cxp2b} & 
      Instance-based &
      $\mrm{\Phi}(R_k)$
      \\ 
      APXp &
      \cref{defn:pxp} &
      Path-based &
      $\mrm{\Phi}(R_k)$
      \\ 
      CPXp &
      \cref{defn:pxp} &
      Path-based &
      $\mrm{\Phi}(R_k)$
      \\ \bottomrule
    \end{tabular}
  \end{center}
  \caption{Types of explanations considered in the paper, both for
    some path $R_k\in\fml{R}$ and for any instance $(\mbf{v},c)$
    consistent with $R_k$}
  \label{tab:xptypes}
\end{table}

\cref{tab:xptypes} summarizes the kinds of explanations considered in
this paper. APXp's and CPXp's are introduced in this paper and, in
contrast with the other kinds of explanations, these are defined in
terms of literals obtained from a specific DT path.
Clearly, due to being instance-independent, path explanations offer a
simpler solution to represent explanations of decision trees that only
depend on the structure of the tree.
Furthermore, a few additional results are consequences of the results
presented in this section. For example, despite being based on a
different semantics, there is a one-to-one mapping between the APXp's
of $R_k$ and the path-restricted AXp's of any instance consistent with
$R_k$.
The sole difference between path-restricted AXp's and APXp's is that
the literals associated with APXp's are taken from the associated
path, whereas the literals associated with path-restricted AXp's are
obtained from a concrete instance (consistent with the path).
Finally, $\mbb{E}_{P}$ denotes the set of explanation problems given a
classification problem $\fml{M}$, and a path $R_k\in\fml{R}$ in a
decision tree $\fml{T}$:
$\mbb{E}_{P}=\{\fml{E}\,|\,\fml{E}=(\fml{M},R_k),\fml{M}\in\mbb{M},R_k\in\fml{R}\}$.
In the rest of the paper, $\fml{M}$ is assumed to be such that the
classification function $\kappa$ is monotone-inducing,

\section{Path Explanation Redundancy in Decision Trees}
\label{sec:rdt}

Given the definition of path explanations in \cref{ssec:pxps}, we can
formalize the concept of path explanation redundancy.

\begin{defn}[Explanation Redundant Path/Feature (XRP/XRF)]
  Given a DT $\fml{T}$, with set of paths $\fml{R}$, and a path
  $R_k\in\fml{R}$, $R_k$ is an \emph{explanation-redundant path} (or
  XRP) if $\mrm{\Phi}(R_k)$ does not represent a path AXp.
  Given a path AXp $\fml{X}$ for $R_k$, any feature
  $i\in\mrm{\Phi}(R_k)$ that is not included in $\fml{X}$ is a
  \emph{explanation-redundant feature} (or XRF).
\end{defn}

Feature redundancy is relative to a given APXp. Different APXp's can
yield different redundant features. Clearly, one can consider the
enumeration of APXp's to identify the set of features that is
never-redundant, by enumerating all APXp's for a given path, and
discarding any of the features deemed redundant for all of the
APXp's.

\subsection{Explanation Redundancy in Running Examples}

The following examples illustrate path explanations and explanation
redundancy.

\begin{exmp} \label{ex:runex02b}
  With respect to~\cref{ex:runex02}, with the DT shown
  in~\cref{fig:runex02}, let the target path be
  $P_1=\langle1,2,4,7,10,15\rangle$. (In this case there is only one
  point in feature space consistent with $P_1$, i.e.\ $(0,0,1,0,1)$.)
  We claim that $\fml{X}=\{3,5\}$ is a weak APXp, and so that $P_1$ is
  explanation-redundant.
  To prove the claim,
  we consider all the possible assignments to the other features:
  \begin{center}
    \begin{tabular}{ccccccccc} \toprule
      Feature & \multicolumn{8}{c}{Assignments} \\ \cmidrule(lr){1-1} \cmidrule(lr){2-9} 
      $x_1$ & 0 & 0 & 0 & 0 & 1 & 1 & 1 & 1 \\
      $x_2$ & 0 & 0 & 1 & 1 & 0 & 0 & 1 & 1 \\
      $x_4$ & 0 & 1 & 0 & 1 & 0 & 1 & 0 & 1 \\ \cmidrule(lr){1-1} \cmidrule(lr){2-9} 
      $\kappa(x_1,x_2,1,x_4,1)$ & 1 & 1 & 1 & 1 & 1 & 1 & 1 & 1 \\
      \bottomrule
    \end{tabular}
  \end{center}
  As can be concluded, as long as $x_3=1$ and $x_5=1$, then the
  prediction remains unchanged, since $\kappa(x_1,x_2,1,x_4,1)$ only
  takes value 1, for any assignment to $x_1,x_2,x_4$. 
  In this case, we can observe that a path-based explanation of size 5
  can be reduced to a (weak) abductive path explanation of size 2.
  Hence, there are (at least) 3 redundant features (namely features 1,
  2 and 4) out of a total of 5 features included in path $P_1$. The
  redundant features represent \tbf{60\%} of the original path length.
  As noted earlier, this DT was generated by the GOSDT/OSDT
  ((generalized scalable) optimal sparse decision trees)
  tools~\cite{rudin-nips19,rudin-icml20,rudin-corr21}, that
  specifically target interpretability.
\end{exmp}

\begin{exmp} \label{ex:runex03b}
  With respect to~\cref{ex:runex03}, with the DT shown
  in~\cref{fig:runex03}, let the target path be
  $P_3=\langle1,2,6,11,12\rangle$. (In this case there is only one
  point in feature space consistent with $P_3$: $(1,2,1,2)$.)
  It is easy to conclude that $\fml{X}=\{1,2,3\}$ is a weak  APXp, and
  so that $P_3$ is explanation-redundant. Indeed, if $x_4$ is allowed
  to take any value, then one can observe that the prediction remains
  unchanged.
\end{exmp}

\begin{exmp} \label{ex:runex04b}
  With respect to~\cref{ex:runex04}, with the DT shown
  in~\cref{fig:runex04}, let the target path be
  $P_2=\langle1,3,4,6,8\rangle$. (An example of a point in feature
  space consistent with $P_2$ is $(2,20,0)$.)
  It is easy to conclude that neither $x_1$ nor $x_2$ are allowed to take any
  value, whereas $x_3$ can be unrestricted. Hence, $\fml{X}=\{1,2\}$
  is a weak APXp. Since neither $x_1$ nor $x_2$ can be dropped, then
  $\{1,2\}$ is an APXp.
  The literals associated with the APXp are
  $\{(x_1\in\{2\}),(x_2\in\{\tn{MnA}..25\})\}$.
\end{exmp}

The examples above reveal that DTs taken from recent textbooks and
papers often exhibit path explanation redundancy. Moreover, the
examples above also show that DTs taken from papers that specifically
address the learning of optimal sparse DTs (which aim at
interpretability) can exhibit path explanation redundancy. In fact,
some examples confirm that there can exist paths in optimal sparse
decision trees for which there are more redundant features than
non-redundant features.
\cref{ssec:xpredp} offers a high-level perspective of the experimental
results, which reveal that path explanation redundancy in DTs is
indeed ubiquitous.
Afterwards, \cref{ssec:xpredt} proves that there are functions for
which path explanation redundancy is unavoidable, even in provably
size-minimal DTs.
These results and observations offer conclusive evidence regarding the
significance of filtering path explanation redundancy from DT
explanations, and further underline the critical importance of
efficient algorithms for computing explanations in DTs.

\subsection{Path Explanation Redundancy in Practice} \label{ssec:xpredp}

This section summarizes some key takeaways that can be drawn from the
experimental results (see~\cref{sec:res}), and which offer ample
practical justification for computing AXp's of DTs (and so finding and
filtering path explanation redundancy).

\paragraph{Path explanation redundancy in published examples.}
\cref{tab:dtrees} (see~\cpageref{tab:dtrees}) summarizes results on
path explanation redundancy for DTs included in representative
bibliography on DTs, namely textbooks and surveys\footnote{%
  A non-exhaustive list of references includes~%
\cite{moret-acmcs82,%
  breiman-bk84,%
  quinlan-bk93,%
  aha-ker97,%
  dzeroski-bk01,%
  rokach-bk08,%
  hebrard-cp09,%
  russell-bk10,%
  berthold-bk10,%
  flach-bk12,%
  zhou-bk12,%
  kotsiantis-air13,%
  alpaydin-bk14,%
  shalev-shwartz-bk14,%
  kelleher-bk15,%
  alpaydin-bk16,%
  valdes-naturesr16,%
  poole-bk17,%
  witten-bk17,%
  bramer-bk20,%
  zhou-bk21}.}.
The key observation is that path explanation redundancy is ubiquitous
in most DTs that have been used as examples in textbooks and surveys
over the years, going back to the inception of tree learning
algorithms.

\paragraph{Path explanation redundancy in learned DTs.}
\cref{tab:iai-res} and \cref{tab:iti-res} (see~\cpageref{tab:iai-res} and~\cpageref{tab:iti-res})
summarize the results
obtained with two different, publicly available tree learning tools,
namely Interpretable~AI (IAI)~\cite{bertsimas-ml17,iai}  and
ITI~\cite{utgoff-ml97}, on a large number of publicly
available datasets. IAI is a recent tool that specifically targets the
learning of \emph{interpretable DTs}. As can be concluded from the
results, for most datasets, the DTs learned by both algorithms exhibit
a significant percentage of explanation redundant paths. Moreover, for
paths that exhibit explanation redundancy, the number of redundant
literals can also be significant.

\paragraph{Large-scale path explanation redundancy.}
\cref{tab:inst} (see~\cpageref{tab:inst}) shows results for DTs
learned on more complex datasets (which are also publicly available).
For these examples, the number of explanation redundant features can
far exceed the number of explanation relevant features. Concretely for
some examples, the number of explanation-redundant features is more
than 7 times larger than the number of features used in an AXp.

\paragraph{Path explanation redundancy in optimal (sparse) DTs.}
\cref{tab:osdt} (see~\cpageref{tab:osdt}) shows results for DTs
learned with recently proposed algorithms that specifically target the
learning of optimal (and so indirectly \emph{interpretable}) DTs,
concretely~\cite{rudin-nips19,rudin-icml20,rudin-corr21} and
also~\cite{verwer-aaai19}.
As can be observed, the \emph{optimal} \emph{sparse} DTs shown in
earlier work exhibit a very significant number of redundant paths
(between 55\% and 75\%). For explanation-redundant paths, the
percentage of explanation-redundant features can reach 60\% (as
illustrated with~\cref{ex:runex02b} for the DT shown
in~\cref{fig:runex02}).

\subsection{Path Explanation Redundancy in Theory}  \label{ssec:xpredt}

This section proves two results. First, we prove that there exist
functions for which paths in smallest-size DTs will exhibit a number
of explanation-redundant literals that grow linearly with the number
of features.  Second, we prove that, for a DT to be irredundant, then
it must represent a generalized decision
function~\cite{hiicams-corr21}.

\paragraph{Optimal decision trees that exhibit redundancy.}
To simplify the statement of the main result, the following
definitions and assumptions are used. A dataset is consistent if for
any point $\mbf{x}$ in feature space,
contains an instance $(\mbf{x},c)$ for at most one class
$c\in\fml{K}$.
A classifier is exact if it correctly classifies any instance in
training data, and that training data is consistent. (A classifier is
perfect if it is exact and is of smallest
size~\cite{icshms-cp20,imsns-ijcai21}.) Furthermore, we assume that a
DT learner will not branch on variables that take constant value on
all the instances in training data that are consistent with the
already chosen literals.

\begin{prop} \label{prop:redxp}
 Consider the boolean function,
  \[
  f(x_1,x_2,\ldots,x_{m-1},x_m)=\bigvee_{i=1}^{m}x_{i} 
  \]
  Then, given any DT learning algorithm that learns an exact DT (one
  that correctly classifies any point in feature space), the learned
  DT contains a path with $m$ literals, for which there exists an AXp
  containing one single feature.
\end{prop}

Before proving the claim of~\cref{prop:redxp}, it should be observed
that a more general result could be stated, where the literals
$(z_i=v_i)$ for a non-boolean feature $i$ with domain $\fml{D}_i$
would replace the boolean literal $x_i$. However, the basic result
remains unchanged, as it reveals in theory the need for explaining
decision trees.

\begin{proof}
  The AXp's for function $f$ are easy to identify. For prediction 1,
  function $f$ has $m$ AXp's, namely $\fml{E}_{1,i}=\{i\}$ with
  $i=1,\ldots,m$. For prediction 0, function $f$ has one AXp, namely
  $\fml{E}_{0,1}=\{1,2,\ldots,m\}$. Any other weak AXp will not be
  subset-minimal.

  Next, we show that, no matter how the DT is constructed, there will
  always be at least one path that grows with $m$, and for which the
  size of the AXp is 1. Since the exercise is purely conceptual, we
  can assume that the dataset has size $2^m$, representing the truth
  table of function $f$.
  We construct a DT as follows. At each step, we let some adversary
  pick any variable, among the variables that have not yet been
  picked, and then show that only one option exists to continue the
  construction of the DT.
  Let the first variable be $x_{i_1}$, with $1\le{i_1}\le{m}$.
  For $x_{i_1}=1$, the prediction is 1, and so the DT must have a
  terminal node labeled 1.
  For $x_{i_1}=0$, the resulting function $f_{i_1}$ mimics $f$, but
  without variable $x_{i_1}$.
  Hence, we let again some adversary pick any variable among those not
  yet chosen. (Clearly, there is no reason to pick a variable already
  picked, since the function $f_{i_1}$ does not depend on $x_{i_1}$.)
  Let the new chosen variable variable be $x_{i_2}$. The analysis for
  $x_{i_2}$ is exactly the same as for $x_{i_1}$, and for $x_{i_2}=0$,
  we get a new function $f_{i_2}$.
  After analyzing all features, the resulting DT has $m$ paths with
  prediction $1$ and 1 path with prediction $0$. Thus,
  $\fml{P}=\{P_1,\ldots,P_m\}$ represents the paths with prediction 1,
  and $\fml{Q}=\{Q_1\}$ represents the path with prediction 0.
  Moreover, $P_m$ has length $m$, with literals
  $\langle{x_{i_1}}=0,x_{i_2}=0,\ldots,x_{i_{n-1}}=0,x_{i_{m}}=1\rangle$.
  (The resulting DT and path $P_m$ are shown
  in~\cref{fig:prop:redxp}.)
  For the instance $\{x_{i_m}=1\}\cup\{x_{i_j}=0,1\le{j}\le{m-1}\}$,
  path $P_m$ is consistent with the instance and it has $m$ literals.
  However, the AXp is $\{i_m\}$, denoting that $x_{i_m}=1$ suffices
  for the prediction. The analysis and conclusion is independent of
  the order of features chosen.
  \qedhere
\end{proof}

\begin{figure}
  \begin{center}
    \scalebox{0.95}{\tikzset{>=Stealth}
\begin{tikzpicture}[->,%
    EV/.style = {font=\footnotesize},
    node distance={2.45cm}, thin,
    term/.style = {draw, rectangle},
    nterm/.style = {draw, circle,minimum width=1cm,minimum height=1cm,inner sep=1.0pt}
  ]
  \node[nterm] (1)                    {$x_{i_1}$};
  \node[term]  (2) [below of=1]       {1};
  \node[nterm] (3) [right of=1]       {$x_{i_2}$};
  \node[term]  (4) [below of=3]       {1};
  \node[nterm] (5) [right = 2.75cm of 3] {$x_{i_{m-1}}$};
  \node[term]  (6) [below of=5]       {1};
  \node[nterm] (7) [right of=5]       {$x_{i_{m}}$};
  \node[term]  (8) [below of=7]       {0};
  \node[term]  (9) [right of=7]       {1};
  \draw[] (1) -- node [EV, left]  {$\in\{1\}$} (2);
  \draw[] (3) -- node [EV, left]  {$\in\{1\}$} (4);
  \draw[] (5) -- node [EV, left]  {$\in\{1\}$} (6);
  \draw[] (7) -- node [EV, left]  {$\in\{0\}$} (8);

  \draw[] (1) -- node [EV, above]  {$\in\{0\}$} (3);
  \draw[dashed] (3) -- node [EV, near start, above] {$\in\{0\}$} (5); 
  \draw[] (5) -- node [EV, above]  {$\in\{0\}$} (7);
  \draw[] (7) -- node [EV, above]  {$\in\{1\}$} (9);
\end{tikzpicture} }
  \end{center}
  \caption{DT construction for proof of~\cref{prop:redxp}}
  \label{fig:prop:redxp}
\end{figure}

Although the proof analyzed AXp's, for the proposed function and
resulting DT, the APXp's would be the same.

\begin{cor}
  There are DT classifiers, defined on $m$ features, for which an
  instance has an AXp of size 1, and the consistent path has length
  $m$, and so it can be made larger by a factor of $m$
  than the size of an AXp.
\end{cor}

\paragraph{Decision trees without path explanation redundancy.}
In this section we argue that for a DT not to exhibit redundancy then
it must correspond to an irreducible generalized decision function
(GDF)~\cite{hiicams-corr21}.
A GDF represents a multi-class classifier, with
$\fml{K}=\{c_1,\ldots,c_K\}$, where each class $c_j\in\fml{K}$ is
classified by a boolean function $\kappa_j$, such that set of boolean
functions $\kappa_i$ respects the following statement:
\begin{equation} \label{eq:gdf}
  \forall(\mbf{x}\in\mbb{F}).\sum_{j=1,\ldots,K}\kappa_j(\mbf{x})=1 
\end{equation}
A GDF is represented by $\fml{G}=\{\kappa_1,\ldots,\kappa_K\}$.
A DNF GDF is a set of boolean classifier functions $\kappa_j$, where
each $\kappa_j$ is represented by a disjunctive normal form (DNF)
formula.
A minimal DNF GDF is a set of boolean classifier functions where each
$\kappa_j$ is represented by an irredundant DNF formula $\varphi_j$,
i.e.\ no term in the DNF $\varphi_j$ is redundant, and no literal in
any term of the DNF $\varphi_j$ is redundant.

\begin{lem} \label{lem:dtgdf1}
  A minimal DNF GDF $\fml{G}$ corresponds to a function representation
  where each term of each DNF for some $\kappa_j$ is a prime implicant
  of $\kappa_j$.
\end{lem}

\begin{proof}
  Suppose a term $t_r$ of the DNF representation of $\kappa_j$ that is
  not a prime implicant of $\kappa_j$. Then, $t_r$ can be simplified
  to $t'_r$, such that $t'_r\equiv{t_r}$. But then the DNF
  representation of $\fml{G}$ would not be minimal; a contradiction.
\end{proof}

\begin{lem} \label{lem:dtgdf2}
  A DT does not exhibit path explanation redundancy iff the
  conjunction of the literals in each path to prediction $c\in\fml{K}$
  represents a prime implicant for the boolean function
  $\kappa(\mbf{x})=c$.
\end{lem}

\begin{proof}
  If the conjunction of the literals in each path is a prime implicant
  for the boolean function $\kappa(\mbf{x})=c$, then no path in the DT
  exhibits path explanation redundancy; otherwise some path would not
  represent a prime implicant, as assumed by hypothesis.\\
  If the DT exhibits no path path explanation redundancy,
  then we can represent the function $\kappa(\mbf{x})=c$ by a
  disjunction of the conjunctions of the literals in the paths
  predicting $c$. Each disjunct must be irreducible; otherwise we
  would be able to also reduce the explanation for some path.
\end{proof}

\begin{prop} \label{prop:dtgdf}
  A DT $\fml{T}$ does not exhibit path explanation redundancy
  iff there exists a minimal DNF GDF $g$ that is equivalent to
  $\fml{T}$.
\end{prop}

\begin{proof}
  This result follows from \cref{lem:dtgdf1} and \cref{lem:dtgdf2}.
\end{proof}

It should be underscored that minimal DNF GDFs represent a fairly
restricted class of decision sets (DS)~\cite{leskovec-kdd16}, namely
minimal DSs exhibiting no overlap~\cite{ipnms-ijcar18}. The complexity
of computing a minimal DS without overlap is not known, but it is
conjectured to be hard for $\stwop$~\cite{ipnms-ijcar18}. Furthermore,
it is well-known that decision trees represent a far less expressive
language than DSs~\cite{rivest-ml87}.
Thus, most functions represented by DSs cannot be represented by DTs
that correspond to minimal DNF GDFs.
As a result,
\cref{prop:dtgdf} offers further evidence that one should expect
decision trees to be extremely unlikely to exhibit no path explanation
redundancy in practice.

\section{Computing Path Explanations in Decision Trees}
\label{sec:xdt}

Although the finding of formal explanations is computationally hard
for a number of ML
models~\cite{inms-aaai19,barcelo-nips20,ims-ijcai21,ims-sat21,marquis-kr21},
it has been shown that for DTs, one AXp can be computed in polynomial
time~\cite{iims-corr20,hiims-kr21}%
\footnote{%
  Furthermore, recent work has shown that computing a smallest size
  AXp is NP-hard~\cite{barcelo-nips20}.}. 
Moreover, and in the case of CXp's, recent work has shown that the
total number of CXp's is polynomial, and that their enumeration runs
in polynomial time~\cite{hiims-kr21}.
This section refines these earlier results in several ways, proposing
simpler and more efficient algorithms.
More importantly, the section specifically considers algorithms for
path explanations, as opposed to instance-based explanations.
Nevertheless, the changes for computing (path restricted/unrestricted)
AXp's/CXp's are straightforward.

We start by offering a simple approach supporting the rationale for
polynomial-time explainability of DTs.
Afterwards, we propose a simplified variant of an existing
algorithm~\cite{iims-corr20}, and then detail a propositional logic
Horn encoding for the problem of computing one AXp/APXp. The proposed 
encoding allows us to exploit existing algorithms for reasoning about 
propositional Horn formulas.

\subsection{Abductive Path Explanations by Explicit Path Analysis}
\label{ssec:mhs}

Since our goal is to compute a path explanation, we consider a
concrete path $P_k$, a partition $(\fml{P},\fml{Q})$ of the set of
paths $\fml{R}$ in $\fml{T}$, with $P_k\in\fml{P}$ and with prediction
$c=\varsigma(\tau(P_k))$ being the same for all paths in $\fml{P}$,
and with the paths in $\fml{Q}$ yielding a prediction other than $c$. 
Let $F_k=\mrm{\Phi}(P_k)$, i.e.\ the set of features $i$ associated
with the edges of $P_k$. (For computing a path-unrestricted AXp, we
would set $F_k=\fml{F}$.)
Recall from~\cref{ssec:dts} (and~\cref{tab:notation}) that $\rchi_I$
and $\rchi_P$ represent, respectively, the set of features that are
inconsistent between either a point or a path and some other path.
For computing AXp's (i.e.\ given an instance) we will be interested in
$\rchi_I(\mbf{v},Q_l)$ for each $Q_l\in\fml{Q}$.
For computing APXp's (i.e.\ given a path) we will be interested in
$\rchi_P(P_k,Q_l)$ for each $Q_l\in\fml{Q}$.
Since the analysis is similar, we will focus on APXp's.

For the prediction to be guaranteed not to change, due to $Q_l$, at
least one feature in $\rchi_P(P_k,Q_l)$ must not be allowed to change
value.
Thus, one APXp is a (subset-)minimal hitting set of the sets
$\rchi_P(P_k,Q_l)$ ranging over the paths $Q_l$ in $\fml{Q}$.
Furthermore, it is well-known that one subset-minimal hitting set can
be computed in polynomial time~\cite{gottlob-sjc95}. For example, we
can construct a set $\fml{X}$ containing the features in
$\cup_{Q_l\in\fml{Q}}\,\rchi_P(P_k,Q_l)$, and then iteratively remove one
feature from $\fml{X}$ while the resulting set $\fml{X}$ is still a
hitting set of all the $\rchi_P(P_k,Q_l)$. (For AXp's, we would use a
similar argument, but considering instead the sets
$\rchi_I(\mbf{v},Q_l)$.)

\begin{exmp} \label{ex:runex01c}
  Consider again the DT shown in~\cref{fig:runex01}.
  For $P_2=\langle1,3,5,7\rangle$, we have that
  $\rchi_P(P_2,Q_1)=\{1\}$ and $\rchi_P(P_2,Q_2)=\{3\}$.
  Thus, the only minimal hitting set is $\{1,3\}$, and so this
  represents the only APXp for $P_2$.
  Similarly, we could consider the instance $(\mbf{v},c)=((1,1,1),1)$,
  with $\rchi_I((1,1,1),Q_1)=\{1\}$ and $\rchi_I((1,1,1),Q_1)=\{3\}$,
  and so also obtain an AXp $\{1,3\}$. Clearly, since $P_2$ is
  consistent with $\mbf{v}=(1,1,1,1)$, all the APXp's of $P_2$ should
  be AXp's of $\mbf{v}$.
\end{exmp}

\begin{exmp} \label{ex:runex02c}
  Consider again the DT shown in~\cref{fig:runex02}.
  For path $P_1=\langle1,2,4,7,10,15\rangle$, we have that
  $\rchi_P(P_1,Q_1)=\{3\}$,
  $\rchi_P(P_1,Q_2)=\{5\}$, and
  $\rchi_P(P_1,Q_3)=\{2,5\}$.
  Clearly, the only minimal hitting set is $\{3,5\}$ and so this
  represents the only APXp for $P_1$.
  Similarly, we could consider the instance
  $(\mbf{v},c)=((0,0,1,0,1),1)$ and so we would also obtain the AXp
  $\{3,5\}$.
\end{exmp}

\begin{exmp} \label{ex:runex03c}
  With respect to the DT shown in~\cref{fig:runex03}, and for path
  $P_3=\langle1,2,6,11,12\rangle$,
  we have that
  $\rchi_P(P_3,Q_1)=\{3\}$,
  $\rchi_P(P_3,Q_2)=\{2\}$,
  $\rchi_P(P_3,Q_3)=\{1,3\}$, and
  $\rchi_P(P_3,Q_4)=\{1\}$.
  Clearly, the only minimal hitting set is $\{1,2,3\}$ and so this
  represents the only APXp for path $P_3$.
\end{exmp}

The previous examples of explanations can also be viewed as
\emph{path-restricted} AXp's.
The following example reveals the differences to path-unrestricted
AXp's~\cite{iims-corr20}.

\begin{exmp} \label{ex:runex01c1}
  Let us consider the example of~\cref{fig:runex01},
  and path $Q_1=\langle1,2\rangle$.
  In this case, we want to keep the paths $P_1$ and $P_2$
  inconsistent. Hence, $\rchi_P(Q_1,P_1)=\{1\}$ and
  $\rchi_P(Q_1,P_3)=\{1\}$, and so the only APXp is $\{1\}$.
  Let us now consider the $((0,1,0),0)$, which is consistent with path
  $Q_1$.
  In this case we get $\rchi_I((0,1,0),P_1)=\{1,2\}$ and
  $\rchi_I((0,1,0),P_2)=\{1,3\}$, and so the AXp's for the
  instance-based explanation problem are $\{1\}$ and $\{2,3\}$.
  Observe that, given the instance, one can understand the AXp
  $\{2,3\}$. However, in terms of explaining the sufficient conditions
  for the prediction to remain the same, given the values specified by
  the path, then it is clear that $\{1\}$ represents the only
  explanation of interest.
\end{exmp}

An apparent drawback of computing explanations with the algorithm
outlined in this section is that all DT paths must be explicitly
listed, and these require worst-case quadratic space given the number
of nodes in the DT. 
The next sections investigate alternative approaches, which perform
better in practice.

\subsection{Abductive Path Explanations by Tree Traversal}
\label{ssec:trav}

One approach to avoid the issue with explicit path representation is
to iteratively traverse the DT as features are removed from the AXp,
and checking whether the paths to predictions other than $c$ remain
inconsistent. This approach was first described in~\cite{iims-corr20}.
Here, we describe a simpler variant.

\cref{alg:pathdel} summarizes the main steps of the proposed approach
for computing an APXp for a concrete path $P_k$.
(For computing a path-restricted AXp given an instance, we would just
identify and use the same algorithm.)
\begin{algorithm}[t]
%
\SetKwFunction{FindAPXp}{{\sc FindAPXp}} 
\SetKwFunction{ChildNodes}{ChildNodes}
\SetKwFunction{Term}{IsTerminal}
\SetKwFunction{Pred}{Prediction}
\SetKwFunction{HasP}{HasPaths}
\SetKwFunction{Univ}{Universal}
\SetKwFunction{Feats}{Features}
\SetKwFunction{Root}{root}
\SetKwFunction{ChkPaths}{{\sc ExistsConsistentQPath}} 

\Func \FindAPXp{$\fml{T},\fml{F},P_k$} \\ 
\Indp
{
  \lnlset{cna:1}{1}
  $\fml{U}\gets\fml{F}\setminus\mrm{\Phi}(P_k)$\tcp*[r]{Features
    $\not\in$ path also $\not\in$ APXp}
  \lnlset{cna:2}{2}
  \ForEach{$i\in\mrm{\Phi}(P_k)$}{
    \lnlset{cna:3}{3}
    $\fml{U}\gets\fml{U}\cup\{i\}$\tcp*[r]{Tentatively drop $i$ from APXp}
    \lnlset{cna:4}{4}
    \If{$\ChkPaths(P_k,\fml{U},\Root(\fml{T}))$}{
      \lnlset{cna:5}{5}
      $\fml{U}\gets\fml{U}\setminus\{i\}$\tcp*[r]{Feature $i$ must be included in APXp}
    }
  }
  \lnlset{cna:6}{6}
  \Return{$\fml{F}\setminus\fml{U}$}\tcp*[r]{Return APXp}
}
\Indm
\BotBlankLine
%

  \caption{Computing one path explanation (or path-restricted AXp)}
  \label{alg:pathdel}
\end{algorithm}
As shown, for APXp's (and also for path-restricted AXp's given some
instance), the features that are not tested in $P_k$ are declared
universal and added to a working set $\fml{U}$.
(For computing a path-unrestricted AXp, the set $\fml{U}$ would be
initialized to $\emptyset$.)
The remaining features are analyzed one at a time. Each feature $i$ is
tentatively declared universal and~\cref{alg:pathdel} then invokes a
path traversal procedure (see~\cref{alg:chkpaths}) for deciding
whether there can exist a consistent path to a prediction other than
$c$. If such a path exists, then the feature is added back to the set
of features that must not be declared universal.
\begin{algorithm}[t]
  

%


%
\SetKwFunction{ChkPaths}{{\sc ExistsConsistentQPath}} 
\SetKwFunction{CDCall}{ChkDown} %
\SetKwFunction{ChildNodes}{ChildNodes}
\SetKwFunction{Term}{IsTerminal}
\SetKwFunction{Pred}{Prediction}
\SetKwFunction{HasP}{HasPaths}
\SetKwFunction{Univ}{Universal}
\SetKwFunction{Feat}{Feature}
\SetKwFunction{Consistent}{Consistent}
\SetKwFunction{Feature}{Feature}
\SetKwFunction{EdgeSet}{EdgSet}
\SetKwFunction{PathSet}{PathSet}
\SetKwData{hasp}{hasp}
\SetKwData{edgeok}{edgeok}

\Func \ChkPaths{$P_k,\fml{U},r$} \\ 
\Indp
{
  \lnlset{cp:1}{1}
  \If(\tcp*[f]{Decide return value if terminal}){$r\in{T}$}{
    \lnlset{cp:2}{2}
    \If{$\varsigma(r)\not=\varsigma(\tau(P_k))$}{
      \lnlset{cp:3}{3}
      \Return{$\true$}\tcp*[r]{Found consistent path to $d\not=c$}
    }
    \lnlset{cp:4}{4}
    \Else{
      \lnlset{cp:5}{5}
      \Return{$\false$}\tcp*[r]{Not a consistent path to $d\not=c$}
    }
  }
  \lnlset{cp:6}{6}
  $i \gets \phi(r)$\tcp*[r]{Pick feature associated with
    node $r$}
  \lnlset{cp:7}{7}
  \ForEach{$s\in\sigma(r)$}{
    \tcp*[h]{Recursively traverse child nodes, as long as}

    \tcp*[h]{edge values exhibit consistent values}
    
    \lnlset{cp:8}{8}
    \If{$(i\in\fml{U})\lor(\rho(i,P_k)\cap\varepsilon((r,s))\not=\emptyset)$}{
      \lnlset{cp:9}{9}
      \If{$\ChkPaths(P_k,\fml{U},s)$}{
        \lnlset{cp:10}{10}
        \Return{$\true$}\tcp*[r]{Found consistent path to $d\not=c$}
      }
    }
  }
  \lnlset{cp:11}{11}
  \Return{$\false$}\tcp*[r]{Unable to find consistent path to $d\not=c$}
}
\Indm
\BotBlankLine
%

  \caption{Checking consistent path to prediction in $\fml{K}\setminus\{c\}$}
  \label{alg:chkpaths}
\end{algorithm}

As can be observed, \cref{alg:pathdel} iteratively removes features
from the set of features associated with $P_k$. For each feature $i$,
\cref{alg:pathdel} then checks whether there exists some path in
$\fml{Q}$ that can be made consistent. If such path exists, then $i$
must be kept in the set of features sufficient for the prediction.
Clearly, the tree traversal algorithm essentially tests whether the
remaining set of features is still a hitting set of the paths in
$\fml{Q}$, and so shares similarities with the algorithm described
in \cref{ssec:mhs}, without exhibiting the drawback of explicitly
enumerating all the paths in the DT.

The operation of both~\cref{alg:pathdel,alg:chkpaths} is summarized
using the following example.

\begin{exmp}
  We analyze the DT shown in~\cref{fig:runex04}.
  Our goal is to find an APXp for path $P_3=\langle1,3,5\rangle$.
  Let us assume that ~\cref{alg:pathdel} adds feature $1$ to set
  $\fml{U}$, i.e.\ feature 1 is removed from the APXp being
  constructed.
  It is clear that, when the tree traversal is at node 1 (i.e.\ the
  root), it will take the left branch, and reach a terminal node with
  a prediction other that $\tbf{N}$; hence feature 1 must be removed
  from $\fml{U}$ and added to the APXp being constructed.
\end{exmp}

The running time of \cref{alg:pathdel,alg:chkpaths} is clearly
polynomial on the size of the DT.
Given a path  $R_k\in\fml{R}$, the algorithm analyzes the decision
tree for each feature. Hence the running time is in
$\fml{O}(|\fml{T}|\times|\fml{F}|)$.
Moreover,  Algorithm~\ref{alg:pathdel} can be run over all paths
$\fml{R}$ in the DT $\fml{T}$. In this case, the running time is thus
in $\fml{O}(|\fml{T}|\times|\fml{F}|\times|\fml{R}|)$.
As the experimental results demonstrate, the running time of the
algorithm is negligible (when compared with the time to learn the DT)
almost without exception.

\subsection{Abductive Path Explanations by Propositional Horn Encoding}
\label{ssec:hornit}

One additional solution for computing an AXp is to formulate the
problem as finding a minimal correction subset (MCS) of a
propositional Horn formula, and then exploiting existing efficient
algorithms~\cite{amms-sat15,msimp-jelia16}. 
Besides enabling efficient implementations, the Horn encoding allows
for integrating constraints that restrict the feature space by
disallowing points in feature space that violate those
constraints~\cite{rubin-corr21}. As long as the added constraints are
also Horn, and this is the case with propositional rules, then the
complexity of reasoning is unaffected.

The general approach is to formulate a Horn optimization problem
composed of a set of hard clauses $\fml{H}$ (which must be satisfied)
and a set of soft clauses $\fml{B}$ (which ideally one would like to
satisfy). Moreover, we seek an assignment to the variables that finds
a subset-maximal set of clauses from $\fml{B}$ that are satisfied
while satisfying the hard clauses. This problem can be solved in
polynomial time in the case of Horn
formulas~\cite{amms-sat15,msimp-jelia16}, based on the fact that Horn
formulas can be decided in linear time~\cite{minoux-ipl88}. (Observe 
that finding a cardinality maximal solution, i.e.\ solving the MaxSAT
problem for Horn formulas, is NP-hard~\cite{jaumard-ipl87} and the
respective decision problem is NP-complete. Similar results have been
obtained for computing a smallest AXp~\cite{barcelo-nips20}.) 

It is straightforward to devise a naive Horn encoding that mimics the
explicit path representation outlined above in~\cref{ssec:mhs}.
The dropping of each feature from the set of features in a APXp is
represented by a boolean variable $p_i$. Ideally one would prefer to
pick all features, and so the soft clauses are:
$\fml{B}=\{(p_i)\,|\,i\in\mrm{\Phi}(P_k)\}$. Moreover, for each set
$\rchi_P(P_k,Q_l)$, with $Q_l\in\fml{Q}$, representing the features
that are pairwise inconsistent between $Q_l$ and $P_k$,
one creates a Horn clause $(\lor_{i\in\rchi_P(P_k,Q_l)}\neg{p_i})$. 
Clearly, such an encoding does not offer any clear advantage with
respect to the minimal hitting set algorithm, besides exploiting
efficient Horn reasoners, since both approaches are based on explicit
enumeration of all tree paths.
A different approach, which avoids the worst-case quadratic
representation on the size of the DT, is to devise a Horn encoding
that bypasses the step of enumerating the paths in the DT. The main
goal of this section is to propose such an encoding.

Let us consider a path $P_k\in\fml{P}$, with prediction
$c\in\fml{K}$.
Moreover, let $\fml{Q}$ denote the paths yielding a prediction other
than $c$. Since the prediction is $c$, then any path in $\fml{Q}$ has
some feature for which the allowed values are inconsistent with
$\mbf{v}$. We say that the paths in $\fml{Q}$ are \emph{blocked}.
(To be clear, a path is blocked as long as some of its literals are
inconsistent.)

For each feature $i$ associated with some node of path $P_k$,
introduce a variable $u_i$. $u_i$ denotes whether feature $i$ is
deemed \emph{universal}, i.e.\ feature $i$ is not included in the
APXp that we will be computing.
(Our goal is to find a subset maximal set of features that can be
deemed universal, such that all the paths resulting in a prediction
other than $c$ remain blocked. Alternatively, we seek to find a
subset-minimal set of features to declare non-universal or fixed, such
that paths with a prediction other than $c$ remain blocked.)
Furthermore, for each DT node $r$, introduce variable $b_r$, denoting
that all sub-paths from node $r$ to any terminal node labeled
$d\in\fml{K}\setminus\{c\}$ must be blocked, i.e.\ some literal in the
sub-path must remain inconsistent. (Our goal is to guarantee that all
paths to terminal nodes labeled $d\in\fml{K}\setminus\{c\}$ remain
blocked even when some variables are allowed to become universal.)

We proceed to describe the proposed Horn encoding. Here, we opt to
describe first the Horn encoding for computing a path-unrestricted
AXp. Afterwards, we describe the Horn encoding for computing a
path-restricted AXp (or an APXp).

First, for a path-unrestricted AXp, the soft clauses $\fml{B}$ are
given by, $\{(u_i)\,|\,i\in\fml{F}\}$.
In contrast, for APXp's and for path-restricted AXp's, the soft
clauses $\fml{B}$ are given by, $\{(u_i)\,|\,i\in\mrm{\Phi}(P_k)\}$.
In both cases, the goal is that one would ideally want to declare
universal as many features as possible (among those that one can
pick), thus minimizing the size of the explanation. (As noted above,
we will settle for finding subset-maximal solutions.)
We describe next the hard constraints $\fml{H}$ for representing
consistent assignments to the $u_i$ variables.
For \emph{path-unrestricted} AXp's~\cite{iims-corr20}, the hard 
constraints are created as follows:
\begin{enumerate}[label=\textbf{H\arabic*.},ref=\textbf{H\arabic*}] 
  \label{enum:cases}
\item For the root node $r$, add the constraint $\top\limply{b_r}$.\\
  (The root node must be blocked.)
  \label{enum:cases:stp01}
\item For each terminal node $r$ with prediction $c$, add the
  constraint $\top\limply{b_r}$.\\
  (Each terminal node with prediction $c$ is also blocked. Also,
  observe that this condition is on the node, not on the path.)
  \label{enum:cases:stp02}
\item For each terminal node $r$ with prediction
  $d\in\fml{K}\setminus\{c\}$, add the constraint
  ${b_r}\limply\bot$.\\
  (Terminal nodes predicting $d\not=c$ cannot be blocked. Also, and as
  above, observe that this condition is on the node, not on the path.)
  \label{enum:cases:stp03}
\item For a node $r$ associated with feature $i$, and connected to the
  child node $s$, such that the edge value(s) is(are)
  \emph{consistent} with the value of feature $i$ in $\mbf{v}$, add
  the constraint $b_r\limply{b_s}$.\\
  (If all sub-paths from node $r$ must be blocked, then all sub-paths
  from node $s$ must all be blocked, independently of the value taken
  by feature $i$.)
  \label{enum:cases:stp04}
\item For a node $r$ associated with feature $i$, and connected to the
  child node $s$, such that the edge value(s) is(are)
  \emph{inconsistent} with the value of feature $i$ in $\mbf{v}$, add
  the constraint $b_r\land{u_i}\limply{b_s}$.\\
  (In this case, the blocking condition along an edge inconsistent
  with the value of feature $i$ in $\mbf{v}$ is only relevant if the
  feature is deemed universal.)
  \label{enum:cases:stp05}
\end{enumerate}

\begin{exmp} \label{ex:runex02d1}
  For the running example of~\cref{fig:runex02}, let
  $(\mbf{v},c)=((0,0,1,0,1),1)$.
  As dictated by the proposed Horn encoding, two sets of variables are
  introduced. The first set represents the variables denoting whether
  a feature is universal, corresponding to 5 variables:
  $\{u_1,u_2,u_3,u_4,u_5\}$. The second set represents the variables
  denoting whether a node is blocked, corresponding to 15 variables:
  $\{b_1,b_2,b_3,b_4,b_5,b_6,b_7,b_8,b_9,b_{10},b_{11},b_{12},b_{13},b_{14},b_{15}\}$.
  The resulting propositional Horn encoding contains hard ($\fml{H}$)
  and soft ($\fml{B}$) constraints, and it is organized as shown
  in~\cref{ex:tab01}.
  \begin{table}[t]
    \begin{center}
    \begin{tabular}{cc} \toprule 
      Hard constraint type & Horn clauses \\ \toprule
      \Cref{enum:cases:stp01} & $\{(b_1)\}$ 
      \\ \midrule
      \Cref{enum:cases:stp02} &
      $\{(b_3),(b_9),(b_{11}),(b_{13}),(b_{15})\}$
      \\ \midrule
      \Cref{enum:cases:stp03} &
      $\{(\neg{b_6}),(\neg{b_{12}}),(\neg{b_{14}})\}$
      \\ \midrule
      \Cref{enum:cases:stp04} &
      $\begin{array}{l}
        \{\,
        (b_1\limply{b_2}),(b_2\limply{b_4}),(b_4\limply{b_7}),(b_5\limply{b_8}),\\
        ~~(b_7\limply{b_{10}}),(b_8\limply{b_{13}}),(b_{10}\limply{b_{15}})
        \,\}
      \end{array}$
      \\ \midrule
      \Cref{enum:cases:stp05} &
      $\begin{array}{l}
        \{\,
        (b_1\land{u_1}\limply{b_3}),(b_2\land{u_2}\limply{b_5}),
        (b_4\land{u_3}\limply{b_6}), \\
        ~~(b_5\land{u_4}\limply{b_9}),(b_7\land{u_4}\limply{b_{11}}),
        (b_8\land{u_5}\limply{b_{12}}),\\
        ~~(b_{10}\land{u_5}\limply{b_{14}})
        \,\}
      \end{array}$
      \\
      \toprule
      Soft constraints, $\fml{B}$ & $\{(u_1),(u_2),(u_3),(u_4),(u_5)\}$ \\
      \bottomrule
    \end{tabular}
  \end{center}
    \caption{Horn clauses for the DT of~\cref{fig:runex02} for
      computing one AXp with $(\mbf{v},c)=((0,0,1,0,1),1)$}
    \label{ex:tab01}
  \end{table}
   
  It is easy to see that, if $u_1=u_2=u_3=u_4=u_5=1$, then $\fml{H}$
  is falsified. Concretely, 
  $(b_1)\land%
  (b_1\limply{b_2})\land%
  (b_2\limply{b_4})\land({u_3})\land%
  (b_4\land{u_3}\limply{b_6})\land(\neg{b_6})\nentails\bot$.
  The goal is then to find a maximal subset $\fml{S}$ of $\fml{B}$
  such that $\fml{S}\cup\fml{H}$ is consistent. (Alternatively, the
  algorithm finds a minimal set $\fml{C}\subseteq\fml{B}$, such that
  $\fml{B}\setminus\fml{C}\cup\fml{H}$ is consistent.)
  For this concrete example, one such minimal set is obtained by
  picking $u_1=u_2=u_4=1$ and $u_3=u_5=0$, and by setting
  $b_1=b_2=b_3=b_4=b_5=b_7=b_8=b_9=b_{10}=b_{11}=b_{13}=b_{15}=1$
  and
  $b_6=b_{12}=b_{14}=0$. Hence, all clauses are satisfied, and so
  $\{3,5\}$ is a weak AXp. An MCS
  extractor~\cite{mshjpb-ijcai13,mpms-ijcai15,mipms-sat16} would
  confirm that $\{3,5\}$ is subset-minimal, and so it is an AXp.
\end{exmp}

Similarly, we can consider \emph{path-restricted}
AXp's~\cite{iims-corr20} (or APXp's).
As noted earlier, in this case, the soft clauses $\fml{B}$ are given
by $\{(u_i)\,|\,i\in\mrm{\Phi}(P_k)\}$.
The previous encoding can be modified to reflect the computation of a 
path-restricted AXp (and also an APXp), where a point
$\mbf{v}\in\mbb{F}$ is no longer assumed. The changes to the previous
encoding are as follows:
\begin{enumerate}[label=\textbf{H$'$\arabic*.},ref=\textbf{H$'$\arabic*},start=4]
  \label{enum:cases2}
\item For a node $r$ associated with feature $i$, and connected to the
  child node $s$, such that the edge value(s) is(are)
  \emph{consistent} with the value of feature $i$ tested in path
  $P_k$, or if feature $i$ is not included in $\Phi(P_k)$, then add
  the constraint $b_r\limply{b_s}$.\\
  (If all sub-paths from node $r$ must be blocked, then all sub-paths
  from node $s$ must all be blocked, independently of the value taken
  by feature $i$.)
  \label{enum:cases2:stp04}
\item For a node $r$ associated with feature $i$, and connected to the
  child node $s$, such that the edge value(s) is(are)
  \emph{inconsistent} with the consistent values of feature $i$ in
  path $P_k$, then add the constraint $b_r\land{u_i}\limply{b_s}$.\\
  (In this case, the blocking condition along an edge inconsistent
  with the consistent values of feature $i$ along $P_k$ is only
  relevant if the feature is deemed universal.)
  \label{enum:cases2:stp05}
\item For each feature $i$ not included in $\Phi(P_k)$, add the unit
  clause $(u_i)$.
  (Features not tested along $P_k$ must not be included in the
  explanation.)
  \label{enum:cases2:stp06}
\end{enumerate}
Concretely, the features not in the path must \emph{not} be included
in a path-restricted AXp or in an APXp.

\begin{exmp} \label{ex:runex02d2}
  For the running example of~\cref{fig:runex02}, and again with
  $(\mbf{v},c)=((0,0,1,0,1),1)$, the path consistent with $\mbf{v}$ is
  $P_1=\langle1,2,4,7,10,15\rangle$. We use the same sets of variables
  as in~\cref{ex:runex02d1}. 
  The resulting propositional Horn encoding contains hard ($\fml{H}$)
  and soft  ($\fml{B}$) constraints, and consists of the following
  constraints shown in~\cref{ex:tab02}.
  \begin{table}[t]
    \begin{center}
    \begin{tabular}{cc} \toprule 
      Hard constraint type & Horn clauses \\ \toprule
      \Cref{enum:cases:stp01} & $\{(b_1)\}$ 
      \\ \midrule
      \Cref{enum:cases:stp02} &
      $\{(b_3),(b_9),(b_{11}),(b_{13}),(b_{15})\}$
      \\ \midrule
      \Cref{enum:cases:stp03} &
      $\{(\neg{b_6}),(\neg{b_{12}}),(\neg{b_{14}})\}$
      \\ \midrule
      \Cref{enum:cases2:stp04} &
      $\begin{array}{l}
        \{\,
        (b_1\limply{b_2}),(b_2\limply{b_4}),(b_4\limply{b_7}),(b_5\limply{b_8}),\\
        ~~(b_7\limply{b_{10}}),(b_8\limply{b_{13}}),(b_{10}\limply{b_{15}})
        \,\}
      \end{array}$
      \\ \midrule
      \Cref{enum:cases2:stp05} &
      $\begin{array}{l}
        \{\,
        (b_1\land{u_1}\limply{b_3}),(b_2\land{u_2}\limply{b_5}),
        (b_4\land{u_3}\limply{b_6}), \\
        ~~(b_5\land{u_4}\limply{b_9}),(b_7\land{u_4}\limply{b_{11}}),
        (b_8\land{u_5}\limply{b_{12}}),\\
        ~~(b_{10}\land{u_5}\limply{b_{14}})
        \,\}
      \end{array}$
      \\ \midrule
      \Cref{enum:cases2:stp06} &
      $\emptyset$ -- all features in path
      \\
      \toprule
      Soft constraints, $\fml{B}$ & $\{(u_1),(u_2),(u_3),(u_4),(u_5)\}$ \\
      \bottomrule
    \end{tabular}
    \end{center}
    \caption{Horn clauses for the DT of~\cref{fig:runex02} for
      computing one APXp with $P_1=\langle1,2,4,7,10,15\rangle$}
    \label{ex:tab02} 
  \end{table}
  i.e.\ the difference are the clauses forcing some features not to be
  included in explanations.
\end{exmp}

\begin{exmp} \label{ex:runex02d23}
  We use again the running example of~\cref{fig:runex02}, but now we
  consider the path $P_4=\langle1,2,5,9\rangle$, e.g.\ by picking for
  example the instance $(\mbf{v},c)=((0,1,1,1,0),1)$.
  As before, we use the same sets of variables as
  in~\cref{ex:runex02d1}. 
  The resulting propositional Horn encoding contains hard ($\fml{H}$)
  and soft  ($\fml{B}$) constraints, and consists of the following
  constraints shown in~\cref{ex:tab03}.
  \begin{table}[t]
    \begin{center}
    \begin{tabular}{cc} \toprule 
      Hard constraint type & Horn clauses \\ \toprule
      \Cref{enum:cases:stp01} & $\{(b_1)\}$ 
      \\ \midrule
      \Cref{enum:cases:stp02} &
      $\{(b_3),(b_9),(b_{11}),(b_{13}),(b_{15})\}$
      \\ \midrule
      \Cref{enum:cases:stp03} &
      $\{(\neg{b_6}),(\neg{b_{12}}),(\neg{b_{14}})\}$
      \\ \midrule
      \Cref{enum:cases2:stp04} &
      $\begin{array}{l}
        \{\,
        (b_1\limply{b_2}),
        (b_2\limply{b_5}),
        (b_4\limply{b_6}),
        (b_4\limply{b_7}), \\
        ~~(b_5\limply{b_9}), 
        (b_7\limply{b_{11}}),
        (b_8\limply{b_{12}}),
        (b_8\limply{b_{13}}),\\
        ~~(b_{10}\limply{b_{14}}),
        (b_{10}\limply{b_{15}})
        \,\}
      \end{array}$
      \\ \midrule
      \Cref{enum:cases2:stp05} &
      $\begin{array}{l}
        \{\,
        (b_1\land{u_1}\limply{b_3}),
        (b_2\land{u_2}\limply{b_4}), \\
        ~~(b_5\land{u_4}\limply{b_8}),
        (b_7\land{u_4}\limply{b_{10}}) 
        \,\}
      \end{array}$
      \\ \midrule
      \Cref{enum:cases2:stp06} &
      $\{(u_3),(u_5)\}$
      \\
      \toprule
      Soft constraints, $\fml{B}$ & $\{(u_1),(u_2),(u_3),(u_4),(u_5)\}$ \\
      \bottomrule
    \end{tabular}
    \end{center}
    \caption{Horn clauses for the DT of~\cref{fig:runex02} for
      computing one APXp with $P_4=\langle1,2,5,9\rangle$}
    \label{ex:tab03}
  \end{table}
  i.e.\ the difference are the clauses forcing some features not to be
  included in explanations.
  \\
  As can be observed, any solution will set $u_3=u_5=1$. It must also
  be the case that $u_2=0$ and $u_4=0$. However, we can safely set
  $u_1=1$. Hence the APXp is $\{2,4\}$.
\end{exmp}

Finally, and as hinted above, we observe that the same formulation can
be used for computing a smallest AXp, by finding a cardinality-minimal
instead of a subset-minimal set of true variables $u_i$. It is
well-known that both problems, i.e.\ computing a smallest explanation
and solving Horn MaxSAT, are hard for
NP~\cite{jaumard-ipl87,barcelo-nips20}.
Thus, we have the following result.
\begin{prop}
  Each maximum cost solution of the Horn formulation yields a
  cardinality-minimal AXp.
\end{prop}
Observe that, by enumerating Horn MaxSAT solutions, we are able to
enumerate smallest AXp's.

\subsection{Contrastive Path Explanations}
\label{ssec:cxps}

Given a path $P_k\in\fml{R}$ in a DT, with prediction $c\in\fml{K}$,
one can consider \emph{any} instance consistent with $P_k$, and
compute a CXp using the polynomial-time algorithm recently proposed
in~\cite{hiims-kr21}. Since CXp's in DTs are constructed by path
analysis, being limited to at most one per path with a different
prediction, this immediately implies that their number is limited to
the number of paths.
Furthermore, a CXp associated with some path $Q_l$ is declared
redundant if some other path $Q_s$ reveals a CXp with a stricter
subset of the features provided by $Q_l$. Thus, we can conclude that
the features associated with each CXp of an instance $(\mbf{v},c)$
consistent with $P_k$ must correspond to the CXp associated with some
path $Q_s$.

Nevertheless, instance-based CXp's can contain features that are not
even tested in path $P_k$. Given a path $P_k$ with prediction $c$, a
path $Q_l$ to a prediction other than $c$ may test a feature not
tested in $P_k$. Hence, a CXp could report features not tested in
$P_k$.
Furthermore, using hitting set dualization for enumerating abductive
explanations will require adapting existing algorithms to filter out
features not tested in path $P_k$.
This section details a more direct solution, one that takes into
account both the generalized (literals obtained from those used in
tree) and the restricted (literals obtained from those used in path
$P_k$) aspects of path explanations in DTs.
The computed explanations will be path contrastive explanations
(CPXp's), and so subsets of actual CXp's for a concrete instance.

The proposed algorithm is based on earlier work~\cite{hiims-kr21},
with a few minor modifications:
\begin{enumerate}
\item Analyze each path $Q_l$ in $\fml{Q}$ with prediction in
  $\fml{K}\setminus\{c\}$.
\item Traverse the path $Q_l$, ignore features that are not tested
  along $P_k$, and record in $\fml{C}$ the features with literals
  inconsistent with those in $P_k$, i.e.\ for a given feature $i$,
  $\rho(i,Q_l)\cap\rho(i,P_k)=\emptyset$.
\item Aggregate the computed sets of features $\fml{C}$, and keep the
  ones that are subset-minimal.
\end{enumerate}

The previous algorithm runs in worst-case time
$\fml{O}(m\times|\fml{Q}|)$. Furthermore, given recent results on the
number of CXp's in DTs~\cite{hiims-kr21}, the number of CPXp's
is bounded by $|\fml{Q}|$.

\begin{exmp}
  With respect to the running example shown in~\cref{fig:runex02}, and
  path $P_4=\langle1,2,5,9\rangle$, the algorithm would execute as
  follows:
  \begin{itemize}
  \item[] $Q_1$: $\fml{C}_1=\{2\}$.
  \item[] $Q_2$: $\fml{C}_2=\{2,4\}$; drop $\fml{C}_2$.
  \item[] $Q_4$: $\fml{C}_3=\{4\}$.
  \end{itemize}
  Hence, the reported CPXp's would be: $\{\{2\},\{4\}\}$.
\end{exmp}

The fact that all CXp's can be enumerated in polynomial-time, offers
an alternative to compute a smallest AXp that differs from the one
proposed in~\cref{ssec:hornit}. Indeed, a minimum-size (or
minimum-cost) hitting set of the CXp's represents a smallest AXp.
\begin{prop} \label{prop:minhs}
  A minimum-cost hitting set of the CXp's is a smallest AXp, and 
  vice-versa.
\end{prop}
Thus, smallest AXp's can also be enumerated by enumerating minimum-cost
hitting sets. (Also, there is a dual result
regarding~\cref{prop:minhs}, its practical uses are unclear, since
the number of AXp's may be exponentially large.)

\subsection{Enumeration of Path Explanations} \label{ssec:exdt}

The enumeration of multiple (or all) abductive or contrastive
explanations can help human decision makers to develop a better
understanding for the reasons of some prediction, but also to gain a
better perception of the underlying classifier.
Recent work~\cite{darwiche-ijcai18} compiles a decision function into a
Sentential Decision Diagram (SDD), from which the enumeration of
AXp's can be instrumented. Moreover, from a compiled representation of
the AXp's, each AXp can be reported in polynomial time. The downside
is that these representations are worst-case exponential in the size
of the original ML model. Furthermore, it is unclear how compilation
could be applied to the case of DTs.
Another line of work for computing AXp's is based on iterative
entailment checks using an NP-oracle~\cite{inms-aaai19}, with
enumeration studied in more recent
work~\cite{inams-aiia20,ims-sat21}.
For classifiers for which AXp's and CXp's can be computed in
polynomial time, a number of alternative algorithms have also been
studied in recent work~\cite{msgcin-nips20,msgcin-icml21,hiims-kr21},
which guarantee that a single NP (in fact SAT) oracle call is required
for each computed AXp or CXp.
This section develops a solution for the enumeration of APXp's which
builds on existing approaches for the enumeration of minimal hitting
sets (MHSes). Despite a number of differences, the approach can be
related with recent work~\cite{hiims-kr21}.
A key insight is that exactly one call to a SAT oracle is required for
each computed AXp, even if the computed AXp's are subset-minimal. This
can in general be formalized as follows.

\begin{prop}
  If the computation of one AXp and one CXp runs in polynomial time,
  then there is an algorithm for the simultaneous enumeration of AXp's
  and CXp's that requires one SAT oracle call per computed AXp or CXp.
\end{prop}

\begin{proof}
  Consider the propositional encodings proposed
  in~\cref{ssec:hornit}. We build $\fml{H}$ iteratively as follows.
  For each picked set of features representing an AXp, add a negative
  clause to $\fml{H}$, preventing the same AXp from being re-computed.
  For each picked set of features representing a CXp, add a positive
  clause to $\fml{H}$, requiring some of the non-picked features to be
  picked the next time. At each iteration, run a SAT oracle on
  $\fml{H}$. If the picked set of features is a weak AXp, then extract
  an AXp, and use it to add another clause to $\fml{H}$. If the picked
  set of features is a weak CXp, then extract a CXp, and use it to add
  another clause to $\fml{H}$. The algorithm iterates while there are
  additional AXp's or CXp's to enumerate.
\end{proof}

Moreover, given that the number of CPXp's is linear on the size of the
decision tree (see~\cref{ssec:cxps}, and given that an APXp must be a
minimal hitting set of all the CPXp's (see~\cref{prop:pdual}), then we
can construct the hypergraph of all CPXp's, which we can implement in
polynomial time, and then exploit an existing hypergraph transversal
(or hitting set dualization)
approach~\cite{bailey-icdm03,kavvadias-jgaa05,khachiyan-dam06,liffiton-jar08}.
Although some of these algorithms resort to NP oracles at each
enumeration step~\cite{liffiton-jar08} with promising experimental
results, in theory each incremental step can be implemented in
quasi-polynomial time~\cite{khachiyan-jalg96}.

The examples in~\cref{ssec:cxps} illustrate the use of hitting set
dualization for computing APXp's from the complete set of CPXp's.

\paragraph{A SAT encoding.} \label{par:satenc}
We consider the case of enumeration of APXp's from CPXp's; the case
concerning the enumeration of (path (un)restricted) AXp's from CXp's
would be similar.
We associate a boolean variable $p_i$ with each feature
$i\in\fml{F}$, denoting (if equal to 1) whether the feature is picked
to be included in some APXp.
The CNF formula $\fml{H}$ is created as follows:
\begin{enumerate}[label=\textbf{C\arabic*.},ref=\textbf{C\arabic*}] 
  \label{enum:ccases}
\item For each CPXp $\fml{Y}=\{j_1,\ldots,j_r\}$, add a (positive)
  clause $(p_{j_1}\lor\ldots\lor{p_{j_r}})$ for $\fml{H}$, i.e.\ each
  APXp must hit all the CPXp's. \label{enum:ccases:stp01}
\end{enumerate}
Furthermore, each time an APXp $\fml{X}=\{i_1,\ldots,i_s\}$ is
computed, a new (negative) clause is added
$(\neg{p_{i_1}}\lor\neg{p_{i_2}}\lor\ldots\lor\neg{p_{i_s}})$ to
$\fml{H}$.
While the formula $\fml{H}$ is satisfied, the computed model
represents a superset of some APXp, that is not yet computed. As a
result, we can then use a polynomial time algorithm for computing such
an APXp, blocking it by adding a new (negative) clause to $\fml{H}$,
and starting the process again.
As can be concluded, the computation of each APXp requires one SAT
oracle call, on a formula $\fml{H}$ whose size grows with the number
of already computed CPXp's and the number of previously computed
APXp's.
Finally, we observe that, even though calling a SAT solver is
computationally harder (in the worst-case) than
a quasi-polynomial enumeration algorithm, e.g.\ the two algorithms
proposed by M.~Fredman and L.~Khachiyan~\cite{khachiyan-jalg96},
existing practical evidence suggests otherwise~\cite{liffiton-jar08}.

\begin{exmp}
  Consider the DT from~\cref{fig:runex02}, and path
  $Q_2=\langle1,2,4,7,10,14\rangle$.
  By analyzing the paths with a different prediction we can identify
  the following weak CPXp's, from which CPXp's are then selected as
  follows:
  \begin{center}
    \begin{tabular}{ccccccc} \toprule
      Path      & $P_1$   & $P_2$   & $P_3$     & $P_4$      & $P_5$
      & CXp's
      \\ \cmidrule(lr){1-1} \cmidrule(lr){2-6} \cmidrule(lr){7-7}
      Weak CPXp's & $\{5\}$ & $\{4\}$ & $\{2,5\}$ & $\{2,4\}$ & $\{1\}$
      & $\{\{1\},\{4\},\{5\}\}$
      \\ \bottomrule
    \end{tabular}
  \end{center}
  It is clear that the only APXp is $\{1,4,5\}$.
  The initial CNF formula $\fml{H}$ is: $\{(p_1),(p_4),(p_5)\}$.
  A SAT solver would compute an assignment that satisfies $\fml{H}$,
  e.g.\ $\{(p_1=1),(p_2=0),(p_3=1),(p_4=1),(p_5=1)\}$. From this
  satisfying assignment, we identify the Weak APXp: $\{1,3,4,5\}$, from
  which the APXp $\{1,4,5\}$ would then be extracted. As a result,
  $\fml{H}$ is extended with the clause
  $(\neg{p_1}\lor\neg{p_2}\lor\neg{p_3})$. Clearly, with the new
  clause, the formula $\fml{H}$ becomes inconsistent, confirming that
  $\{1,4,5\}$ is the only APXp.
\end{exmp}

Finally, we observe that the proposed SAT encoding can be used for
enumerating smallest AXp's, as a direct consequence
of~\cref{prop:minhs}. For computing a smallest AXp,the hard clauses
are the ones proposed above (see~\cref{enum:ccases:stp01}
on~\cpageref{enum:ccases}), whereas the soft clauses are defined as
follows:
\begin{enumerate}[label=\textbf{C\arabic*.},ref=\textbf{C\arabic*},resume] 
  \label{enum:ccases2}
\item For each feature $i$, add a soft unit clause $(p_i)$.
  \label{enum:ccases2:stp02}
\end{enumerate}
Thus, instead of just enumerating AXp's using a SAT formulation, the
proposed MaxSAT formulation can be used for enumerating smallest AXp's,
but also for enumerating AXp's by increasing size. Hence, we have the
following result.
\begin{prop}
  The minimum-cost models of the propositional logic encoding
  summarized in~\cref{enum:ccases:stp01} and~\cref{enum:ccases2:stp02}
  represent smallest AXp's. Given the MaxSAT encoding proposed above,
  each of its optimum solutions represents one smallest AXp. The
  enumeration of MaxSAT solutions by decreasing size will produce
  AXp's by increasing size.
\end{prop}

\section{Experimental Results} \label{sec:res}

This section presents a summary of experimental evaluation of the
explanation redundancy of two state-of-the-art heuristic DT
classifiers and runtime assessment of the proposed algorithms 
to extract  (path-restricted) AXp's from DTs, and also
explanation redundancy in a range of DTs reported in the literature.

\paragraph{Experimental setup.}
We use the well-known DT learning tools \emph{ITI}
(\emph{Incremental Tree Induction})~\cite{utgoff-ml97,iti} and
\emph{IAI} (\emph{Interpretable AI})~\cite{bertsimas-ml17,iai}.
ITI is run with the pruning option enabled, which helps avoiding
overfitting and aims at constructing shallow DTs.
To enforce IAI to produce shallow DTs and achieve high accuracy, it is
set to use the optimal tree classifier method with the maximal depth
of 6.
This choice is motivated by our results, which confirm that larger
maximal depths would in most cases increase the percentage of 
explanation redundant paths; on the other hand, a smaller maximal 
depth would not improve accuracy.
The experiments consider datasets with categorical (non-binarized)
data, which both ITI and IAI can handle. 
(Note that other known DT learning tools, including
scikit-learn~\cite{scikitlearn-full} and
DL8.5~\cite{schaus-aaai20,schaus-ijcai20a} can only handle numerical
and binary features, respectively, and so could not be included in the
experiments.)
Furthermore, the experiments are performed on a MacBook Pro with a Dual-Core 
Intel Core~i5 2.3GHz CPU with 8GByte RAM running macOS Catalina.

\setlength{\tabcolsep}{7pt}
\let\lpr\undefined
\let\rpr\undefined
\newcommand{\lpr}{(}
\newcommand{\rpr}{)}

\begin{table}[ht] 
\centering
\resizebox{0.8\textwidth}{!}{
  \begin{tabular}{lcS[table-format=3.0]S[table-format=3.0]S[table-format=2.0]S[table-format=2.0]S[table-format=2.0]ccc}
\toprule[1.2pt]
\multirow{2}{*}{\bf Dataset} & \multicolumn{9}{c}{\bf IAI}  \\
\cmidrule[0.8pt](lr{.75em}){2-10}
&  {\bf D} & {\bf \#N} & {\bf \%A} & {\bf \#P} & {\bf \%R} & {\bf \%C} & {\bf \%m} & {\bf \%M} & {\bf \%avg} \\
\toprule[1.2pt]

adult &  6 & 83 & 78 & 42 & 33 & 25 & 20 & 40 & 25 \\
ann-thyroid & 6 & 61 & 97 & 31 & 25 & 30 & 20 & 50 & 36 \\
anneal &  6 & 29 & 99 & 15 & 26 & 16 & 16 & 33 & 21 \\
backache  & 4 & 17 & 72 & 9 & 33 & 39 & 25 & 33 & 30 \\
bank  & 6 & 113 & 88 & 57 & 5 & 12 & 16 & 20 & 18 \\
biodegradation  & 5 & 19 & 65 & 10 & 30 & 1 & 25 & 50 & 33 \\
cancer & 6 & 37 & 87 & 19 & 36 & 9 & 20 & 25 & 21 \\
car  & 6 & 43 & 96 & 22 & 86 & 89 & 20 & 80 & 45 \\
colic  & 6 & 55 & 81 & 28 & 46 & 6 & 16 & 33 & 20 \\
compas  & 6 & 77 & 34 & 39 & 17 & 8 & 16 & 20 & 17 \\
contraceptive  & 6 & 99 & 49 & 50 & 8 & 2 & 20 & 60 & 37 \\
dermatology  & 6 & 33 & 90 & 17 & 23 & 3 & 16 & 33 & 21 \\
divorce  & 5 & 15 & 90 & 8 & 50 & 19 & 20 & 33 & 24 \\
german  & 6 & 25 & 61 & 13 & 38 & 10 & 20 & 40 & 29 \\
heart-c  & 6 & 43 & 65 & 22 & 36 & 18 & 20 & 33 & 22 \\
heart-h  & 6 & 37 & 59 & 19 & 31 & 4 & 20 & 40 & 24 \\
kr-vs-kp  & 6 & 49 & 96 & 25 & 80 & 75 & 16 & 60 & 33 \\
lending  & 6 & 45 & 73 & 23 & 73 & 80 & 16 & 50 & 25 \\
letter  & 6 & 127 & 58 & 64 & 1 & 0 & 20 & 20 & 20 \\
lymphography  & 6 & 61 & 76 & 31 & 35 & 25 & 16 & 33 & 21 \\
mushroom  & 6 & 39 & 100 & 20 & 80 & 44 & 16 & 33 & 24 \\
pendigits  & 6 & 121 & 88 & 61 & 0 & 0 & \textemdash & \textemdash & \textemdash \\
promoters  & 1 & 3 & 90 & 2 & 0 & 0 & \textemdash & \textemdash & \textemdash \\
recidivism & 6 & 105 & 61 & 53 & 28 & 22 & 16 & 33 & 18 \\
seismic\_bumps  & 6 & 37 & 89 & 19 & 42 & 19 & 20 & 33 & 24 \\
shuttle & 6 & 63 & 99 & 32 & 28 & 7 & 20 & 33 & 23 \\
soybean & 6 & 63 & 88 & 32 & 9 & 5 & 25 & 25 & 25 \\
spambase  & 6 & 63 & 75 & 32 & 37 & 12 & 16 & 33 & 19 \\
spect & 6 & 45 & 82 & 23 & 60 & 51 & 20 & 50 & 35 \\
splice  & 3 & 7 & 50 & 4 & 0 & 0 & \textemdash & \textemdash & \textemdash \\

\bottomrule[1.2pt]
\end{tabular}
}

 \caption{%
     \footnotesize{Path explanation redundancy in decision trees
    obtained with IAI.
    The table shows tree statistics for IAI, namely, tree
    depth {\bf D}, number of nodes {\bf \#N}, test accuracy {\bf \%A}
    and number of paths {\bf \#P}.
    The percentage of explanation-redundant paths (XRP's) is given as {\bf \%R}
    while the percentage of data instances (measured for the \emph{entire}
    feature space) covered by XRP's is {\bf \%C}.
    Focusing solely on the XRP's, the average (min.~or max., resp.)
    percentage of explanation-redundant features per path is denoted
    by {\bf \%avg}  ({\bf \%m} and {\bf \%M}, resp.). }
  }
  \label{tab:iai-res}
\end{table}


%
\setlength{\tabcolsep}{7pt}
\let\lpr\undefined
\let\rpr\undefined
\newcommand{\lpr}{(}
\newcommand{\rpr}{)}

\begin{table}[ht] 
\centering
\resizebox{0.8\textwidth}{!}{
  \begin{tabular}{lcS[table-format=3.0]S[table-format=3.0]S[table-format=2.0]S[table-format=2.0]S[table-format=2.0]ccc}
\toprule[1.2pt]
\multirow{2}{*}{\bf Dataset} & \multicolumn{9}{c}{\bf ITI}  \\
\cmidrule[0.8pt](lr{.75em}){2-10}
&  {\bf D} & {\bf \#N} & {\bf \%A} & {\bf \#P} & {\bf \%R} & {\bf \%C} & {\bf \%m} & {\bf \%M} & {\bf \%avg} \\
\toprule[1.2pt]

adult &  6 & 83 & 78 & 42 & 33 & 25 & 20 & 40 & 25 \\
ann-thyroid & 6 & 61 & 97 & 31 & 25 & 30 & 20 & 50 & 36 \\
anneal &  6 & 29 & 99 & 15 & 26 & 16 & 16 & 33 & 21 \\
backache  & 4 & 17 & 72 & 9 & 33 & 39 & 25 & 33 & 30 \\
bank  & 6 & 113 & 88 & 57 & 5 & 12 & 16 & 20 & 18 \\
biodegradation  & 5 & 19 & 65 & 10 & 30 & 1 & 25 & 50 & 33 \\
cancer & 6 & 37 & 87 & 19 & 36 & 9 & 20 & 25 & 21 \\
car  & 6 & 43 & 96 & 22 & 86 & 89 & 20 & 80 & 45 \\
colic  & 6 & 55 & 81 & 28 & 46 & 6 & 16 & 33 & 20 \\
compas  & 6 & 77 & 34 & 39 & 17 & 8 & 16 & 20 & 17 \\
contraceptive  & 6 & 99 & 49 & 50 & 8 & 2 & 20 & 60 & 37 \\
dermatology  & 6 & 33 & 90 & 17 & 23 & 3 & 16 & 33 & 21 \\
divorce  & 5 & 15 & 90 & 8 & 50 & 19 & 20 & 33 & 24 \\
german  & 6 & 25 & 61 & 13 & 38 & 10 & 20 & 40 & 29 \\
heart-c  & 6 & 43 & 65 & 22 & 36 & 18 & 20 & 33 & 22 \\
heart-h  & 6 & 37 & 59 & 19 & 31 & 4 & 20 & 40 & 24 \\
kr-vs-kp  & 6 & 49 & 96 & 25 & 80 & 75 & 16 & 60 & 33 \\
lending  & 6 & 45 & 73 & 23 & 73 & 80 & 16 & 50 & 25 \\
letter  & 6 & 127 & 58 & 64 & 1 & 0 & 20 & 20 & 20 \\
lymphography  & 6 & 61 & 76 & 31 & 35 & 25 & 16 & 33 & 21 \\
mushroom  & 6 & 39 & 100 & 20 & 80 & 44 & 16 & 33 & 24 \\
pendigits  & 6 & 121 & 88 & 61 & 0 & 0 & \textemdash & \textemdash & \textemdash \\
promoters  & 1 & 3 & 90 & 2 & 0 & 0 & \textemdash & \textemdash & \textemdash \\
recidivism & 6 & 105 & 61 & 53 & 28 & 22 & 16 & 33 & 18 \\
seismic\_bumps  & 6 & 37 & 89 & 19 & 42 & 19 & 20 & 33 & 24 \\
shuttle & 6 & 63 & 99 & 32 & 28 & 7 & 20 & 33 & 23 \\
soybean & 6 & 63 & 88 & 32 & 9 & 5 & 25 & 25 & 25 \\
spambase  & 6 & 63 & 75 & 32 & 37 & 12 & 16 & 33 & 19 \\
spect & 6 & 45 & 82 & 23 & 60 & 51 & 20 & 50 & 35 \\
splice  & 3 & 7 & 50 & 4 & 0 & 0 & \textemdash & \textemdash & \textemdash \\

\bottomrule[1.2pt]
\end{tabular}
}

 \caption{%
     \footnotesize{Path explanation redundancy in decision trees
    obtained with ITI.
    %
      %
    Columns {\bf D}, {\bf \#N}, {\bf \#P}, {\bf \%R}, {\bf \%C},
    {\bf \%m}, {\bf \%M} and  {\bf \%avg} have the same
    meaning as in~\autoref{tab:iai-res}.}
  }
  \label{tab:iti-res}
\end{table}


%
\setlength{\tabcolsep}{7pt}
\let\lpr\undefined
\let\rpr\undefined
\newcommand{\lpr}{(}
\newcommand{\rpr}{)}

\begin{table}[ht] 
\centering
\resizebox{0.8\textwidth}{!}{
  \begin{tabular}{l cccS[table-format=2.2,round-mode=places,round-precision=2] cccS[table-format=2.2,round-mode=places,round-precision=2] }
\toprule[1.2pt]
\multirow{3}{*}{\bf Dataset}  & \multicolumn{8}{c}{\bf IAI} \\
\cmidrule[0.8pt](lr{.75em}){2-9}
& \multicolumn{4}{c}{\bf Traversal}  & \multicolumn{4}{c}{\bf Horn}   \\
\cmidrule[0.8pt](lr{.75em}){2-5}
\cmidrule[0.8pt](lr{.75em}){6-9}
&  {\bf m} & {\bf M} & {\bf avg} & {\bf Tot} & {\bf m} & {\bf M}  & {\bf avg} & {\bf Tot} \\
\toprule[1.2pt]

adult & 0.001 & 0.059 & 0.002 & 3.52 & 0.001 & 0.005 & 0.002 & 2.93 \\
ann-thyroid & 0.001 & 0.005 & 0.002 & 3.67 & 0.001 & 0.005 & 0.001 & 2.85 \\
anneal & 0.001 & 0.005 & 0.001 & 1.22 & 0.001 & 0.003 & 0.001 & 0.75 \\
backache & 0.001 & 0.001 & 0.001 & 0.13 & 0.000 & 0.001 & 0.001 & 0.094 \\
bank & 0.002 & 0.062 & 0.003 & 34.45 & 0.002 & 0.008 & 0.002 & 25.53 \\
biodegradation & 0.000 & 0.003 & 0.001 & 0.21 & 0.000 & 0.002 & 0.001 & 0.18 \\
cancer & 0.001 & 0.003 & 0.001 & 0.53 & 0.001 & 0.003 & 0.001 & 0.39 \\
car & 0.001 & 0.004 & 0.001 & 0.52 & 0.001 & 0.002 & 0.001 & 0.48 \\
colic & 0.001 & 0.005 & 0.002 & 0.71 & 0.001 & 0.002 & 0.001 & 0.43 \\
compas & 0.001 & 0.004 & 0.002 & 0.66 & 0.001 & 0.004 & 0.002 & 0.52 \\
contraceptive & 0.001 & 0.005 & 0.002 & 0.70 & 0.002 & 0.004 & 0.002 & 0.78 \\
dermatology & 0.001 & 0.005 & 0.001 & 0.50 & 0.001 & 0.002 & 0.001 & 0.31 \\
divorce & 0.000 & 0.002 & 0.001 & 0.10 & 0.000 & 0.001 & 0.001 & 0.076 \\
german & 0.001 & 0.004 & 0.001 & 0.76 & 0.001 & 0.002 & 0.001 & 0.65 \\
heart-c & 0.001 & 0.003 & 0.001 & 0.42 & 0.001 & 0.003 & 0.001 & 0.30 \\
heart-h & 0.001 & 0.004 & 0.001 & 0.39 & 0.001 & 0.005 & 0.001 & 0.27 \\
kr-vs-kp & 0.001 & 0.008 & 0.002 & 2.17 & 0.001 & 0.004 & 0.001 & 1.22 \\
lending & 0.001 & 0.003 & 0.001 & 1.99 & 0.001 & 0.003 & 0.001 & 1.51 \\
letter & 0.002 & 0.062 & 0.002 & 13.36 & 0.002 & 0.007 & 0.003 & 14.30 \\
lymphography & 0.001 & 0.007 & 0.002 & 0.32 & 0.001 & 0.003 & 0.001 & 0.20 \\
mushroom & 0.001 & 0.004 & 0.001 & 3.11 & 0.001 & 0.002 & 0.001 & 2.20 \\
pendigits & 0.002 & 0.063 & 0.003 & 10.03 & 0.002 & 0.007 & 0.003 & 8.45 \\
promoters & 0.000 & 0.000 & 0.000 & 0.017 & 0.000 & 0.000 & 0.000 & 0.024 \\
recidivism & 0.002 & 0.061 & 0.003 & 3.65 & 0.002 & 0.006 & 0.002 & 2.59 \\
seismic\_bumps & 0.001 & 0.003 & 0.001 & 1.08 & 0.001 & 0.002 & 0.001 & 0.68 \\
shuttle & 0.001 & 0.006 & 0.001 & 22.65 & 0.001 & 0.005 & 0.001 & 22.35 \\
soybean & 0.001 & 0.058 & 0.002 & 1.17 & 0.001 & 0.005 & 0.001 & 0.82 \\
spambase & 0.001 & 0.008 & 0.003 & 3.43 & 0.001 & 0.003 & 0.001 & 1.81 \\
spect & 0.001 & 0.006 & 0.002 & 0.52 & 0.001 & 0.004 & 0.001 & 0.24 \\
splice & 0.000 & 0.001 & 0.000 & 0.22 & 0.000 & 0.002 & 0.000 & 0.30 \\

\bottomrule[1.2pt]
\end{tabular}
}

  \caption{%
   \footnotesize{Assessing runtimes of the tree traversal
    algorithm and the propositional horn encoding approach for
    extracting one AXp.
    The table reports the results for DTs trained with  IAI
    learning tool. Columns  {\bf m}, {\bf M} and  {\bf avg}  report,
    resp.\ , the minimal, maximal and average runtime (in second)
    to compute an AXp, while column {\bf Tot} reports the total
    runtime (in second) of all tested instances in a dataset. }
  }
  \label{tab:iai-runtime}
\end{table}

\setlength{\tabcolsep}{7pt}
\let\lpr\undefined
\let\rpr\undefined
\newcommand{\lpr}{(}
\newcommand{\rpr}{)}

\begin{table}[ht] 
\centering
\resizebox{0.8\textwidth}{!}{
  \begin{tabular}{l cccS[table-format=2.2,round-mode=places,round-precision=2] cccS[table-format=2.2,round-mode=places,round-precision=2] }
\toprule[1.2pt]
\multirow{3}{*}{\bf Dataset}  & \multicolumn{8}{c}{\bf ITI} \\
\cmidrule[0.8pt](lr{.75em}){2-9}
& \multicolumn{4}{c}{\bf Traversal}  & \multicolumn{4}{c}{\bf Horn}   \\
\cmidrule[0.8pt](lr{.75em}){2-5}
\cmidrule[0.8pt](lr{.75em}){6-9}
&  {\bf m} & {\bf M} & {\bf avg} & {\bf Tot} & {\bf m} & {\bf M}  & {\bf avg} & {\bf Tot} \\
\toprule[1.2pt]

adult &  0.004 & 0.038 & 0.007 & 12.89 & 0.008 & 0.036 & 0.010 & 19.00 \\ 
ann-thyroid  &  0.002  &  0.041  &  0.006  & 12.14 &  0.004  &  0.029  &  0.005  & 9.64 \\
anneal  &  0.001  &  0.006  &  0.001  & 0.96 &  0.001  &  0.004  &  0.001  & 0.70 \\
backache  &  0.000  &  0.001  &  0.000  & 0.08 &  0.000  &  0.001  &  0.000  & 0.07 \\
bank  &  0.013  &  0.090  &  0.027  & 19.64 &  0.025  &  0.092  &  0.033  & 24.21 \\
biodegradation  &  0.002  &  0.007  &  0.003  & 0.98 &  0.001  &  0.003  &  0.002  & 0.48 \\
cancer  &  0.001  &  0.004  &  0.001  & 0.32 &  0.001  &  0.002  &  0.001  & 0.26 \\
car  &  0.001  &  0.002  &  0.001  & 0.47 &  0.001  &  0.002  &  0.001  & 0.59 \\
colic  &  0.000  &  0.001  &  0.001  & 0.24 &  0.000  &  0.001  &  0.000  & 0.18 \\
compas  &  0.002  &  0.065  &  0.004  & 1.37 &  0.003  &  0.005  &  0.004  & 1.28 \\
contraceptive  &  0.003  &  0.064  &  0.006  & 2.64 &  0.006  &  0.012  &  0.007  & 2.87 \\
dermatology  &  0.000  &  0.001  &  0.001  & 0.27 &  0.000  &  0.001  &  0.001  & 0.19 \\
divorce  &  0.000  &  0.001  &  0.000  & 0.03 &  0.000  &  0.002  &  0.000  & 0.04 \\
german  &  0.002  &  0.007  &  0.003  & 3.19 &  0.002  &  0.003  &  0.002  & 2.03 \\
heart-c  &  0.000  &  0.001  &  0.001  & 0.17 &  0.000  &  0.001  &  0.000  & 0.14 \\
heart-h  &  0.001  &  0.001  &  0.001  & 0.21 &  0.001  &  0.001  &  0.001  & 0.18 \\
kr-vs-kp  &  0.001  &  0.009  &  0.004  & 3.41 &  0.001  &  0.003  &  0.002  & 1.49 \\
lending  &  0.004 & 0.030 & 0.006 & 9.13 & 0.008 & 0.039 & 0.010 & 15.99 \\
letter  &  0.034  &  0.110  &  0.052  & 19.50 &  0.078  &  0.16  &  0.110  & 40.77 \\
lymphography  &  0.000  &  0.002  &  0.001  & 0.13 &  0.001  &  0.001  &  0.001  & 0.09 \\
mushroom  &  0.001  &  0.003  &  0.001  & 2.46 &  0.001  &  0.002  &  0.001  & 1.54 \\
pendigits &  0.008 & 0.047 & 0.011 & 37.59 & 0.015 & 0.056 & 0.019 & 61.59 \\
promoters  &  0.000  &  0.001  &  0.000  & 0.04 &  0.000  &  0.001  &  0.000  & 0.04 \\
recidivism  &  0.005  &  0.087  &  0.010  & 11.69 &  0.010  &  0.084  &  0.015  & 18.26 \\
seismic\_bumps  &  0.001  &  0.004  &  0.002  & 1.83 &  0.001  &  0.002  &  0.001  & 0.69 \\
shuttle  &  0.002  &  0.061  &  0.002  & 2.85 &  0.003  &  0.060  &  0.003  & 3.56 \\
soybean  &  0.001  &  0.005  &  0.003  & 1.65 &  0.001  &  0.003  &  0.002  & 0.95 \\
spambase  &  0.002  &  0.069  &  0.009  & 11.20 &  0.003  &  0.062  &  0.003  & 4.18 \\
spect  &  0.000  &  0.001  &  0.001  & 0.17 &  0.000  &  0.001  &  0.000  & 0.11 \\
splice  &  0.001  &  0.064  &  0.002  & 1.48 &  0.003  &  0.069  &  0.004  & 3.43 \\

\bottomrule[1.2pt]
\end{tabular}
}

  \caption{%
   \footnotesize{Assessing runtimes of the tree traversal
    algorithm and the propositional horn encoding approach for
    extracting one AXp.
    The table reports the results for DTs trained with  ITI
    learning tool.
    Columns  {\bf m}, {\bf M} ,  {\bf avg} and {\bf Tot} have the same
    meaning as in~\autoref{tab:iai-runtime}. }
  }
  \label{tab:iti-runtime}
\end{table}

\paragraph{Benchmarks.}
The assessment is performed on a selection of 67 publicly
available datasets, which originate from \emph{UCI Machine Learning
  Repository}~\cite{uci}, \emph{Penn Machine Learning
  Benchmarks}~\cite{pennml}, and \emph{OpenML
repository}~\cite{openml}.
(We opt to report the results only for a subset of datasets. However, 
the results shown mimic the results for the complete benchmark set;
these are provided as supplementary
materials~\footnote{\url{https://github.com/yizza91/jair22sub}}.)
The number of features (data instances, resp.) in the benchmark suite
vary from 2 to 58 (87 to 58000, resp.) with the average being 31.2
(6045.3, resp.).

\paragraph{Prototype implementation.}
The poly-time explanation-redundancy check algorithm presented 
in~\cite{iims-corr20} and  AXp extraction by Tree Traversal outlined 
in~\cref{ssec:trav} are  
implemented in Perl. (An implementation using  PySAT~\cite{imms-sat18} toolkit 
and the solver Glucose, was instrumented in validating the results, but for the DTs
considered, it was in general slower by at least one order of
magnitude.)
Additionally, the Propositional Horn Encoding approach outlined
in~\cref{ssec:hornit} as well as the enumeration of  AXp's/CXp's  
described in~\cref{ssec:exdt}, are implemented in Python\footnote{
Sources are provided as a Python package and available in 
\url{https://github.com/yizza91/xpg}}. 

\paragraph{Results.}
Training DTs with IAI takes from 4s to 2310s with the average run time
per dataset being 70s.
In contrast, the time spent on eliminating explanation redundancy
is \emph{negligible}, taking from 0.026s to 0.4s per tree, with an
average time of 0.06s.
ITI runs much faster than IAI and takes from 0.1s to 2s with
0.1s on average; the elimination of explanation redundancy is
slightly more time consuming than for IAI, taking from
0.025s to 5.4s with 0.29s on average.
This slowdown results from DTs learned with ITI being deeper on
average, and features being tested multiple times along 
a same path.

\autoref{tab:iai-res} and \autoref{tab:iti-res} summarize, resp., the results  
of the explanation redundancy evaluation of IAI and ITI trees.
Observe that despite the shallowness of the trees produced by IAI and
ITI, for the majority of datasets and with a few exceptions, the paths
in trees trained by both tools exhibit significant
explanation redundancy.
In particular, on average, 32.1\% (46.9\%, resp.) of paths are
explanation redundant for the trees obtained by IAI (ITI, resp.).
For some DTs, obtained with either IAI and ITI, more than 85\% of tree
paths are explanation redundant (XRP).
Also, explanation redundant paths (XRP's) of the trees of IAI (ITI, resp.)
cover on average 20.1\% (37.7\%, resp.) of feature space\footnote{
The coverage of a path is the feature space size of  uninvolved/untested 
features in this path.}. 
Moreover, in some cases, up to 89\% and 98\% of the entire feature
space is covered by the XRP's for IAI and ITI,
respectively.
This means that DTs produced by IAI and ITI are unable to provide a
user with a succinct explanation for the \emph{vast majority} of data
instances.
In addition, the average number of explanation redundant features 
(XRF's) in XRP's  for both IAI and ITI varies from 16\% to
65\%, but for some DTs it exceeds 80\%.

To summarize, the numbers shown for the selected datasets and for the two
state-of-the-art DT training tools (IAI and ITI) contrast with the
common belief in the \emph{inherent interpretability} of decision tree
classifiers.
Perhaps as importantly, the performance figures confirm that
the elimination of explanation redundancy in the DTs produced with
available tools has \emph{negligible} computational cost.

To demonstrate the effectiveness of the proposed algorithms,
concretely tree traversal and propositional Horn encoding, we assess
their running times to compute  path-restricted AXp's from DTs
obtained with ITI and IAI.
The results are summarized in \autoref{tab:iai-runtime} and 
\autoref{tab:iti-runtime}.
As is quite evident from these results, the proposed solutions  are
effective in practice and the average running times are almost similar
for all datasets and both DT learning tools. 
As final remark, we notice that  in terms of comparison between 
the two algorithms, the Horn encoding approach is faster in 
44/62 explained DTs trained with IAI and 41/62 DTs  trained 
with ITI.  Therefore, one can use a portfolio of the two approaches, 
terminating when one finishes. 

\setlength{\tabcolsep}{5pt}
\let\lpr\undefined
\let\rpr\undefined
\newcommand{\lpr}{(}
\newcommand{\rpr}{)}

\begin{table}[t] 
\centering
\resizebox{0.6\textwidth}{!}{
  \begin{tabular}{l  ccc  S[table-format=2.0]   S[table-format=2.0]S[table-format=2.0]S[table-format=2.1]S[table-format=2.0]}
\toprule[1.2pt]
\multirow{2}{*}{\bf Dataset} & \multicolumn{3}{c}{\bf DT} & {\bf Path} & \multicolumn{4}{c}{\bf AXp} \\
\cmidrule[0.8pt](lr{.75em}){2-4}
\cmidrule[0.8pt](lr{.75em}){5-5}
\cmidrule[0.8pt](lr{.75em}){6-9}
&  {\bf D} & {\bf \#N} & {\bf \%A} & {\bf L} &  {\bf m} & {\bf M}  & {\bf avg} & {\bf n}  \\
\toprule[1.2pt]

\multirow{3}{*}{adult}  &  \multirow{3}{*}{17} &  \multirow{3}{*}{509} & \multirow{3}{*}{73} & 16 & 8 & 8 & 8 & 2 \\
 &  &  &  & 14 & 5 & 6 & 5.5 & 2 \\
 &  &  &  & 16 & 5 & 5 & 5 & 1 \\

\midrule

\multirow{3}{*}{allhyper} &  \multirow{3}{*}{14} &  \multirow{3}{*}{49} & \multirow{3}{*}{96} & 14 & 4 &  5 & 4.6 & 6 \\ 
 &   &  &  & 9 & 4 &  5 & 4.5 & 8 \\ 
 &  &   &  & 14 & 4 &  5 & 4.6 & 6 \\ 
	 
\midrule

\multirow{3}{*}{ann-thyroid} &  \multirow{3}{*}{48} & \multirow{3}{*}{222} & \multirow{3}{*}{93} & 40 & 5 & 6 & 5.6 & 3 \\
 &   &  &  & 36 & 5 & 6 & 5.6 & 3 \\
 &   &  &  & 20 & 5 & 6 & 5.5 & 2 \\

\midrule

\multirow{3}{*}{coil2000}  &  \multirow{3}{*}{12} & \multirow{3}{*}{177} & \multirow{3}{*}{91} & 10 &  2 & 4 & 3.8 & 39 \\
 &   &  &  & 10 &  2 & 4 & 3.8 & 39 \\
 &   &  &  & 10 &  2 & 4 & 3.8 & 30 \\

\midrule

\multirow{3}{*}{fars} & \multirow{3}{*}{60} & \multirow{3}{*}{9969} & \multirow{3}{*}{76} & 29 & 11 & 11 & 11 & 2 \\
 &  &  &  & 42 & 10 & 14 & 12.3 & 9 \\
 &  &  &  & 48 & 9 & 9 & 9 & 1 \\

\midrule

\multirow{3}{*}{kddcup} & \multirow{3}{*}{29} & \multirow{3}{*}{269} & \multirow{3}{*}{99} &  23  & 11 & 12 & 11.5 & 16\\
 &  &  &  &  23  & 11 & 12 & 11.5 & 16\\
 &  &  &  &  27 & 12 & 13 & 12.5 & 8 \\

\bottomrule[1.2pt]
\end{tabular}
}
\caption{\footnotesize{
    Examples of 6 real-world datasets highlighting computed 
    path AXp's (APXp's) in DTs learned with ITI, that require
    deep trees.
    For each dataset, the table displays 3 tested paths. 
    Columns  {\bf D}, {\bf \#N} and {\bf \%A}  denote,
    resp.\  depth, number of nodes and accuracy of the DT. 
    Next, column  {\bf L} reports the path's length. 
    Then, the average (min. or max., resp.) length of the computed 
    APXp's, is denoted by  {\bf avg} ({\bf m} and {\bf M}, resp.).
    Finally, the total number of APXp's is shown in  column {\bf n}.
    }
 } \label{tab:inst}
\end{table}

\begin{table}[t] 
\centering
\resizebox{\textwidth}{!}{
  \begin{tabular}{lS[table-format=1.0]S[table-format=2.0]S[table-format=2.0]S[table-format=2.0]S[table-format=2.0]S[table-format=2.0]S[table-format=2.0]S[table-format=2.0]}
\toprule[1.2pt]
{\bf DT Ref} & {\bf D} & {\bf \#N}  & {\bf \#P} & {\bf \%R} & {\bf \%C} & {\bf \%m} & {\bf \%M} & {\bf \%avg} \\
\toprule[1.2pt]

\cite[Ch.~09,~Fig.~9.1]{alpaydin-bk14} & 2 & 5 & 3 & 33 & 25 & 50 & 50 & 50 \\
\cite[Ch.~03,~Fig.~3.2]{alpaydin-bk16} & 2 & 5 & 3 & 33 & 25 & 50 & 50 & 50 \\
\cite[Ch.~01,~Fig.~1.3]{bramer-bk20} & 4 & 9 & 5 & 60 & 25 & 25 & 50 & 36 \\
\cite[Figure~1]{aha-ker97} & 3 & 12 & 7 & 14 & 8 & 33 & 33 & 33 \\
\cite[Ch.~08,~Fig.~8.2]{berthold-bk10} & 3 & 7 & 4 & 25 & 12 & 50 & 50 & 50 \\
\cite[Ch.~01,~Fig.~1.1]{breiman-bk84} & 3 & 7 & 4 & 50 & 25 & 33 & 33 & 33 \\
\cite[Ch.~01,~Fig.~1.2a]{dzeroski-bk01} & 2 & 5 & 3 & 33 & 25 & 33 & 33 & 33 \\
\cite[Ch.~01,~Fig.~1.2b]{dzeroski-bk01} & 2 & 5 & 3 & 33 & 25 & 33 & 33 & 33 \\
\cite[Ch.~04,~Fig.~4.14]{kelleher-bk15} & 3 & 7 & 4 & 25 & 12 & 50 & 50 & 50 \\
\cite[Sec.~4.7,~Ex.~4]{kelleher-bk15} & 2 & 5 & 3 & 33 & 25 & 50 & 50 & 50 \\
\cite[Ch.~01,~Fig.~1.3]{quinlan-bk93} & 3 & 12 & 7 & 28 & 17 & 33 & 50 & 41 \\
\cite[Ch.~01,~Fig.~1.5]{rokach-bk08} & 3 & 9 & 5 & 20 & 12 & 33 & 33 & 33 \\
\cite[Ch.~01,~Fig.~1.4]{rokach-bk08} & 3 & 7 & 4 & 50 & 25 & 33 & 33 & 33 \\
\cite[Ch.~01,~Fig.~1.2]{witten-bk17} & 3 & 7 & 4 & 25 & 12 & 50 & 50 & 50 \\
\cite[Figure~4]{valdes-naturesr16} & 6 & 39 & 20 & 65 & 63 & 20 & 40 & 33 \\
\cite[Ch.~02,~Fig.~2.1(right)]{flach-bk12} & 2 & 5 & 3 & 33 & 25 & 50 & 50 & 50 \\
\cite[Figure~1]{kotsiantis-air13} & 3 & 10 & 6 & 33 & 11 & 33 & 33 & 33 \\
\cite[Figure~1]{moret-acmcs82} & 3 & 9 & 5 & 80 & 75 & 33 & 50 & 41 \\
\cite[Ch.~07,~Fig.~7.4]{poole-bk17} & 3 & 7 & 4 & 50 & 25 & 33 & 33 & 33 \\
\cite[Ch.~18,~Fig.~18.6]{russell-bk10} & 4 & 12 & 8 & 25 & 6 & 25 & 33 & 29 \\
\cite[Ch.~18,~Page~212]{shalev-shwartz-bk14} & 2 & 5 & 3 & 33 & 25 &  50 & 50 & 50 \\
\cite[Ch.~01,~Fig.~1.3]{zhou-bk12} & 2 & 5 & 3 & 33 & 25 & 33 & 33 & 33 \\
\cite[Figure~1b]{hebrard-cp09} & 4 & 13 & 7 & 71 & 50 & 33 & 50 & 36 \\ 
\cite[Ch.~04,~Fig.~4.3]{zhou-bk21} & 4 & 14 & 9 & 11 & 2 & 25 & 25 & 25  \\
\bottomrule[1.2pt]
\end{tabular}
}
  \caption{
    \footnotesize{Results on path explanation redundancy for example
      DTs found in the literature.
    Columns {\bf D}, {\bf \#N}, {\bf \#P}, {\bf \%R}, {\bf \%C},
    {\bf \%m}, {\bf \%M} and  {\bf \%avg} have the same
    meaning as in~\autoref{tab:iai-res}.
    }
    \label{tab:dtrees}
  }
\end{table}

\begin{table}[t] 
 \centering
\resizebox{0.75\textwidth}{!}{
  \begin{tabular}{llS[table-format=1.0]S[table-format=2.0]S[table-format=3.0]S[table-format=2.0]S[table-format=2.0]S[table-format=2.0]S[table-format=2.0]S[table-format=2.0]S[table-format=2.0]}
\toprule[1.2pt]
{\bf Dataset} & { \bf Tool} & {\bf D} & {\bf \#N} & {\bf \%A} & {\bf \#P} & {\bf \%R} & {\bf \%C} & {\bf \%m} & {\bf \%M} & {\bf \%avg} \\
\toprule[1.2pt]

\multirow{2}{*}{monk1} & BinOCT & 3 & 13 & 91 & 7 & 28 & 11 & 66 & 66  & 66 \\
 & OSDT & 5 & 13  & 100 & 7 & 57 & 41 & 33 & 33 & 33 \\
 \midrule
 
\multirow{2}{*}{tic-tac-toe} & BinOCT & 4 & 15 & 77 &  8 & 75 & 75 & 33 & 33 & 33  \\
& OSDT & 5 & 15 & 83 & 8 & 75 & 37 & 25 & 60 & 43 \\
\midrule

compas & OSDT & 4 & 9 & 67 &  5 & 60 & 37 & 33 & 33 & 33 \\
\midrule

\multirow{2}{*}{monk2} & CART & 6 & 31 & 69 & 16 & 62 & 22 & 20 & 66 & 33 \\
 & GOSDT & 6 & 17 & 73 & 9 & 55 & 48 & 16 & 40 & 31 \\

\bottomrule[1.2pt]
\end{tabular}
} 
\caption{\footnotesize{Additional results on path explanation redundancy in
    (optimal) DTs, trained with different training tools: BinOCT
    \cite{verwer-aaai19}, CART \cite{breiman-bk84}, OSDT
    \cite{rudin-nips19} and GOSDT \cite{rudin-icml20}, that have been
    presented in \cite{rudin-nips19,rudin-corr21}.
    The results for CART are solely included for completeness.
    Columns {\bf D}, {\bf \#N}, {\bf \%A}, {\bf \#P}, {\bf \%R}, {\bf
      \%C}, {\bf \%m}, {\bf \%M} and {\bf \%avg}  hold the same
    meaning in~\autoref{tab:iai-res}. }
  \label{tab:osdt}}
\end{table}

Focusing merely on complex datasets that require deep trees, 
\autoref{tab:inst} shows results on computed path AXp's  for 
a set of  DTs generated by ITI.  
The results show that for these examples, paths can be much longer 
than path AXp's, namely,  
the number of explanation redundant features is bigger than 
the number of  features  involved in the explanation. 
Notably for some examples, the number of explanation redundant features 
is more than 7 times larger than the number of features belonging to 
the abductive explanation.

Finally, additional results on explanation redundancy of DTs reported 
in the literature are shown in \autoref{tab:dtrees} and 
\autoref{tab:osdt}.
As can be seen, the same observations made for DTs of IAI and ITI hold
for DTs obtained with different training tools  existing in the
literature. 
More notably, these results demonstrate that also 
optimal (sparse)  DTs, deemed succinctly explainable due to their shallowness, 
exhibit  explanation-redundant paths/features.

\section{Related Work} \label{sec:relw}
As indicated in~\cref{sec:intro}, there exists a growing
body of work on (optimally) learning DTs aiming for
interpretability\footnote{%
  Example references include~\cite{nijssen-kdd07,hebrard-cp09,nijssen-dmkd10,bertsimas-ml17,verwer-cpaior17,nipms-ijcai18,verwer-aaai19,rudin-nips19,avellaneda-corr19,avellaneda-aaai20,schaus-cj20,schaus-aaai20,rudin-icml20,janota-sat20,hebrard-ijcai20,schaus-ijcai20a,schaus-ijcai20b,demirovic-aaai21,szeider-aaai21a,szeider-aaai21b,mcilraith-cp21,ansotegui-corr21,demirovic-jmlr22,rudin-aaai22}.}. 
There is also general consensus on the interpretability of
DTs~\cite{breiman-ss01,rudin-naturemi19,molnar-bk20}.
The results in this paper prove that efforts for learning optimal
DTs are necessarily incomplete, since the trees generated by such
tools can (and inevitably will) exhibit path explanation redundancy.
Furthermore, if interpretability is to be related with explanation
succinctness, then our results prove (in theory and in practice) that
learned optimal DTs should not in general be deemed interpretable,
because more succinct (and in some cases far more succinct)
explanations can be obtained with the algorithms proposed in this
paper.

To our best knowledge, the assessment of path explanation redundancy
in DTs when compared to AXp's has not been investigated in depth,
besides our own work~\cite{iims-corr20,hiims-kr21} and results on the
complexity of explaining DTs~\cite{barcelo-nips20} or the
intelligibility of DTs~\cite{marquis-kr21}. However, some of the
earlier results focus on boolean
DT classifiers~\cite{barcelo-nips20,marquis-kr21}, and so
the generalization to non-boolean DT classifiers is unclear. Moreover,
recent work~\cite{darwiche-corr20} outlines logical encodings of
decision trees, but that is orthogonal to the work reported in this
paper.
It should be underscored that, in contrast with our own earlier
work~\cite{iims-corr20,hiims-kr21}, this paper highlights path
explanations, both abductive and contrastive.
In addition, there has been work on applying explainable AI (XAI) to
decision trees~\cite{lundberg-naturemi20}, but with the focus of
improving the quality of local (heuristic) explanations, where the
goal is to relate a local approximate model against a reference model;
hence there is no immediate relationship with the formal explanations
studied in this paper.
Similarly, one could consider exploiting non-formal model-agnostic
explainers.
There is a large body of work on non-formal model-agnostic XAI
approaches~\cite{berrada-ieee-access18,muller-dsp18,muller-bk19,pedreschi-acmcs19,muller-ieee-proc21,guan-ieee-tnnls21,holzinger-bk22,holzinger-xxai22b,doran-jair22}.
Well-known examples include LIME~\cite{guestrin-kdd16},
SHAP~\cite{lundberg-nips17} and Anchor~\cite{guestrin-aaai18}, for
model-agnostic explanations, and sensitivity
analysis~\cite{vedaldi-iclr14} and LRP~\cite{muller-plosone15} in the
case of saliency maps for neural networks. 
However, such model-agnostic explainers offer no guarantees of
rigor. More importantly, the explanations computed by (non-formal)
model-agnostic explainers can be
unsound~\cite{inms-corr19,lukasiewicz-corr19,ignatiev-ijcai20,weller-ecai20}.
In addition, the running times of these non-formal tools are not on
par with the algorithms proposed in this paper, being in general
orders of magnitude slower.
There is recent work on approximate explanations with probabilistic
guarantees~\cite{kutyniok-jair21,vandenbroeck-ijcai21,tan-nips21},
with initial results for DTs reported in~\cite{iincms-corr21}.

\section{Conclusions} \label{sec:conc}

This paper investigates path explanation redundancy in decision
trees, i.e.\ the existence of features that are irrelevant for the
prediction associated with a given path.
In addition, the paper also shows that the computation of irredundant
path explanations in DTs is tightly related with recent work on
computing abductive explanations~\cite{inms-aaai19}. Furthermore, the
paper proposes several algorithms for computing path explanations, all
of which run in worst-case polynomial time.

The experimental results offer conclusive evidence supporting the
following claim: \emph{DTs consistently exhibit path explanation
  redundancy, which is often significant, not only in the number of
  paths exhibiting explanation redundancy, but also in the number of
  features that can be deemed explanation-redundant for the path}.
This claim is supported by the analysis of DTs used in a comprehensive
range of examples taken from textbooks and survey papers, some of
which dating back to the inception of well-known tree-learning
algorithms~\cite{breiman-bk84,quinlan-bk93}. This claim is also
supported by the analysis of the DTs learned with well-known
tree-learning algorithms, one of which explicitly targets
interpretability~\cite{bertsimas-ml17,iai}. Finally, the claim is
supported by the analysis of publicly available DTs generated with
so-called optimal (sparse) decision tree 
learners~\cite{verwer-aaai19,rudin-nips19,rudin-icml20,rudin-corr21},
which also explicitly target interpretability.

More importantly, the experimental results presented in this paper do
\emph{not} endorse the case made in recent research that DTs are
intrinsically interpretable, concretely when interpretability
correlates with succinctness of explanations.
However, these same experimental results support making the
alternative case: \emph{that DTs require being explained in practice, 
  that explaining DTs is computationally efficient in theory and in
  practice, and that explaining DTs must be a stepping stone for
  deploying ML in high-risk and safety-critical applications}.
Moreover, we conjecture that the same case can be made for other
classifiers that can be related with DTs in terms of the efficiency of
computing explanations.
Furthermore, we observe that the informal concept of interpretability
in the case of DTs is justified not by the intrinsic property of
explanations of DTs being succinct and irreducible, but by the fact
that rigorous explanations can be efficiently computed, both in the
case of DTs and possibly in the case of other related classifiers.
%

%
\section*{Acknowledgments}
%
This work was supported by the AI Interdisciplinary Institute ANITI, 
funded by the French program ``Investing for the Future -- PIA3''
under Grant agreement no.\ ANR-19-PI3A-0004, and by the H2020-ICT38
project COALA ``Cognitive Assisted agile manufacturing for a Labor
force supported by trustworthy Artificial intelligence''.
This work received comments from several colleagues, including
N.\ Asher, M.\ Cooper, E.\ Hebrard, X.\ Huang, C.\ Menc\'{\i}a,
N.\ Narodytska, R.\ Passos and J.\ Planes.

\newtoggle{mkbbl}

\settoggle{mkbbl}{false}

\addcontentsline{toc}{section}{References}
\vskip 0.2in
\iftoggle{mkbbl}{
  \bibliographystyle{abbrv}
  \bibliography{refs,dts,team}
}{

}
\label{lastpage}
%


\end{document}